\documentclass{article}

     \usepackage[preprint]{neurips_2022}

\usepackage{titletoc}
\usepackage[utf8]{inputenc} % allow utf-8 input
\usepackage[T1]{fontenc}    % use 8-bit T1 fonts
\usepackage[colorlinks=true, linkcolor=blue, citecolor=blue,urlcolor=black]{hyperref}       % hyperlinks
\usepackage{url}            % simple URL typesetting
\usepackage{booktabs}       % professional-quality tables
\usepackage{amsfonts}       % blackboard math symbols
\usepackage{nicefrac}       % compact symbols for 1/2, etc.
\usepackage{microtype}      % microtypography
\usepackage{xcolor}         % colors
\usepackage{enumitem}

\usepackage[utf8]{inputenc} % allow utf-8 input
\usepackage[T1]{fontenc}    % use 8-bit T1 fonts
\usepackage{hyperref}       % hyperlinks
\usepackage{url}            % simple URL typesetting
\usepackage{booktabs}       % professional-quality tables
\usepackage{amsfonts}       % blackboard math symbols
\usepackage{nicefrac}       % compact symbols for 1/2, etc.
\usepackage{microtype}      % microtypography
\usepackage{mathtools}
\usepackage{natbib}
\usepackage{xcolor}
\usepackage{amsmath}
\usepackage{mathtools}
\usepackage{amsthm}
\usepackage{amssymb}

\usepackage[algo2e, ruled, vlined]{algorithm2e}
\usepackage{algorithm}
\SetKwBlock{MyRepeat}{repeat}{}

\newcommand{\sinit}{s_{\text{init}}}

\newcommand{\calA}{{\mathcal{A}}}

\newcommand{\calC}{{\mathcal{C}}}
\newcommand{\calV}{{\mathcal{V}}}
\newcommand{\calH}{\mathcal{H}}

\newcommand{\calS}{{\mathcal{S}}}
\newcommand{\calF}{{\mathcal{F}}}

\newcommand{\calI}{{\mathcal{I}}}
\newcommand{\calJ}{{\mathcal{J}}}

\newcommand{\calE}{{\mathcal{E}}}

\newcommand{\calM}{{\mathcal{M}}}

\newcommand{\KL}{\text{\rm KL}}

\DeclareMathOperator*{\argmin}{argmin}

\newcommand{\eat}[1]{}

\newcommand{\rbr}[1]{\left(#1\right)}
\newcommand{\sbr}[1]{\left[#1\right]}
\newcommand{\cbr}[1]{\left\{#1\right\}}

\newcommand{\abr}[1]{\left|#1\right|}

\newcommand{\tilO}[1]{\otil\left( #1 \right)}

\newcommand{\bigo}[1]{\order( #1 )}
\newcommand{\tilo}[1]{\otil( #1 )}
\newcommand{\lowo}[1]{\lorder( #1 )}
\DeclarePairedDelimiter\ceil{\lceil}{\rceil}

\newcommand{\frN}{\mathfrak{N}}

\newcommand{\dM}{M_{\dagger}}
\newcommand{\summk}{\sum_{m=1}^K}
\renewcommand{\L}{L_{M'}}
\newcommand{\Lc}{L_{c,M'}}
\newcommand{\Lp}{L_{P,M'}}
\newcommand{\sumic}{\sum_{m=i^c_{M'}}^{M'}}
\newcommand{\sumip}{\sum_{m=i^P_{M'}}^{M'}}
\newcommand{\cQ}{\check{Q}}
\newcommand{\cV}{\check{V}}
\newcommand{\cc}{\check{c}}
\newcommand{\m}{\mathbf{m}}
\newcommand{\n}{\mathbf{n}}
\newcommand{\hatchi}{\widehat{\chi}}
\newcommand{\C}{\mathbf{C}}
\newcommand{\Mp}{\M^+}

\newcommand{\rcalM}{\mathring{\calM}}

\newcommand{\rR}{\mathring{R}}

\newcommand{\rQ}{\mathring{Q}}
\newcommand{\rV}{\mathring{V}}
\newcommand{\rx}{\mathring{x}}
\newcommand{\sumhm}{\sum_{h=1}^{H_m}}

\newcommand{\summp}{\sum_{m=1}^{M'}}
\newcommand{\frA}{\mathfrak{A}}
\newcommand{\jstar}{j^{\star}}
\newcommand{\geo}{\text{Geometric}}

\newcommand{\fstar}{f^{\star}}
\newcommand{\tilg}{\widetilde{g}}
\newcommand{\hatn}{\widehat{n}}

\newcommand{\hatR}{\widehat{R}}

\newcommand{\Tmax}{\ensuremath{T_{\max}}}

\newcommand{\T}{\ensuremath{T_\star}}
\newcommand{\B}{B_\star}

\newcommand{\var}{\textsc{Var}}

\newcommand{\bernoulli}{\textrm{Bernoulli}}

\newcommand{\SA}{\calS\times\calA}%{\Gamma}
\renewcommand{\P}{\bar{P}}
\newcommand{\istar}{i^{\star}}
\newcommand{\Np}{\N^+}

\newcommand{\optV}{V^{\star}}
\newcommand{\optQ}{Q^{\star}}

\newcommand{\optq}{q^{\star}}

\newcommand{\sumh}{\sum_{h=1}^H}

\newcommand{\sumk}{\sum_{k=1}^K}

\newcommand{\hatc}{\widehat{c}}
\newcommand{\barc}{\bar{c}}

\newcommand{\optpi}{\pi^\star}

\newcommand{\tilR}{\widetilde{R}}

\newcommand{\tilc}{\widetilde{c}}
\newcommand{\tilf}{\widetilde{f}}

\newcommand{\tilP}{\widetilde{P}}

\newcommand{\N}{\mathbf{N}} % counter N
\newcommand{\M}{\mathbf{M}}

\newcommand{\hatr}{\widehat{r}}

\newcommand{\hatalpha}{\widehat{\alpha}}

\newcommand{\field}[1]{\mathbb{#1}}
\newcommand{\fA}{\field{A}}

\newcommand{\fR}{\field{R}}

\newcommand{\fN}{\field{N}}
\newcommand{\E}{\field{E}}
\newcommand{\fV}{\field{V}}

\newcommand{\Ind}{\field{I}}

\newcommand{\norm}[1]{\left\|{#1}\right\|}

\newtheorem{lemma}{Lemma}
\newtheorem{theorem}{Theorem}

\newtheorem{assumption}{Assumption}

\newcommand{\order}{\ensuremath{\mathcal{O}}}
\newcommand{\lorder}{\ensuremath{\Omega}}
\newcommand{\otil}{\ensuremath{\tilde{\mathcal{O}}}}

\usepackage{prettyref}
\newcommand{\pref}[1]{\prettyref{#1}}
\newcommand{\pfref}[1]{Proof of \prettyref{#1}}
\newcommand{\savehyperref}[2]{\texorpdfstring{\hyperref[#1]{#2}}{#2}}
\newrefformat{eq}{\savehyperref{#1}{Eq.~\textup{(\ref*{#1})}}}
\newrefformat{eqn}{\savehyperref{#1}{Equation~\ref*{#1}}}
\newrefformat{lem}{\savehyperref{#1}{Lemma~\ref*{#1}}}
\newrefformat{def}{\savehyperref{#1}{Definition~\ref*{#1}}}
\newrefformat{line}{\savehyperref{#1}{Line~\ref*{#1}}}
\newrefformat{thm}{\savehyperref{#1}{Theorem~\ref*{#1}}}
\newrefformat{corr}{\savehyperref{#1}{Corollary~\ref*{#1}}}
\newrefformat{cor}{\savehyperref{#1}{Corollary~\ref*{#1}}}
\newrefformat{sec}{\savehyperref{#1}{Section~\ref*{#1}}}
\newrefformat{subsec}{\savehyperref{#1}{Section~\ref*{#1}}}
\newrefformat{app}{\savehyperref{#1}{Appendix~\ref*{#1}}}
\newrefformat{assum}{\savehyperref{#1}{Assumption~\ref*{#1}}}
\newrefformat{ex}{\savehyperref{#1}{Example~\ref*{#1}}}
\newrefformat{fig}{\savehyperref{#1}{Figure~\ref*{#1}}}
\newrefformat{alg}{\savehyperref{#1}{Algorithm~\ref*{#1}}}
\newrefformat{rem}{\savehyperref{#1}{Remark~\ref*{#1}}}
\newrefformat{conj}{\savehyperref{#1}{Conjecture~\ref*{#1}}}
\newrefformat{prop}{\savehyperref{#1}{Proposition~\ref*{#1}}}
\newrefformat{proto}{\savehyperref{#1}{Protocol~\ref*{#1}}}
\newrefformat{prob}{\savehyperref{#1}{Problem~\ref*{#1}}}
\newrefformat{claim}{\savehyperref{#1}{Claim~\ref*{#1}}}
\newrefformat{que}{\savehyperref{#1}{Question~\ref*{#1}}}
\newrefformat{op}{\savehyperref{#1}{Open Problem~\ref*{#1}}}
\newrefformat{fn}{\savehyperref{#1}{Footnote~\ref*{#1}}}
\newrefformat{tab}{\savehyperref{#1}{Table~\ref*{#1}}}
\newrefformat{fig}{\savehyperref{#1}{Figure~\ref*{#1}}}
\newrefformat{prop}{\savehyperref{#1}{Property~\ref*{#1}}}

\title{Near-Optimal Goal-Oriented Reinforcement Learning in Non-Stationary Environments}

\author{%
  Liyu Chen \\
  University of Southern California\\
  \texttt{liyuc@usc.edu} \\
  \And
  Haipeng Luo \\
  University of Southern California\\
  \texttt{haipengl@usc.edu}\\
}

\begin{document}

\maketitle

\begin{abstract}
We initiate the study of dynamic regret minimization for goal-oriented reinforcement learning modeled by a non-stationary stochastic shortest path problem with changing cost and transition functions.
We start by establishing a lower bound $\lowo{(\B SA\T(\Delta_c + \B^2\Delta_P))^{1/3}K^{2/3}}$, where $\B$ is the maximum expected cost of the optimal policy of any episode starting from any state, $\T$ is the maximum hitting time of the optimal policy of any episode starting from the initial state, $SA$ is the number of state-action pairs, $\Delta_c$ and $\Delta_P$ are the amount of changes of the cost and transition functions respectively, and $K$ is the number of episodes.
The different roles of $\Delta_c$ and $\Delta_P$ in this lower bound inspire us to design algorithms that estimate costs and transitions separately.
Specifically, assuming the knowledge of $\Delta_c$ and $\Delta_P$, we develop a simple but sub-optimal algorithm and another more involved minimax optimal algorithm (up to logarithmic terms).
These algorithms combine the ideas of finite-horizon approximation~\citep{chen2021improved}, special Bernstein-style bonuses of the MVP algorithm~\citep{zhang2020reinforcement}, adaptive confidence widening~\citep{wei2021non}, as well as some new techniques such as properly penalizing long-horizon policies.
	Finally, when $\Delta_c$ and $\Delta_P$ are unknown,
	we develop a variant of the MASTER algorithm~\citep{wei2021non} and integrate the aforementioned ideas into it to achieve $\tilo{\min\{\B S\sqrt{ALK}, (\B^2S^2A\T(\Delta_c+\B\Delta_P))^{1/3}K^{2/3}\}}$ regret, where $L$ is the unknown number of changes of the environment.
\end{abstract}

\section{Introduction}
% !TEX root = main.tex
Goal-oriented reinforcement learning studies how to achieve a certain goal with minimal total cost in an unknown environment via sequential interactions.
It has often been modeled as online learning in an episodic Stochastic Shortest Path (SSP) model, where in each episode, starting from a fixed initial state, the learner sequentially takes an action, suffers a cost, and transits to the next state, until the goal state is reached.
The performance of the learner can be measured by her \textit{regret}, generally defined as the difference between her total cost and that of a sequence of benchmark policies (one for each episode).

Despite the recent surge of studies on this problem, all previous works consider minimizing \textit{static regret}, a special case where the benchmark policy is the same for every episode. This is reasonable only for (near) stationary environments where one single policy performs well over all episodes.
In reality, however, the environment is often non-stationary with both the cost function and the transition function changing over episodes, making static regret an unreasonable metric.
Instead, the desired objective is to minimize \textit{dynamic regret}, where the benchmark policy for each episode is the optimal policy for that corresponding environment,
and the hope is to obtain sublinear dynamic regret whenever the non-stationarity is not too large.

%Online learning in SSP has recently received great attention since it suitably captures many real world application scenarios, such as navigation, robotic manipulation, etc.
%However, most previous works focus on learning in a stationary environment and minimize \textit{static regret}, that is, comparing with a single optimal policy in hindsight (all unknown benchmark policies are the same).
%This setting is restrictive as in many real world applications the environment is changing continuously.
%When non-stationarity is present, the desired algorithm should be adaptive to the non-stationarity and still achieve low regret when the total amount of non-stationarity is small.
%To this end, we consider learning in the \textit{non-stationary SSP} model where the cost and transition functions are changing over episodes (but fixed within an episode), and we study a more suitable metric called \textit{dynamic regret}, where the benchmark in each episode is the optimal policy w.r.t the current environment.

Based on this motivation, we initiate the study of dynamic regret minimization for non-stationary SSP and develop the first set of results.
Specifically, our contributions are as follows:
\begin{itemize}[leftmargin=*]
	\item To get a sense on the difficulty of the problem, we start by establishing a dynamic regret lower bound in \pref{sec:lb}. Specifically, we prove that $\lowo{(\B SA\T(\Delta_c + \B^2\Delta_P))^{1/3}K^{2/3}}$ regret is unavoidable, where $\B$ is the maximum expected cost of the optimal policy of any episode starting from any state, $\T$ is the maximum hitting time of the optimal policy of any episode starting from the initial state, $S$ and $A$ are the number of states and actions respectively, $\Delta_c$ and $\Delta_P$ are the amount of changes of the cost and transition functions respectively, and $K$ is the number of episodes.
	Note the different roles of $\Delta_c$ and $\Delta_P$ here --- the latter is multiplied with an extra $\B^2$ factor, which we find surprising for a technical reason discussed in \pref{sec:lb}.
	More importantly, 
	%	The ratio of $\B$ dependency between $\Delta_c$ and $\Delta_P$ in the lower bound is a bit surprising as it is different from the ratio in the optimal value change due to non-stationarity (see \pref{app:optV diff}).
	this inspires us to estimate costs and transitions independently in subsequent algorithm design.

	\item For algorithms, we first present a simple one (\pref{alg:MVP-SSP} in \pref{sec:subopt}) that achieves sub-optimal regret of $\tilo{(\B SA\Tmax(\Delta_c + \B^2\Delta_P))^{1/3}K^{2/3}}$, where $\Tmax \geq \T$ is the maximum hitting time of the optimal policy of any episode starting from any state.
	Except for replacing $\T$ with the larger quantity $\Tmax$, 
	this bound is optimal in all other parameters.
	Moreover, this also translates to a minimax optimal regret bound in the finite-horizon setting (a special case of SSP), making \pref{alg:MVP-SSP} the first model-based algorithm with the optimal $(SA)^{1/3}$ dependency.
	
	\item To improve the $\Tmax$ dependency to $\T$, in \pref{sec:opt}, we present a more involved algorithm (\pref{alg:mvp-test}) that achieves a near minimax optimal regret bound matching the earlier lower bound up to logarithmic terms.
	
	\item Both algorithms above require the knowledge of $\Delta_c$ and $\Delta_P$.
	Moreover, for a special kind of non-stationary environments where the cost/transition function only changes $L$ times, they are not able to achieve a more favorable dynamic regret bound of the form $\sqrt{LK}$.
	To overcome these issues altogether,
	in \pref{sec:unknown}, we develop a variant of the MASTER algorithm \citep{wei2021non} and integrate the earlier algorithmic ideas into it, which finally leads to a (sub-optimal) $\tilo{\min\{\B S\sqrt{ALK}, (\B^2S^2A\T(\Delta_c+\B\Delta_P))^{1/3}{K}^{2/3}\}}$ regret bound without knowing the non-stationarity $\Delta_c$, $\Delta_P$, or $L$.
\end{itemize}

\paragraph{Techniques} All our algorithms are built on top of a finite-horizon approximation scheme first proposed by~\citet{cohen2021minimax} and later improved by~\citet{chen2021improved}; see \pref{sec:fha}.
Both the sub-optimal \pref{alg:MVP-SSP} and the optimal \pref{alg:mvp-test} are then developed based on ideas from the MVP algorithm~\citep{zhang2020reinforcement} (for the finite-horizon setting), which adopts a UCBVI-style update rule \citep{azar2017minimax} with a special Bernstein-style bonus term.
The sub-optimal algorithm further integrates the idea of adaptive confidence widening \citep{wei2021non} into the UCBVI-style update by subtracting a bias from the cost function uniformly over all state-action pairs, which helps control the magnitude of the estimated value function.
The minimax optimal algorithm, on the other hand, adds a positive correction term to the cost function to penalize long-horizon policies, which helps improve the $\Tmax$ dependency to $\T$.
It also incorporates several non-stationarity tests to ensure that the algorithm resets its knowledge of the environment when the amount of non-stationarity is large.
Both algorithms maintain (update and reset) cost and transition estimation independently, which is the key to achieve the correct $\B$ dependency for both the $\Delta_c$-related and $\Delta_P$-related terms.

To handle unknown non-stationarity, we adopt the idea of the MASTER algorithm from~\citep{wei2021non}.
Although the nature of MASTER is a blackbox reduction, we cannot apply it directly due to the
presence of the correction term that changes continuously and brings extra challenges in tracking the learner's performance.
We handle this by redesigning the first non-stationarity test of the MASTER algorithm.
Specifically, we maintain multiple running averages of the estimated value function to detect different levels of non-stationarity.

\paragraph{Related Work} Static regret minimization in SSP has been heavily studied in recent years, for both stochastic costs~\citep{tarbouriech2020no,cohen2020near,cohen2021minimax,tarbouriech2021stochastic,chen2021implicit,jafarnia2021online,vial2021regret,min2021learning,chen2021improved} and adversarial costs~\citep{rosenberg2020adversarial,chen2021minimax,chen2021finding,chen2022policy}.
To the best of our knowledge, we are the first to study dynamic regret for non-stationary SSP.

There is also a surge of studies on online learning in non-stationary environments, ranging from bandits~\citep{auer2019adaptively,chen2019new,chen2021combinatorial,russac2020algorithms,faury2021regret,abbasi2022new,suk2021tracking} to reinforcement learning~\citep{gajane2018sliding,pmlr-v115-ortner20a,cheung2020reinforcement,fei2020dynamic,mao2020model,zhou2020nonstationary,touati2020efficient,domingues2021kernel,wei2021non,ding2022provably,lykouris2021corruption,wei2022model}.
Compared to previous work, the model we study is quite general and subsumes multi-armed bandit and finite-horizon reinforcement learning. On the other hand, it also introduces extra and unique challenges as we will discuss.
%Non-stationary reinforcement learning has been well studied in the finite-horizon setting~\citep{?} and infinite-horizon setting~\citep{?}.
%Learning under non-stationary environment is also a popular topic in bandits~\citep{?}.

\section{Preliminaries}
\label{sec:pre}
% !TEX root = main.tex

A non-stationary SSP instance consists of state space $\calS$, action space $\calA$, initial state $\sinit\in\calS$, goal state $g\notin\calS$, a set of cost mean functions $\{c_k\}_{k=1}^K$ with $c_k\in[0,1]^{\SA}$, and a set of transition functions $\{P_k\}_{k=1}^K$ with $P_k=\{P_{k,s,a}\}_{(s, a)\in\SA}$ and $P_{k,s,a}\in\Delta_{\calS_+}$, where $\calS_+=\calS\cup\{g\}$, $\Delta_{\calS_+}$ is the simplex over $\calS_+$, and $K$ is the number of episodes.
The set of cost and transition functions are unknown to the learner and determined by the environment before learning starts.

The learning protocol is as follows: the learner interacts with the environment for $K$ episodes.
In episode $k$, starting from the initial state $\sinit$, the learner sequentially takes an action, incurs a cost, and transits to the next state until reaching the goal state.
We denote by $(s^k_i, a^k_i, c^k_i, s^k_{i+1})$ the $i$-th state-action-cost-afterstate tuple observed in episode $k$, where $c^k_i$ is sampled from an unknown distribution with support $[0,1]$ and mean $c_k(s^k_i, a^k_i)$, and $s^k_{i+1}$ is sampled from $P_{k,s^k_i,a^k_i}$. 
We denote by $I_k$ the total number of steps in episode $k$, such that $s^k_{I_k+1}=g$.

\paragraph{Learning Objective} Intuitively, in each episode the learner aims at finding a policy that minimizes the total cost of reaching the goal state.
Formally, a policy $\pi\in\calA^{\calS}$ assigns an action $\pi(s)$ to each state $s\in\calS$, %and it is \textit{proper} if following $\pi$ (that is, taking action $\pi(s)$ whenever in state $s$) reaches the goal state with probability $1$.
and its expected cost for episode $k$ starting from a state $s$ is denoted as $V^{\pi}_k(s)=\E\big[\sum_{i=1}^{I_k} c_k(s_i^k, \pi(s_i^k))|P_k, s_1^k=s\big]$ where  the expectation is with respect to the randomness of next states $s_{i+1}^k\sim P_{k, s^k_i, \pi(s_i^k)}$ and the number of steps $I_k$ before reaching $g$.
The optimal policy $\optpi_k$ for episode $k$ is then the policy that minimizes $V^{\pi}_k(s)$ for all $s$.
Using $\optV_k$ as a shorthand for $V^{\optpi_k}_k$, we formally define the dynamic regret of the learner as
%Denote by $\optpi_k$ the optimal policy in episode $k$ and $V^{\star}_k(s)=\E[\sum_{i=1}^Ic_k(s_i, \optpi_k(s_i))|\optpi_k, P_k, s_1=s]$ the expected cost of $\optpi_k$ starting from state $s$, where the expectation is w.r.t the randomness of next states $s_{i+1}\sim P_{k, s^k_i, a^k_i}$ and the number of steps $I$ before reaching $g$.
%The formal objective of the learner is to minimize her dynamic regret, define as
\begin{align*}
	R_K = \sumk\rbr{\sum_{i=1}^{I_k}c^k_i - \optV_k(\sinit)}.
\end{align*}
When $I_k=\infty$ for some $k$, we let $R_K=\infty$.

Several parameters play a key role in characterizing the difficulty of this problem:
$\B=\max_{k,s}\optV_k(s)$, the maximum cost of the optimal policy of any episode starting from any state; $\T=\max_kT^{\optpi_k}_k(\sinit)$ (where $T^{\pi}_k(s)$ is expected number of steps it takes for policy $\pi$ to reach the goal in episode $k$ starting from state $s$), the maximum hitting time of the optimal policy of any episode starting from the initial state; $\Tmax=\max_{k,s}T^{\optpi_k}_k(s)$, the maximum hitting time of the optimal policy of any episode starting from any state;
$\Delta_c=\sum_{k=1}^{K-1}\norm{c_{k+1}-c_k}_{\infty}$, the amount of non-stationarity in the cost functions;
and finally $\Delta_P=\sum_{k=1}^{K-1}\max_{s, a}\norm{P_{k+1,s, a} - P_{k,s,a}}_1$, the amount of non-stationarity in the transition functions.
Throughout the paper we assume the knowledge of $\B$, $\T$, and $\Tmax$, and also
$\B\geq 1$ for simplicity.
$\Delta_c$ and $\Delta_P$ are assumed to be known for the first two algorithms we develop, but unknown for the last one.

\paragraph{Other Notations}
For a value function $V\in\fR^{\calS_+}$ and a distribution $P$ over $\calS_+$, define $PV=\E_{s'\sim P}[V(s')]$ (mean) and $\fV(P, V) = \E_{s'\sim P}[V(s')^2] - (PV)^2$ (variance).
Let $S=|\calS|$ and $A=|\calA|$ be the number of states and actions respectively.
The notation $\tilo{\cdot}$ hides all logarithmic dependency including $\ln K$ and $\ln\frac{1}{\delta}$ for some failure probability $\delta\in(0, 1)$.
Also define a value function upper bound $B=16\B$. 
For integers $s$ and $e$, we define $[s, e]=\{s, s+1,\ldots,e\}$ and $[e]=\{1,\ldots,e\}$.

\section{Lower Bound}
\label{sec:lb}
% !TEX root = main.tex

%In this section, we establish the lower bound for the non-stationary SSP problem.
To better understand the difficulty of learning non-stationary SSP, we first establish the following dynamic regret lower bound.
\begin{theorem}
	\label{thm:lb}
	%For any algorithm and any values of $\B, \T, SA, K$ with $\B\geq 1$, $\T\geq 3\B$, $K\geq SA\geq 10$, there exists a non-stationary SSP instance with $\max_{k,s}\optV_k(s)\leq2\B$ and $\max_kT^{\optpi_k}_k(\sinit)\leq \T+1$ such that the expected dynamic regret of the learner is at least $\lowo{(\B SA\T(\Delta_c + \B^2\Delta_P))^{1/3}K^{2/3}}$.
	In the worst case, the learner's regret is at least $\lowo{(\B SA\T(\Delta_c + \B^2\Delta_P))^{1/3}K^{2/3}}$.
\end{theorem}

The lower bound construction is similar to that in~\citep{mao2020model}, where the environment is piecewise stationary.
In each stationary period, the learner is facing a hard SSP instance with a slightly better hidden state. Details are deferred to \pref{app:pf lb}.

In a technical lemma in \pref{app:optV diff}, we show that for any two episodes $k_1$ and $k_2$, the change of the optimal value function due to non-stationarity satisfies $\optV_{k_1}(\sinit) - \optV_{k_2}(\sinit)\leq (\Delta_c+\B\Delta_P)\T$, with only one extra $\B$ factor for the $\Delta_P$-related term.
We thus find our lower bound somewhat surprising since an extra $\B^2$ factor shows up for the $\Delta_P$-related term.
%the ratio of $\B$ dependency between $\Delta_P$ and $\Delta_c$ is $\B^2$ instead of $\B$, which is different from the ratio shown in the change of optimal value function due to non-stationarity (see \pref{app:optV diff}).
%As a result, treating the non-stationarity of cost and transition as a whole as in \citep{wei2021non} does not give the right dependency (see \pref{sec:}).
This comes from the fact that constructing the hard instance with perturbed costs requires a larger amount of perturbation compared to that with perturbed transitions; see \pref{thm:lb cost} and \pref{thm:lb transition} for details.

More importantly, this observation implies that simply treating these two types of non-stationarity as a whole and only consider the non-stationarity in value function as done in \citep{wei2021non} does not give the right $\B$ dependency. % (see \pref{sec:unknown}).
This further inspires us to consider cost and transition estimation independently in our subsequent algorithm design.

\section{Basic Framework: Finite-Horizon Approximation}
\label{sec:fha}
% !TEX root = main.tex

%Compared to other Markov Decision Process (MDP) models, one unique challenge of learning non-stationary SSP is the variable episode length.
%Standard analysis of non-stationary reinforcement learning shows that dynamic regret scales linearly with the non-stationarity times the total steps of the learner.
%Therefore, we need to bound the total time steps before deriving a regret bound.
%However, a standard analysis of SSP (see for example \citep{cohen2020near,tarbouriech2021stochastic,chen2021implicit}) is to first derive a sub-linear regret bound, and then conclude a bound on the total number of time steps.
%This leads to circular arguments in our setting.
%Moreover, the implicit interval analysis in \citep{cohen2020near,tarbouriech2021stochastic} can no longer be implicit since the learner may need to restart in the middle of an episode.
%This further complicates the algorithm design and analysis.

Our algorithms are all built on top of the finite-horizon approximation scheme of~\citep{cohen2021minimax}, whose analysis is greatly simplified and improved by~\citep{chen2021improved}, making it applicable to our non-stationary setting as well.
This scheme makes use of an algorithm $\frA$ that deals with a special case of SSP where each episode ends within $H= \tilo{\Tmax}$ steps, and applies it to the original SSP following \pref{alg:fha}.
Specifically, call each ``mini-episode'' $\frA$ is facing an \textit{interval}.
At each step $h$ of interval $m$, the learner receives the decision $a_h^m$ from $\frA$, takes this action, observes the cost $c_h^m$, transits to the next state $s_{h+1}^m$, and then feed the observation $c_h^m$ and $s_{h+1}^m$ to $\frA$  (\pref{line:execute} and \pref{line:feed}).
The interval $m$ ends whenever one of the following happens (\pref{line:new interval}): 
the goal state is reached, $H$ steps have passed, or $\frA$ requests to start a new interval.\footnote{This last condition is not present in prior works. We introduce it since later our instantiation of $\frA$ will change its policy in the middle of an interval, and creating a new interval in this case allows us to make sure that the policy in each interval is always fixed, which simplifies the analysis.}
In the first case, the initial state $s_1^{m+1}$ of the next interval $m+1$ will be set to $\sinit$, while in the other two cases, it is naturally set to the learner's current state, which is also $s_{H_m+1}^m$ where $H_m$ is the length of interval $m$ (see \pref{line:next_interval}).
At the end of each interval, we artificially let $\frA$ suffer a \textit{terminal cost} $c_f(s_{H_m+1}^m)$ where $c_f(s)=2\B\Ind\{s\neq g\}$.

\DontPrintSemicolon
\setcounter{AlgoLine}{0}
\begin{algorithm}[t]
	\caption{Finite-Horizon Approximation of SSP}
	\label{alg:fha}
	\textbf{Input:} Algorithm $\frA$ for finite-horizon MDP $\rcalM$ with horizon $H= 4\Tmax\ln (8K)$.
	
	\textbf{Initialize:} interval counter $m\leftarrow 1$.
	
	\For{$k=1,\ldots,K$}{
	\nl	Set $s^m_1 \leftarrow \sinit$. \label{line:next_episode}
		
	\nl	\While{$s^m_1 \neq g$}{
     \nl         Feed initial state $s^m_1$ to $\frA$, $h\leftarrow 1$.
			
	\nl		\While{True}{	     
			    
	\nl			Receive action $a^m_h$ from $\frA$, play it, and observe cost $c^m_h$ and next state $s^m_{h+1}$. \label{line:execute}

	\nl		    Feed $c^m_h$ and $s^m_{h+1}$ to $\frA$. \label{line:feed}
				
	\nl			\If{$h=H$ or $s^m_{h+1}=g$ or $\frA$ requests to start a new interval}{\label{line:new interval}
	\nl				$H_m\leftarrow h$. \textbf{break}.
				}
	\nl			\lElse{$h\leftarrow h + 1$.}
			}
     \nl        Set $s^{m+1}_1 = s^m_{H_m+1}$ and $m\leftarrow m+1$. \label{line:next_interval}
		}
	}
	
\end{algorithm}

This procedure (adaptively) generates a non-stationary finite-horizon Markov Decision Process (MDP) that $\frA$ faces: $\rcalM=(\calS, \calA, g, \{c^m\}_{m=1}^M, \{P^m\}_{m=1}^M, c_f, H)$.
Here, $c^m = c_{k(m)}$ and $P^m = P_{k(m)}$ where $k(m)$ is the unique episode that interval $m$ belongs to,
and $M$ is the total number of intervals over $K$ episodes, a random variable determined by the interactions.
Let $V^{\pi,m}_1(s)$ be the expected cost (including the terminal cost) of following policy $\pi$ starting from state $s$ in interval $m$.
Define the regret of $\frA$ over the first $M'$ intervals in $\rcalM$ as $\rR_{M'}=\summp(\sum_{h=1}^{H_m+1}c^m_h - V^{\optpi_{k(m)},m}_1(s^m_1))$ 
where we use $c^m_{H_m+1}$ as a shorthand for the terminal cost $c_f(s^m_{H_m+1})$.
%$V^{\star,m}_1$ as the optimal value function of the first layer of $\rcalM$ (formally defined in \pref{app:}).
Following similar arguments as in \citep{cohen2021minimax,chen2021improved}, the regret in $\calM$ and $\rcalM$ are close in the following sense. 
\begin{lemma}
	\label{lem:fha}
	\pref{alg:fha} ensures $R_K\leq \rR_M + \B$.
\end{lemma}

See \pref{app:fha} for the proof.
Based on this lemma, in following sections we focus on developing the finite-horizon algorithm $\frA$ and analyzing how large $\rR_M$ is.
Note, however, that while this finite-horizon reduction is very useful, it does not mean that our problem is as easy as learning non-stationary finite-horizon MDPs and that we can directly plug in an existing algorithm as $\frA$. 
Great care is still needed when designing $\frA$ in order to obtain tight regret bounds as we will show.

\section{A Simple Sub-Optimal Algorithm}
\label{sec:subopt}
% !TEX root = main.tex

\DontPrintSemicolon
\setcounter{AlgoLine}{0}
\begin{algorithm}[t]
	\caption{Non-Stationary MVP}
	\label{alg:MVP-SSP}
	\SetKwFunction{update}{Update}
	\SetKwProg{proc}{Procedure}{}{}
	\textbf{Parameters:} window sizes $W_c$ (for costs) and $W_P$ (for transitions), and failure probability $\delta$.
	
	\textbf{Initialize:} for all $(s, a, s')$, $\C(s, a)\leftarrow 0$, $\M(s, a)\leftarrow 0$, $\N(s, a)\leftarrow 0$, $\N(s, a, s')\leftarrow 0$.
	
	\textbf{Initialize:} \update{$1$}.
	
	\For{$m=1,\ldots,M$}{	
	
		\For{$h=1,\ldots,H$}{
		\nl	Play action $a^m_h\leftarrow\argmin_aQ_h(s^m_h, a)$, receive cost $c^m_h$ and next state $s^m_{h+1}$. \label{line:greedy_action}
			
			$\C(s^m_h, a^m_h)\overset{+}{\leftarrow} c^m_h$, $\M(s^m_h, a^m_h)\overset{+}{\leftarrow}1$, $\N(s^m_h, a^m_h)\overset{+}{\leftarrow}1$, $\N(s^m_h, a^m_h, s^m_{h+1})\overset{+}{\leftarrow}1$.\footnotemark
			
			\nl \If{$s^m_{h+1}=g$ or $\M(s^m_h, a^m_h)=2^l$ or $\N(s^m_h, a^m_h)=2^l$ for some integer $l\geq 0$ \label{line:request_new_interval}}{
				\textbf{break} (which starts a new interval).
			}
		}
		
		\nl \lIf{$W_c$ divides $m$}{reset $\C(s, a)\leftarrow 0$ and $\M(s, a)\leftarrow0$ for all $(s, a)$.}\label{line:reset M}
		
		\nl \lIf{$W_P$ divides $m$}{reset $\N(s, a, s')\leftarrow0$ and $\N(s, a)\leftarrow0$ for all $(s, a, s')$.}\label{line:reset N}
		
		\update{$m+1$}.
	}
	
	\proc{\update{$m$}}{
		$V_{H+1}(s)\leftarrow2\B\Ind\{s\neq g\}$, $V_h(g)\leftarrow 0$ for $h\leq H$,  $\iota\leftarrow 2^{11}\cdot\ln\big(\frac{2SAHKm}{\delta}\big)$, and $x \leftarrow \frac{1}{mH}$.
		
		\For{all $(s, a)$}{
			$\N^+(s, a)\leftarrow\max\{1, \N(s, a)\}$, $\M^+(s, a)\leftarrow\max\{1, \M(s, a)\}$, $\barc(s, a)\leftarrow \frac{\C(s, a)}{\M^+(s, a)}$, 
			
			$\hatc(s, a)\leftarrow \max\Big\{0, \barc(s, a) - \sqrt{\frac{\barc(s, a)\iota}{\M^+(s, a)}} - \frac{\iota}{\M^+(s, a)}\Big\}$, 
			$\P_{s, a}(\cdot)\leftarrow\frac{\N(s, a, \cdot)}{\N^+(s, a)}$.
		}
		
		\While{True}{
		\For{$h=H,\ldots,1$}{
		\nl $b_h(s, a)\leftarrow \max\Big\{7\sqrt{\frac{\fV(\P_{s, a}, V_{h+1})\iota}{\Np(s, a)}}, \frac{49B\sqrt{S}\iota}{\Np(s, a)}\Big\}$ for all $(s, a)$. \label{line:bonus}
		
			\nl $Q_h(s, a)\leftarrow\max\{0, \hatc(s, a) + \P_{s, a}V_{h+1} - b_h(s, a) - x\}$ for all $(s, a)$. \label{line:update}
			
			$V_h(s)\leftarrow\min_aQ_h(s, a)$ for all $s$.
		}
		\nl \lIf{$\max_{s, a, h}Q_h(s, a)\leq B/4$}{\textbf{break}; \textbf{else} $x \leftarrow 2x$.} \label{line:double_search}
		}
	}
\end{algorithm}

\footnotetext{$z\overset{+}{\leftarrow}y$ is a shorthand for $z \leftarrow z+y$.}

In this section, we present a relatively simple finite-horizon algorithm $\frA$ for $\rcalM$ which, in combination with the reduction of \pref{alg:fha}, achieves a regret bound that almost matches our lower bound except that $\T$ is replaced by $\Tmax$.
The key steps are shown in \pref{alg:MVP-SSP}.
It follows the ideas of the MVP algorithm~\citep{zhang2020reinforcement} and adopts a UCBVI-style update rule (\pref{line:update}) with a Bernstein-type bonus term (\pref{line:bonus}) to maintain a set of $Q_h$ functions, which then determines the action at each step in a greedy manner (\pref{line:greedy_action}).
The two crucial new elements are the following.
First, in the update rule \pref{line:update}, we subtract a positive value $x$ uniformly over all state-action pairs so that $\norm{Q_h}_{\infty}$ is of order $\order(\B)$ (recall $B = 16\B$), and we find the (almost) smallest such $x$ via a doubling trick (\pref{line:double_search}).
This is similar to the adaptive confidence widening technique of~\citep{wei2021non}, where they increase the size of the transition confidence set to ensure a bounded magnitude on the estimated value function;
our approach is an adaptation of their idea to the UCBVI style update rule.

Second, we periodically restart the algorithm (by resetting some counters and statistics) in \pref{line:reset M} and \pref{line:reset N}.
While periodic restart is a standard idea to deal with non-stationarity,
the novelty here is a two-scale restart schedule:
we set one window size $W_c$ related to costs and another one $W_P$ related to transitions, and restart after every $W_c$ intervals or every $W_P$ intervals.
As mentioned, this two-scale schedule is inspired by the lower bound in \pref{sec:lb}, which indicates that cost estimation and transition estimation play different roles in the final regret and should be treated separately.

%Our algorithm design follows the ideas of MVP~\citep{zhang2020reinforcement}: it adopts a UCBVI-style update rule with a Bernstein-type bonus term (\pref{line:update}) with two crucial differences: 1) we maintain separate accumulators for cost and transition estimation, and reset them independently (\pref{line:reset M} and \pref{line:reset N}); this is inspired by the lower bound construction in \pref{sec:lb}, which indicates that the regret from cost estimation and transition estimation are additive and thus independent of each other; 2) when computing the estimated action-value function $Q_h$, we subtract a minimum positive value $x$ uniformly over the costs of all state-action pairs so that $\norm{Q_h}_{\infty}$ is of order $\B$ (\pref{line:update});
%this is similar to the adaptive confidence widening in \citep{wei2021non}, where they increase the size of transition confidence set to ensure a bounded magnitude on the estimated value function;
%our approach is an adaptation of their idea to the UCBVI style update rule.
%which avoids an extra $S^{1/3}$ factor.

Another small modification is that we start a new interval when the visitation to some $(s,a)$ doubles (\pref{line:request_new_interval}), which helps remove $\Tmax$ dependency in  lower-order terms and is important for following sections.
With all these elements, we prove the following regret guarantee of \pref{alg:MVP-SSP}.
%We define $L_{c, M'}$ and $L_{P, M'}$ as the number of resets of counter $\M$ and $\N$ in the first $M'$ intervals respectively and $\Delta_{c,m}$.

\begin{theorem}
	\label{thm:mvp}
	For any $M'\leq M$, with probability at least $1-22\delta$ \pref{alg:MVP-SSP} ensures %$\rR_{M'} = \tilo{\sqrt{\B SA/W_c}M' + \B\sqrt{SA/W_P}M' + \B SAM'/W_c + \B S^2AM'/W_P + (\Delta_cW_c+\B\Delta_PW_P)\Tmax}$.
	$\rR_{M'} = \tilO{
	M'\Big(\sqrt{\B SA\big(\nicefrac{1}{W_c}+\nicefrac{\B}{W_P}\big)} + \B SA\rbr{\nicefrac{1}{W_c}+\nicefrac{S}{W_P}}\Big) + (\Delta_cW_c+\B\Delta_PW_P)\Tmax
	}$.
\end{theorem}

Thus, %if we set $W_c=\tilo{(\B SA)^{1/3}(M'/(\Delta_c\Tmax))^{2/3}}$ and $W_P=\tilo{(SA)^{1/3}(M'/(\Delta_P\Tmax))^{2/3} }$, 
with a proper tunning of $W_c$ and $W_P$ (that is in term of $M'$), \pref{alg:MVP-SSP} ensures %$\rR_{M'}= \tilo{ (\B SA\Delta_c\Tmax)^{1/3}{M'}^{2/3} + \B(SA\Delta_P\Tmax)^{1/3}{M'}^{2/3} }$.
$\rR_{M'} = \tilo{(\B SA\Tmax(\Delta_c + \B^2\Delta_P))^{1/3}{M'}^{2/3}}$.
However, this does not directly imply a bound on $\rR_M$ since $M$ is a random variable (and the tunning above would depend on $M$). 
Fortunately, to resolve this it suffices to perform a doubling trick on the number of intervals, that is, first make a guess on $M$, and then double the guess whenever $M$ exceeds it.
We summarize this idea in \pref{alg:MVP-SSP-Restart}. 
Finally, combining it with \pref{alg:fha}, \pref{lem:fha}, and the simplified analysis of~\citep{chen2021improved} which is able to bound the total number of intervals $M$ in terms of the total number of episodes $K$ (\pref{lem:bound M}), we obtain the following result (all proofs are deferred to \pref{app:subopt}).

\begin{theorem}
	\label{thm:any episode}
	With probability at least $1-22\delta$, applying \pref{alg:fha} with $\frA$ being \pref{alg:MVP-SSP-Restart} ensures $R_{K'}=\tilo{(\B SA\Tmax(\Delta_c + \B^2\Delta_P))^{1/3}{K'}^{2/3}}$ (ignoring lower order terms) for any $K'\leq K$.
\end{theorem}
Note that \pref{thm:any episode} actually provides an anytime regret guarantee (that is, holds for any $K' \leq K$), which is important in following sections.
Compared to our lower bound in \pref{thm:lb}, the only sub-optimality is in replacing $\T$ with the larger quantity $\Tmax$.
Despite its sub-optimality for SSP, however, as a side result our algorithm in fact implies the first model-based finite-horizon algorithm that achieves the optimal dependency on $SA$ and matches the minimax lower bound of~\citep{mao2020model}.
Specifically, in previous works, the optimal $SA$ dependency is only achievable by model-free algorithms, which unfortunately have sub-optimal dependency on the horizon by the current analysis (see \citep[Lemma 10]{mao2020model}).
On the other hand, existing model-based algorithms for finite state-action space all follow the idea of extended value iteration, which gives sub-optimal dependency on $S$ and also brings difficulty in incorporating entry-wise Bernstein confidence sets.\footnote{Note that the transition non-stationarity $\Delta_P$ is defined via $L_1$ norm. Thus, naively applying entry-wise confidence widening to Bernstein confidence sets introduces extra dependency on $S$.}
Our approach, however, resolves all these issues. See \pref{app:finite horizon} for more discussions.

\paragraph{Technical Highlights} The key step of our proof for \pref{thm:mvp} is to bound the term $\summp\sumhm\fV(P^m_{s^m_h, a^m_h}, V^{\star,m}_{h+1} - V^m_{h+1})$, where $V^m_{h+1}$ is the value of $V_{h+1}$ at the beginning of interval $m$, and $V^{\star,m}_{h+1}$ is the optimal value function of $\rcalM$ in interval $m$ (formally defined in \pref{app:pre}).
The standard analysis on bounding this term requires $V^{\star,m}_{h+1}(s) - V^m_{h+1}(s) \geq 0$, which is only true in a stationary environment due to optimism.
%Without non-stationarity the difference $V^{\star,m}_{h+1} - V^m_{h+1}$ is shrinking and the sum is sub-linear.
To handle this in non-stationarity environments, we carefully choose a set of constants $\{z^m_h\}$ so that $V^{\star,m}_{h+1}(s) + z^m_h - V^m_{h+1}(s) \geq 0$ (\pref{lem:opt Q}), and then apply similar analysis on $\summp\sumhm\fV(P^m_{s^m_h, a^m_h}, V^{\star,m}_{h+1} - V^m_{h+1})=\summp\sumhm\fV(P^m_{s^m_h, a^m_h}, V^{\star,m}_{h+1} + z^m_h - V^m_{h+1})$.
See \pref{lem:opt-til var sum} for more details.

\begin{algorithm}[t]
	\caption{Non-Stationary MVP with a Doubling Trick}
	\label{alg:MVP-SSP-Restart}
	\For{$n=1,2,\ldots$}{
		Initialize an instance of \pref{alg:MVP-SSP} with $W_c=\ceil{(\B SA)^{1/3}(2^{n-1}/(\Delta_c\Tmax))^{2/3}}$ and $W_P=\ceil{(SA)^{1/3}(2^{n-1}/(\Delta_P\Tmax))^{2/3}}$, and execute it in intervals $m=2^{n-1},\ldots,2^n-1$.
	}
\end{algorithm}

\section{A Minimax Optimal Algorithm}
\label{sec:opt}
% !TEX root = main.tex

\setcounter{AlgoLine}{0}
\begin{algorithm}[t]
	\caption{MVP with Non-Stationarity Tests}
	\label{alg:mvp-test}
	\SetKwFunction{update}{Update}
	\SetKwFunction{resetc}{ResetC}
	\SetKwFunction{resetp}{ResetP}
	\SetKwProg{proc}{Procedure}{}{}
	%textbf{Define:} $b^m(s, a, V) = \max\cbr{7\sqrt{\frac{\fV(\P^m_{s, a}, V)\iota_m}{\Np_m(s, a)}}, \frac{120B_m\sqrt{S}\iota_m}{\Np_m(s, a)}}$ and $\tilc^m(s, a)=\hatc^m(s, a)+\eta_m$.
	
	%\textbf{Define:} $\nu^c=m-i^c+1$ and $\nu^P=m-i^P+1$.
	
	\textbf{Parameters:} window sizes $W_c$ and $W_P$, coefficients $c_1$, $c_2$, sample probability $p$, and failure probability $\delta$.
	
	\textbf{Initialize:} \resetc{}, \resetp{}, \update{$1$}.
	
	\For{$m=1,\ldots,M$}{	
		\For{$h=1,\ldots,H$}{
			Play action $a^m_h\leftarrow\argmin_a\cQ_h(s^m_h, a)$, receive cost $c^m_h$ and next state $s^m_{h+1}$.
			
			$\C(s^m_h, a^m_h)\overset{+}{\leftarrow} c^m_h$, $\M(s^m_h, a^m_h)\overset{+}{\leftarrow}1$, $\N(s^m_h, a^m_h)\overset{+}{\leftarrow}1$, $\N(s^m_h, a^m_h, s^m_{h+1})\overset{+}{\leftarrow}1$.
			
			\nl $\hatchi^c\overset{+}{\leftarrow}c^m_h - \hatc(s^m_h, a^m_h)$, $\hatchi^P\overset{+}{\leftarrow}\cV_{h+1}(s^m_{h+1}) - \P_{s^m_h, a^m_h}\cV_{h+1}$. \label{line:chi_estimate}
			
			\If{$s^m_{h+1}=g$ or $\M(s^m_h, a^m_h)=2^l$ or $\N(s^m_h, a^m_h)=2^l$ for some integer $l\geq 0$}{
				\textbf{break} (which start a new interval).
			}
		}
		
		\nl \lIf{$\hatchi^c>\chi^c_m$ (defined in \pref{lem:chic})}{\resetc{}. \textbf{(Test 1)}}\label{line:test 1}
		
		\nl \lIf{$\hatchi^P > \chi^P_m$ (defined in \pref{lem:chip})}{\resetc{} and \resetp{}. \textbf{(Test 2)}}\label{line:test 2}
		
		\nl \lIf{$\nu^c=W_c$}{\resetc{}.}\label{line:period c}
		
		\nl \lIf{$\nu^P=W_P$}{\resetc{} and \resetp{}.}\label{line:period p}
	
		$\nu^c\overset{+}{\leftarrow} 1$, $\nu^P\overset{+}{\leftarrow} 1$, \update{$m+1$}.
	
		\nl \If{$\norm{\cV_h}_{\infty}>B/2$ for some $h$ \textbf{(Test 3)}}{ \label{line:test 3}
			\resetc{}, with probability $p$ execute \resetp{},
		     and \update{$m+1$}.
		}
	}
	
	\proc{\update{$m$}}{
		$\cV_{H+1}(s)\leftarrow2\B\Ind\{s\neq g\}$, $\cV_h(g)\leftarrow0$ for all $h\leq H$, and $\iota\leftarrow 2^{11}\cdot\ln\big(\frac{2SAHKm}{\delta}\big)$.
		
		\nl $\rho^c\leftarrow \min\{\frac{c_1}{\sqrt{\nu^c}}, \frac{1}{2^8H}\}$, $\rho^P\leftarrow\min\{\frac{c_2}{\sqrt{\nu^P}}, \frac{1}{2^8H}\}$, $\eta\leftarrow \rho^c + B\rho^P$. \label{line:correct}
		
		\For{all $(s, a)$}{
			$\N^+(s, a)\leftarrow\max\{1, \N(s, a)\}$, $\M^+(s, a)\leftarrow\max\{1, \M(s, a)\}$, $\barc(s, a)\leftarrow \frac{\C(s, a)}{\M^+(s, a)}$,  
			
			$\P_{s, a}(\cdot)\leftarrow\frac{\N(s, a, \cdot)}{\N^+(s, a)}$, 
			$\hatc(s, a)\leftarrow \max\Big\{0, \barc(s, a) - \sqrt{\frac{\barc(s, a)\iota}{\M^+(s, a)}} - \frac{\iota}{\M^+(s, a)}\Big\}$, 
			
			\nl $\cc(s, a)\leftarrow\hatc(s, a) + 8\eta$.\label{line:compute}
		}
		
		\For{$h=H,\ldots,1$}{
		    $b_h(s, a)\leftarrow \max\cbr{7\sqrt{\frac{\fV(\P_{s, a}, \cV_{h+1})\iota}{\Np(s, a)}}, \frac{49B\sqrt{S}\iota}{\Np(s, a)}}$ for all $(s, a)$.
		
			$\cQ_h(s, a)=\max\{0, \cc(s, a) + \P_{s, a}\cV_{h+1} - b_h(s, a)\}$ all $(s, a)$.
			
			$\cV_h(s)=\argmin_a\cQ_h(s, a)$ for all $s$.
		}
	}
	
	\proc{\resetc{}}{
		$\nu^c\leftarrow 1$, $\hatchi^c\leftarrow 0$, $\C(s, a)\leftarrow 0$, $\M(s, a)\leftarrow 0$ for all $(s, a)$.
	}
	\proc{\resetp{}}{
		 $\nu^P\leftarrow 1$, $\hatchi^P\leftarrow 0$, $\N(s, a, s')\leftarrow 0$, $\N(s, a)\leftarrow 0$ for all $(s, a, s')$.
	}
\end{algorithm}

In this section, we present an improved algorithm that achieves the minimax optimal regret bound up to logarithmic terms, starting with a refined version of \pref{alg:MVP-SSP} shown in \pref{alg:mvp-test}.
Below, we focus on describing the new elements introduced in \pref{alg:mvp-test} (that is, Lines~\ref{line:chi_estimate}-\ref{line:test 2} and \ref{line:test 3}-\ref{line:compute}).\footnote{\pref{line:period c} and \pref{line:period p}, althogh written in a different form, are similar to \pref{line:reset M} and \pref{line:reset N} of \pref{alg:MVP-SSP}.}

The main challenge in replacing $\Tmax$ with $\T$ is that the regret due to non-stationarity accumulates along the learner's trajectory, which can be as large as $\bigo{(\Delta_c+\B\Delta_P)H}$ since the horizon is $H$ (recall $H=\tilo{\Tmax}$).
Moreover, bounding the number of steps needed for the learner's policy to reach the goal is highly non-trivial due to the changing transitions.
Our main idea to address these issues is to incorporate a correction term $\eta$ (computed in \pref{line:correct}) into the estimated cost (\pref{line:compute}) to penalize policies that take too long to reach the goal.
This correction term is set to be an upper bound of the learner's average regret per interval (defined through $\rho^c$ and $\rho^P$ in \pref{line:correct}).
It introduces the effect of canceling the non-stationarity along the learner's trajectory when it is not too large.
%Ideally, we want to add a correction term equals to the accumulated non-stationarity up to the current interval.
%At first glance, it is unclear why this is possible since the non-stationarity in each interval is unknown.
%Our key inspiration from \citep{wei2021non} is that the non-stationarity does not affect the regret bound when it is upper bounded by the average regret per interval (that is, when the environment is nearly stationary), and the latter can be upper bounded explicitly (by ).
%It is thus tempting to directly use the latter as correction term (\pref{line:compute}), which intuitively is the largest correction term one can incorporate without affecting the regret bound.
%This addresses the small non-stationarity case.
When the non-stationarity is large, on the other hand, we detect it through two non-stationary tests (\pref{line:test 1} and \pref{line:test 2}), and reset the knowledge of the environment (more details to follow).

However, this correction leads to one issue: we cannot perform adaptive confidence widening (that is, the $-x$ bias) anymore as it would cancel out the correction term.
To address this, we introduce another test (\pref{line:test 3}, \textbf{Test 3}) to directly check whether the magnitude of the estimated value function is bounded as desired.
If not, we reset again since that is also an indication of large non-stationarity.

%We manage to resolve these issues by incorporating several non-stationary tests (Line \ref{line:test 1}, \ref{line:test 2}, and \ref{line:test 3}), where \textbf{Test 1} and \textbf{Test 2} resolves issue 2) by directly enforcing a bound on the estimated regret, and \textbf{Test 3} resolves issue 1) by controlling the magnitude of the estimated value function.
%The main idea is that when these tests are violated, the non-stationarity is large enough so that it is safe to reset the knowledge of the environment.

We now provide some intuitions on the design of \textbf{Test 1} and \textbf{Test 2}.
First, one can show that the two quantities $\hatchi^c$ and $\hatchi^P$ we maintain in \pref{line:chi_estimate} are such that their sum is roughly an upper bound on the estimated accumulated regret. %$\sum_m(C^m - \cV^m_1(s^m_1))$, where $\cV^m_1$ is the value of $\cV_1$ at the beginning of interval $m$.
So directly checking whether $\hatchi^c+\hatchi^P$ is too large would be similar to the second test of the MASTER algorithm~\citep{wei2021non}.
Here, however, we again break it into two tests where \textbf{Test 1} only guards the non-stationarity in cost, and \textbf{Test 2} mainly guards the non-stationarity in transition.
%simulate independent reset of the cost and transition estimation.
Note that \textbf{Test 2} also involves cost information through $\cV$,
but %Ideally, we want each test to guard only the cost non-stationarity or transition non-stationarity. Unfortunately, this is not possible when $\cV$ is involved.
our observation is that we can still achieve the desired regret bound as long as the ratio of the number of resets caused by procedures ResetC() and ResetP() is of order $\tilo{\B}$.
This inspires us to reset both the cost and the transition estimation when \textbf{Test 2} fails, but reset the  transition estimation only with some probability $p$ (eventually set to $1/\B$) when \textbf{Test 3} fails.

%However, this leads to many subsequent issues, which we manage to resolve by incorporating several non-stationary tests and a carefully designed reset rules when these tests fail (Line \ref{line:test 1}, \ref{line:test 2}, and \ref{line:test 3}).
%The rest of the algorithm is similar to \pref{alg:MVP-SSP}, and we summarize its pseudo code in \pref{alg:mvp-test}.

For analysis, we first establish a regret guarantee of \pref{alg:mvp-test} in an ideal situation where the first state of each interval is always $\sinit$. (Proofs of this section are deferred to \pref{app:opt}.)

%1) the accumulated non-stationarity is unknown; 
%moreover, we can no longer subtract some value uniformly over all state-action pairs when computing the estimated value function, since it simply cancels out the correction term;
%fortunately, we can assume that the accumulated non-stationarity is upper bounded by a non-increasing function on the number of intervals, and specify the correction term accordingly (\pref{line:correct});
%we further incorporate some non-stationarity tests (Line \ref{line:test 1}, \ref{line:test 2}, and \ref{line:test 3}) to make sure we reset the algorithm when the precondition above does not hold; 2) we need to show that 

\begin{theorem}
	\label{thm:mvp-test}
	Let $c_1=\sqrt{\B SA}/\T$, $c_2=\sqrt{SA}/\T$, $W_c=\ceil{(\B SA)^{1/3}(K/(\Delta_c\T))^{2/3}}$, $W_P=\ceil{(SA)^{1/3}(K/(\Delta_P\T))^{2/3}}$, and $p=1/\B$.
	Suppose $s^m_1=\sinit$ for all $m\leq K$, 
	then \pref{alg:mvp-test} ensures $\rR_K =\tilo{(\B SA\T(\Delta_c + \B^2\Delta_P))^{1/3}K^{2/3}}$ (ignoring lower order terms) with probability at least $1-40\delta$.
\end{theorem}

\begin{algorithm}[t]
	\caption{A Two-Phase Variant of \pref{alg:fha}}
	\label{alg:two phase}
	\textbf{Initialize:} Phase $1$ algorithm instance $\frA_1$ and Phase $2$ algorithm instance $\frA_2$.
	
	Execute \pref{alg:fha} with $\frA=\frA_1$ for every first interval of an episode, and $\frA=\frA_2$ otherwise.
	
%	\For{$k=1,\ldots,K$}{
%		Execute $\frA_1$ for one interval and receive the last state of this interval $s_e$.
%		
%		\lIf{$s_e\neq g$}{
%			Execute \pref{alg:fha} with $\frA=\frA_2$ until the goal state is reached.
%		}
%	}
\end{algorithm}

%\paragraph{Two-Phase Algorithm} 
The reason that we only analyze this ideal case is that, if the initial state is not $\sinit$, then even the optimal policy does not guarantee $\T$ hitting time by definition.
This also inspires us to eventually deploy a two-phase algorithm slightly modifying \pref{alg:fha}: feed the first interval of each episode into an instance of \pref{alg:mvp-test}, and the rest of intervals into an instance of \pref{alg:MVP-SSP-Restart} (see \pref{alg:two phase}).
%Note that if executing \pref{alg:mvp-test} does not reach the goal state in the first interval of some episode, it does not make sense to start a new interval with \pref{alg:mvp-test} since now the initial may not be $\sinit$ and the expected hitting time of the optimal policy may not be $\T$ anymore.
Thanks to the large terminal cost, we are able to show that the regret in the second phase is upper bounded by a constant, leading to the following final result.

\begin{theorem}
	\label{thm:two phase}
	\pref{alg:two phase} with $\frA_1$ being \pref{alg:mvp-test} and $\frA_2$ being \pref{alg:MVP-SSP-Restart} ensures $R_K=\tilo{(\B SA\T(\Delta_c + \B^2\Delta_P))^{1/3}K^{2/3}}$ (ignoring lower order terms) with probability at least $1-64\delta$.
\end{theorem}

Ignoring logarithmic and lower-order terms, our bound is minimax optimal.
Also note that the bound is sub-linear (in $K$) as long as $\Delta_c$ and $\Delta_P$ are sub-linear (that is, not the worst case). %$o(K)$ after taking lower order terms into consideration.

\section{Learning without Knowing $\Delta_c$ and $\Delta_P$}
\label{sec:unknown}
% !TEX root = main.tex

To handle unknown non-stationarity, we combine our algorithmic ideas in previous sections with a new variant of the MASTER algorithm \citep{wei2021non}.
The original MASTER algorithm is a blackbox reduction that takes a base algorithm for (near) stationary environments as input, and turns it into another algorithm for non-stationarity environments.
For many problems (including multi-armed bandits, contextual bandits, linear bandits, finite-horizon or infinite-horizon MDPs), \citet{wei2021non} show that the final algorithm achieves optimal regret without knowing the non-stationarity. 
While powerful, MASTER can not be directly used in our problem to achieve the same strong result.
As we will discuss, some modification is needed, and even with this modification, some extra difficulty unique to SSP still prevents us from eventually obtaining the optimal regret.

Specifically, in order to obtain $\T$ dependency, we again follow the two-phase procedure \pref{alg:two phase} and instantiate a MASTER algorithm with a different base algorithm in each phase.
In Phase $1$, since it is unclear how to update cost and transition estimation independently under the framework of MASTER, we adopt a simpler version of \pref{alg:mvp-test} as the base algorithm, which performs synchronized cost and transition estimation and a simpler non-stationarity test; see \pref{alg:mvp-base} (all algorithms/proofs in this section are deferred to \pref{app:unknown} due to space limit).
In Phase $2$, we use \pref{alg:MVP-SSP} as the base algorithm.
%Also note that MASTER is modified for correction term.

Our version of the MASTER algorithm (\pref{alg:master}) requires a different \textbf{Test 1} compared to that in~\citep{wei2021non}, which is essential due to the presence of the correction terms in \pref{alg:mvp-base}.
Specifically, it no longer makes sense to simply maintain the maximum of estimated value functions over the past intervals, since the cost function combined with the correction term is changing adaptively, and a large correction term will interfere with the detection of a small amount of non-stationarity.
%and the estimated value functions in different intervals are not comparable.
Our key observation is that for a base algorithm scheduled on a given range by MASTER, the average of its correction terms within the same range is of the desired order that does not interfere with non-stationarity detection. 
This inspires us to maintain multiple running averages of the estimated value functions with different scales (see \pref{line:U} of \pref{alg:master}).
Then, to detect a certain level of non-stationarity, we refer to the running average with the matching scale (see \pref{line:master test 1}).

We show that the algorithm described above achieves the following regret guarantee without knowledge of the non-stationarity.
%Define $L=1+\sumk\Ind\{P_{k+1}\neq P_k\}$.
%We let $\Delta=(\Delta_c+\B\Delta_P)\T$ be the total number of non-stationarity and $L=1+\sum_{k=1}^{K-1}\Ind\{P^{k+1}\neq P^k\}$ be the number of distribution changes.

\begin{theorem}
	\label{thm:unknown}
	Let $\frA_1$ be an instance of \pref{alg:master} with \pref{alg:mvp-base} as the base algorithm and $\frA_2$ be an instance of \pref{alg:master} with \pref{alg:MVP-SSP} as the base algorithm. Then \pref{alg:two phase} with $\frA_1$ and $\frA_2$ ensures with high probability (ignoring lower order terms):
	\begin{align*}
		R_K = \tilO{ \min\cbr{\B S\sqrt{ALK}, \B S\sqrt{AK} + (\B^2S^2A(\Delta_c+\B\Delta_P)\T)^{1/3}{K}^{2/3}} },
	\end{align*}
where $L=1+\sum_{k=1}^{K-1}\Ind\{P_{k+1}\neq P_k \;\text{or}\; c_{k+1}\neq c_k\}$ is the number changes of the environment (plus one).
Moreover, this is achieved without the knowledge of $\Delta_c$, $\Delta_P$, or $L$.
\end{theorem}

The advantage of this result compared to \pref{thm:two phase} is two-fold.
First, it adapts to different levels of non-stationarity ($\Delta_c$, $\Delta_P$, and $L$) automatically.
Second, it additionally achieves a bound of order $\tilo{\B S\sqrt{ALK}}$, which could be much better than that in \pref{thm:two phase}; for example, when $L = \bigo{1}$, the former is a $\sqrt{K}$-order bound while the latter is of order $K^{2/3}$.
As discussed in~\citep{wei2021non}, this is a unique benefit brought by the MASTER algorithm and is not achieved by any other algorithms even with the knowledge of $L$.

The disadvantage of \pref{thm:unknown}, on the other hand, is its sub-optimality in the $\B$ dependency for the $\Delta_c$-related term and the $S$ dependency for both terms.
The extra $\B$ dependency is due to the synchronized cost and transition estimation.
As mentioned, it is unclear how to update cost and transition estimation independently as we do in \pref{alg:mvp-test} under the framework of MASTER, which we leave as an important future direction.
On the other hand, the extra $S$ dependency comes from the fact that the lower-order term in the regret bound of the base algorithm affects the final regret bound (see the statement of \pref{thm:master}).
Specifically, the lower-order term is $\B S^2A$ instead of $\B SA$, which eventually leads to extra $S$ dependency.
How to remove the extra $S$ factor in the base algorithm, or eliminate the undesirable lower-order term effect brought by the MASTER algorithm, is another important future direction.

\section{Conclusion}
In this work, we develop the first set of results for dynamic regret minimization in non-stationary SSP, including a (near) minimax optimal algorithm and two others that are either simpler or advantageous in some other cases.
Besides the immediate next step such as improving our results when the non-stationarity is unknown,
our work also opens up many other possible future directions on this topic, such as extension to 
more general settings with function approximation.
It would also be interesting to study more adaptive dynamic regret bounds in this setting. For example, our $\B$ and $\T$ are defined as the maximum optimal expected cost and hitting time over all episodes,
which is undesirable if only a few episodes admit a large optimal expected cost or hitting time.
Ideally, some kind of (weighted) average would be a more reasonable measure in these cases.

\begin{ack}
	The authors thank Aviv Rosenberg and Chen-Yu Wei for many helpful discussions.
\end{ack}

\bibliographystyle{plainnat}
\bibliography{ref}

\newpage
\appendix
%\renewcommand{\appendixpagename}{\centering \sffamily \LARGE Supplementary Materials}
%\appendixpage

{\large\textbf{Contents of Appendix}}

\startcontents[section]
\printcontents[section]{ }{1}{}

\section{Preliminaries}
\label{app:pre}
% !TEX root = main.tex

\paragraph{Extra Notations} We first define (or restate) some notations used throughout the whole Appendix.
\begin{itemize}
	\item Let $\Delta_{c,[i,j]}=\sum_{\tau=i}^{j-1}\norm{c^{\tau+1}-c^{\tau}}_{\infty}$, $\Delta_{P,[i,j]}=\sum_{\tau=i}^{j-1}\max_{s, a}\norm{P_{s, a}^{\tau+1} - P_{s, a}^{\tau}}_1$.
It is straightforward to verify that $\Delta_{c,[1,M]}=\Delta_c$ and $\Delta_{P,[1,M]}=\Delta_P$.
	\item Define $\Delta_{c, m}=\Delta_{c,[i^c_m, m]}$ and $\Delta_{P,m}=\Delta_{P,[i^P_m, m]}$, where $i^c_m$ and $i^P_m$ are the first intervals after the last resets of $\M$ and $\N$ before interval $m$ respectively.
	\item For all algorithms, denote by $\hatc^m$, $\barc^m$, $\P^m_{s, a}$, $b^m_h$, $\Np_m$, $\Mp_m$, $\iota_m$ the value of $\hatc$, $\barc$, $\P_{s, a}$, $b_h$, $\Np$, $\Mp$, $\iota$ at the beginning of interval $m$, and define $\hatc^m_h=\hatc^m(s^m_h, a^m_h)$, $\barc^m_h=\barc(s^m_h, a^m_h)$, $\N^m_h=\Np(s^m_h, a^m_h)$, and $\M^m_h=\Mp(s^m_h, a^m_h)$.
We also slightly abuse the notation and write $b^m(s^m_h, a^m_h)$ as $b^m_h$ when there is no confusion.
	\item Define $\tilc^m(s, a)=\frac{1}{\Mp_m(s, a)}\sum_{m'=i^c_m}^{m-1}\sum_{h=1}^{H_{m'}}c^{m'}(s, a)\Ind\{(s^{m'}_h, a^{m'}_h)=(s, a)\}$, $\tilc^m_h=\tilc^m(s^m_h, a^m_h)$, $\tilP^m_{s, a}=\frac{1}{\Np_m(s, a)}\sum_{m'=i^P_m}^{m-1}\sum_{h=1}^{H_{m'}}P^{m'}_{s, a}\Ind\{(s^{m'}_h, a^{m'}_h)=(s, a)\}$, $\P^m_h=\P^m_{s^m_h, a^m_h}$, and $\tilP^m_h=\tilP^m_{s^m_h, a^m_h}$.
	\item Denote by $L_{c,[i,j]}$ and $L_{P, [i,j]}$ one plus the number of resets of $\M$ and $\N$ within intervals $[i, j]$ respectively, and define $L_{c,m}=L_{c,[1,m]}$, $L_{P,m}=L_{P,[1,m]}$, $L_m=L_{c,m}+L_{P,m}$ for any $m\geq 1$.
	\item Define $f^c(m)$ (or $f^P(m)$) as the earliest interval at or after interval $m$ in which the learner resets $\M$ (or $\N$).
	\item Define $\m^m_h=\Ind\{\M^m(s^m_h, a^m_h)=0\}$, $\n^m_h=\Ind\{\N^m(s^m_h, a^m_h)=0\}$, $C_{M'}=\summp\sum_{h=1}^{H_m+1}c^m_h$, and bonus function $b^m(s, a, V)=\max\cbr{7\sqrt{\frac{\fV(\P^m_{s, a}, V)\iota_m}{\Np_m(s, a)}}, \frac{49B\sqrt{S}\iota_m}{\Np_m(s, a)}}$.
	\item Define $T^{\optpi,m}_h(s)$ (or $T^{\optpi,m}_h(s, a)$) as the hitting time (reaching $g$ or layer $H+1$) of $\optpi_{k(m)}$ starting from state $s$ (or state-action pair $(s, a)$) in layer $h$ w.r.t transition $P^m$, such that $T^{\optpi,m}_h(s, a)=1 + P^m_{s, a}T^{\optpi,m}_{h+1}$, $T^{\optpi,m}_h(s)=T^{\optpi,m}_h(s, \optpi_{k(m)}(s))$, and $T^{\optpi,m}_{H+1}(s)=T^{\optpi,m}_{H+1}(s, a)=T^{\optpi,m}_h(g)=T^{\optpi,m}_h(g, a)=0$.
	\item For notational convenience, we often write $V^{\optpi_{k(m)},m}_h$ as $V^{\optpi,m}_h$.
	\item Define $(x)_+=\max\{0, x\}$.
\end{itemize}

%Define $f^c(m)$ (or $f^P(m)$) as the last interval before the next reset of $\M$ (or $\N$) after interval $m$.

\paragraph{Optimal Value Functions of $\rcalM$} We denote by $Q^{\star,m}_h$ and $V^{\star,m}_h$ the optimal value functions in interval $m$.
It is not hard to see that they can be defined recursively as follows: $V^{\star,m}_{H+1}=c_f$ and for $h\leq H$,
\begin{align*}
	Q^{\star,m}_h(s, a) = c^m(s, a) + P^m_{s, a}V^{\star,m}_{h+1}, \qquad V^{\star,m}_h(s) = \min_aQ^{\star,m}_h(s, a).
\end{align*}
For notational convenience, we also let $Q^{\star,m}_{H+1}(s, a)=V^{\star,m}_{H+1}(s)$ for any $(s, a)\in\SA$.

\begin{lemma}
	\label{lem:bound optQ}
	For any $m\geq1$ and $h\leq H+1$, $Q^{\star,m}_h(s, a) \leq Q^{\optpi,m}_h(s, a) \leq 4\B$.
\end{lemma}
\begin{proof}
	This is simply by $Q^{\optpi,m}_h(s, a) \leq 1 + \max_s\optV_k(s) + 2\B \leq 4\B$.
\end{proof}

\paragraph{Auxiliary Lemmas} Below we provide auxiliary lemmas used throughout the whole Appendix and for all algorithms.

\begin{lemma}
	\label{lem:c diff}
	With probability at least $1-3\delta$, $\summp\sumhm(c^m(s^m_h, a^m_h) - \hatc^m_h) \leq 3\summp\sumhm\rbr{ \sqrt{\frac{\barc^m_h\iota_m}{\M^m_h}} + \frac{\iota_m}{\M^m_h}} + \summp\sumhm\Delta_{c,m} \leq \tilO{ \sqrt{SA\Lc C_{M'}} + SA\Lc } + 2\summp\sumhm\Delta_{c,m}$ and $\summp\sumhm\rbr{\sqrt{\frac{\barc^m_h\iota_m}{\M^m_h}} + \frac{\iota_m}{\M^m_h} }\leq\tilo{\sqrt{SA\Lc C_{M'}} + SA\Lc + \sqrt{SAL_{c,M'}\summp \sumhm\Delta_{c,m}}}$ for any $M'\leq M$.
\end{lemma}
\begin{proof}
	First note that by \pref{lem:freedman}, with probability at least $1-\delta$, for any $m\geq 1$ and $(s, a)\in\SA$,
	\begin{equation}
		\tilc^m(s, a) - \barc^m(s, a) \leq \sqrt{\frac{\barc^m(s^m_h, a^m_h)}{\Mp_m(s, a)}} + \frac{1}{\Mp_m(s, a)}.\label{eq:tilc-barc}
	\end{equation}
	For the first inequality in the first statement, note that
	\begin{align*}
		&\summp\sumhm(c^m(s^m_h, a^m_h) - \hatc^m_h)\\
		&\leq \summp\sumhm\rbr{\tilc^m(s^m_h, a^m_h) - \barc^m(s^m_h, a^m_h) + \sqrt{\frac{\barc^m_h\iota_m}{\M^m_h}} + \frac{\iota_m}{\M^m_h} + \m^m_h} + \summp\sumhm \Delta_{c,m} \tag{definition of $\hatc^m_h$ and $c^m(s^m_h, a^m_h) \leq \tilc^m(s^m_h, a^m_h) + \Delta_{c,m} + \m^m_h$}\\
		&\leq 3\summp\sumhm\rbr{ \sqrt{\frac{\barc^m_h\iota_m}{\M^m_h}} + \frac{\iota_m}{\M^m_h}}  + \summp\sumhm\Delta_{c,m} \tag{\pref{eq:tilc-barc} and $\m^m_h\leq\frac{1}{\M^m_h}$}.
	\end{align*}
	The second inequality in the first statement simply follows from applying AM-GM inequality on the second statement.
	To prove the second statement, first note that by \pref{lem:freedman}, Cauchy-Schwarz inequality, and \pref{lem:sum}, with probability at least $1-\delta$,
	\begin{align*}
		\summp\sumhm\barc^m_h &= \tilO{\summp\sumhm\rbr{\tilc^m_h + \sqrt{\frac{\barc^m_h}{\M^m_h}} + \frac{1}{\M^m_h}} }\\
		&= \tilO{\summp\sumhm\tilc^m_h + \sqrt{SA\Lc\summp\sumhm\barc^m_h} + SA\Lc}.
	\end{align*}
	Solving a quadratic inequality w.r.t $\summp\sumhm\barc^m_h$ (\pref{lem:quad}) gives $\summp\sumhm \barc^m_h = \tilo{\summp\sumhm \tilc^m_h + SA\Lc}$.
	Therefore, with probability at least $1-\delta$,
	\begin{align*}
		&\summp\sumhm \rbr{\sqrt{\frac{\barc^m_h\iota_m}{\M^m_h}} + \frac{\iota_m}{\M^m_h} } = \tilO{\sqrt{SAL_{c,M'}\summp\sumhm\barc^m_h} + SAL_{c,M'}} \tag{Cauchy-Schwarz inequality and \pref{lem:sum}} \\
		&=\tilO{ \sqrt{SAL_{c,M'}\summp\sumhm\tilc^m_h} + SAL_{c,M'}}\\
		&=\tilO{ \sqrt{SAL_{c,M'}\summp \sumhm\Delta_{c,m}} + \sqrt{SAL_{c,M'}\summp\sumhm c^m(s^m_h, a^m_h)} + SAL_{c,M'}}\\
		&= \tilO{ \sqrt{SAL_{c,M'}\summp \sumhm\Delta_{c,m}} + \sqrt{SAL_{c,M'}C_{M'}} + SAL_{c,M'}}. \tag{\pref{lem:e2r}}
	\end{align*}
	This completes the proof.
\end{proof}

%\begin{lemma}
%	\label{lem:barc2tilc}
%	With probability at least $1-\delta$, $\summp\sumhm \barc^m_h = \tilo{\summp\sumhm \tilc^m_h + SA\Lc}$ and $\summp\sumhm\rbr{\sqrt{\frac{\barc^m_h}{\M^m_h}} + \frac{1}{\M^m_h} }\leq\tilo{SA\Lc + \sqrt{SA\Lc C_{M'}}} + \frac{1}{64}\summp\sumhm\Delta_{c,m}$ for any $M'\geq 1$.
%\end{lemma}
%\begin{proof}
%	By \pref{lem:freedman}, Cauchy-Schwarz inequality, and \pref{lem:sum}, with probability at least $1-\delta$,
%	\begin{align*}
%		\summp\sumhm\barc^m_h &= \tilO{\summp\sumhm\rbr{\tilc^m_h + \sqrt{\frac{\barc^m_h}{\M^m_h}} + \frac{1}{\M^m_h}} }\\
%		&= \tilO{\summp\sumhm\tilc^m_h + \sqrt{SA\Lc\summp\sumhm\barc^m_h} + SA\Lc}.
%	\end{align*}
%	Solving a quadratic inequality w.r.t $\summp\sumhm\barc^m_h$ (\pref{lem:quad}) proves the first statement.
%	For the second statement,
%	\begin{align*}
%		&\summp\sumhm \rbr{\sqrt{\frac{\barc^m_h}{\M^m_h}} + \frac{1}{\M^m_h} } \leq \sqrt{SAL_c\summp\sumhm\barc^m_h} + \tilO{SAL_{c,M'}} \tag{Cauchy-Schwarz inequality, \pref{lem:sum}} \\
%		&=\tilO{ \sqrt{SAL_{c,M'}\summp\sumhm\tilc^m_h} + SAL_{c,M'}}\tag{the first statement}\\
%		&=\tilO{ \sqrt{SAL_{c,M'}\summp H\Delta_{c,m}} + \sqrt{SAL_{c,M'}\summp\sumhm c^m(s^m_h, a^m_h)} + SAL_{c,M'}}\\
%		&= \tilO{SAHL_{c,M'} + \sqrt{SAL_{c,M'}C_{M'}}} + \frac{1}{64}\summp\sumhm\Delta_{c,m}. \tag{AM-GM inequality and \pref{lem:e2r}}
%	\end{align*}
%\end{proof}

\begin{lemma}
	\label{lem:P diff}
	With probability at least $1-\delta$, for any $m\geq 1$, $(s,a)\in\SA$ and $s'\in\calS_+$,
	$\abr{\tilP^m_{s, a}(s') - \P^m_{s, a}(s')} \leq \sqrt{\frac{\tilP^m_{s, a}(s')\iota_m}{2\Np_m(s, a)}} + \frac{\iota_m}{2\Np_m(s, a)} \leq \sqrt{\frac{\P^m_{s, a}(s')\iota_m}{\Np_m(s, a)}} + \frac{\iota_m}{\Np_m(s, a)}$.
	%\lesssim \sqrt{\frac{P^m_h(s')}{N^m_h}} + \frac{1}{N^m_h} + \sqrt{\frac{\Delta_P}{N^m_h}}$.
\end{lemma}
\begin{proof}
	The first inequality hold with probability at least $1-\delta/2$ by applying \pref{lem:freedman} for each $(s, a)\in\SA$ and $s'\in\calS_+$.
	Also by \pref{lem:e2r}, we have $\tilP^m_{s, a}(s')\leq 2\P^m_{s, a}(s') + \frac{\iota_m}{2\Np_m(s, a)}$ for any $(s, a)\in\SA, s'\in\calS_+$ with probability at least $1-\delta/2$.
	Substituting this back and applying $\sqrt{a+b}\leq\sqrt{a}+\sqrt{b}$ proves the second inequality.
\end{proof}

\begin{lemma}
	\label{lem:dc}
	With probability at least $1-\delta$, for any $(s, a)\in\SA$ and $m\geq 1$, $\hatc^m(s, a) \leq c^m(s, a) + \Delta_{c,m}$.
\end{lemma}
\begin{proof}
	For any $(s, a)$ and $m\geq 1$, when $\M_m(s, a)=0$, the statement clearly holds since $\barc^m(s, a)=0$.
	Otherwise, by \pref{lem:freedman} and \pref{lem:e2r}, with probability at least $1-\delta$, for all $(s, a)$ and $m\geq 1$ simultaneously,
	\begin{align}
		&|\barc^m(s, a) - \tilc^m(s, a)| \leq 3\sqrt{\frac{\tilc^m(s, a)}{\Mp_m(s, a)}\ln\frac{32SAm^5}{\delta}} + \frac{2\ln\frac{32SAm^5}{\delta}}{\Mp_m(s, a)}\notag\\
		&\leq 3\sqrt{\frac{\rbr{2\barc^m(s, a) + \frac{12\ln\frac{4SAm}{\delta}}{\Mp_m(s, a)}}}{\Mp_m(s, a)}\ln\frac{32SAm^5}{\delta}} + \frac{2\ln\frac{32SAm^5}{\delta}}{\Mp_m(s, a)} \leq \sqrt{\frac{\barc^m(s, a)\iota_m}{\Mp_m(s, a)}} + \frac{\iota_m}{\Mp_m(s, a)}.\label{eq:barc-tilc}
	\end{align}
	Therefore, by $\max\{0, a\}-\max\{0, b\}\leq\max\{0, a-b\}$,
	\begin{align*}
		&\hatc^m(s, a) - c^m(s, a) \leq \hatc^m(s, a) - \tilc^m(s, a) + \Delta_{c,m}\\
		&\leq \max\cbr{0, \barc^m(s, a) - \tilc^m(s, a) - \sqrt{\frac{\barc^m(s, a)\iota_m}{\Mp_m(s, a)}} - \frac{\iota_m}{\Mp_m(s, a)} } + \Delta_{c,m} \leq \Delta_{c,m},
	\end{align*}
	where the last step is by \pref{eq:barc-tilc}.
\end{proof}

\begin{lemma}
	\label{lem:dPv}
	Given function $V \in [-B, B]^{\calS_+}$ for some $B>0$, we have with probability at least $1-\delta$,
	$|(\tilP^m_{s, a} - \P^m_{s, a})V| \leq \tilO{\sqrt{\frac{S\fV(P^m_{s, a}, V)}{\Np_m(s, a)}} + \frac{SB}{\Np_m(s, a)}} + \frac{B\Delta_{P,m}}{64}$ for any $m\geq 1$.
\end{lemma}
\begin{proof}
	Note that with probability at least $1-\delta$,
	\begin{align*}
		|(\tilP^m_{s, a} - \P^m_{s, a})V| &= |(\tilP^m_{s, a} - \P^m_{s, a})(V-P^m_{s, a}V)|\\ 
		&= \tilO{\sum_{s'}\rbr{\sqrt{\frac{\tilP^m_{s, a}(s')}{\Np_m(s, a)}}|V(s') - P^m_{s, a}V| + \frac{B}{\Np_m(s, a)}} } \tag{\pref{lem:P diff}}\\
		&= \tilO{\sqrt{\frac{S\tilP^m_h(V - P^m_{s, a}V)^2}{\Np_m(s, a)}} + \frac{SB}{\Np_m(s, a)}}\tag{Cauchy-Schwarz inequality}\\
		&= \tilO{\sqrt{\frac{SP^m_h(V - P^m_{s, a}V)^2}{\Np_m(s, a)}} + \frac{SB}{\Np_m(s, a)} + B\sqrt{\frac{S\Delta_{P,m}}{\Np_m(s, a)}} }.
	\end{align*}
	Applying AM-GM inequality completes the proof.
\end{proof}

\begin{lemma}
	\label{lem:emp var}
	With probability at least $1-\delta$, $\fV(\P^m_h, V^m_{h+1}) \leq 2\fV(P^m_h, V^m_{h+1}) + \tilO{ \frac{SB^2}{\N^m_h} } + 2B^2\Delta_{P,m}$ for any $m\geq 1$.
\end{lemma}
\begin{proof}
	Note that:
	\begin{align*}
		\fV(\P^m_h, V^m_{h+1}) &\leq \P^m_h(V^m_{h+1} - P^m_hV^m_{h+1})^2 \tag{$\frac{\sum_ip_ix_i}{\sum_ip_i}=\argmin_z\sum_ip_i(x_i-z)^2$}\\
		&= \fV(P^m_h, V^m_{h+1}) + (\P^m_h-P^m_h)(V^m_{h+1} - P^m_hV^m_{h+1})^2\\
		&\leq \fV(P^m_h, V^m_{h+1}) + (\P^m_h-\tilP^m_h)(V^m_{h+1}-P^m_hV^m_{h+1})^2 + B^2\Delta_{P,m}\\
		&\leq \fV(P^m_h, V^m_{h+1}) + \tilO{B\sqrt{\frac{S\tilP^m_h(V^m_{h+1}-P^m_hV^m_{h+1})^2}{\N^m_h}} + \frac{SB^2}{\N^m_h} } + B^2\Delta_{P,m} \tag{\pref{lem:P diff} and Cauchy-Schwarz inequality}\\
		&\leq \fV(P^m_h, V^m_{h+1}) + \tilO{B\sqrt{\frac{S\fV(P^m_h, V^m_{h+1})}{\N^m_h}} + B^2\sqrt{\frac{S\Delta_{P,m}}{\N^m_h}} + \frac{SB^2}{\N^m_h}} + B^2\Delta_{P,m}\\
		&\leq 2\fV(P^m_h, V^m_{h+1}) + \tilO{\frac{SB^2}{\N^m_h}} + 2B^2\Delta_{P,m}. \tag{AM-GM inequality}
	\end{align*}
\end{proof}

\begin{lemma}
	\label{lem:dPV}
	Given an oblivious set of value functions $\calV$ with $|\calV|\leq (2HK)^6$ and $\norm{V}_{\infty}\leq B$ for any $V\in\calV$, we have with probability at least $1-\delta$, for any $V\in\calV$, $(s, a)\in\SA$, and $m\geq 1$, $|(\P^m_{s,a} - \tilP^m_{s,a})V| \leq \sqrt{\frac{\fV(P^m_{s, a}, V)\iota_m}{\Np_m(s, a)}} + \frac{17B\iota_m}{\Np_m(s, a)} + \frac{B\Delta_{P,m}}{64}$ and $|(\P^m_{s,a} - \tilP^m_{s,a})V| \leq \sqrt{\frac{2\fV(\P^m_{s, a}, V)\iota_m}{\Np_m(s, a)}} + \frac{3B\sqrt{S}\iota_m}{\Np_m(s, a)}$.
\end{lemma}
\begin{proof}
	For each $(s, a)\in\SA$ and $V\in\calV$, by \pref{lem:freedman}, with probability at least $1-\frac{\delta}{2SA(2HK)^6}$, for any $m\geq 1$
	\begin{equation}
		|(\P^m_{s,a} - \tilP^m_{s,a})V| \leq \frac{1}{\Np_m(s, a)}\rbr{\sqrt{\sum_{i=1}^{\N_m(s, a)}\fV(P^{m_i}_{s, a}, V)\iota_m} + B\iota_m}.\label{eq:bar til PV}
	\end{equation}
	Denote by $m_i$ the interval where the $i$-th visits to $(s, a)$ lies in among those $\N_m(s, a)$ visits, we have
	\begin{align*}
		&\frac{1}{\Np_m(s, a)}\sum_{i=1}^{\N_m(s, a)}\fV(P^{m_i}_{s, a}, V) = \frac{1}{\Np_m(s, a)}\sum_{i=1}^{\N_m(s, a)}P^{m_i}_{s, a}(V - P^{m_i}_{s, a}V)^2\\
		&\leq \frac{1}{\Np_m(s, a)}\sum_{i=1}^{\N_m(s, a)}P^{m_i}_{s, a}(V - P^m_{s, a}V)^2 \leq \fV(P^m_{s, a}, V) + B^2\Delta_{P,m},
	\end{align*}
	where the second last inequality is by $\frac{\sum_ip_ix_i}{\sum_ip_i}=\argmin_z\sum_ip_i(x_i-z)^2$.
	Thus by \pref{eq:bar til PV},
	\begin{align*}
		|(\P^m_{s,a} - \tilP^m_{s,a})V| &\leq \sqrt{\frac{\fV(P^m_{s, a}, V)\iota_m}{\Np_m(s, a)}} + \frac{B\iota_m}{\Np_m(s, a)} + B\sqrt{\frac{\Delta_{P,m}\iota_m}{\Np_m(s, a)}}\\
		&\leq \sqrt{\frac{\fV(P^m_{s, a}, V)\iota_m}{\Np_m(s, a)}} + \frac{17B\iota_m}{\Np_m(s, a)} + \frac{B\Delta_{P,m}}{64}. \tag{AM-GM inequality}
	\end{align*}
	Moreover, again by $\frac{\sum_ip_ix_i}{\sum_ip_i}=\argmin_z\sum_ip_i(x_i-z)^2$,
	\begin{align*}
		&\frac{1}{\Np_m(s, a)}\sum_{i=1}^{\N_m(s, a)}\fV(P^{m_i}_{s, a}, V) \leq \frac{1}{\Np_m(s, a)}\sum_{i=1}^{\N_m(s, a)}P^{m_i}_{s, a}(V - \P^m_{s, a}V)^2\\
		&\leq \fV(\P^m_{s, a}, V) + (\tilP^m_{s, a} - \P^m_{s, a})(V-\P^m_{s, a}V)^2 \leq \fV(\P^m_{s, a}, V) + B\sqrt{\frac{S\fV(\P^m_{s, a}, V)\iota_m}{\Np_m(s, a)}} + \frac{SB^2\iota_m}{\Np_m(s, a)} \tag{\pref{lem:P diff} and Cauchy-Schwarz inequality}\\
		&\leq 2\fV(\P^m_{s, a}, V) + \frac{2SB^2\iota_m}{\Np_m(s, a)}. \tag{AM-GM inequality}
	\end{align*}
	Thus by \pref{eq:bar til PV}, $|(\P^m_{s,a} - \tilP^m_{s,a})V| \leq \sqrt{\frac{2\fV(\P^m_{s, a}, V)\iota_m}{\Np_m(s, a)}} + \frac{3B\sqrt{S}\iota_m}{\Np_m(s, a)}$.
\end{proof}

\begin{lemma}
	\label{lem:sum var}
	For any sequence of value functions $\{V^m_h\}_{m, h}$ with $\norm{V^m_h}_{\infty}\in [0, B]$, we have with probability at least $1-\delta$, for all $M'\geq 1$, 
	$\summp\sumhm\fV(P^m_h, V^m_{h+1}) = \tilO{\summp V^m_{H_m+1}(s^m_{H_m+1})^2 + \summp\sumhm B(V^m_h(s^m_h) - P^m_hV^m_{h+1})_+ + B^2} $.
\end{lemma}
\begin{proof}
	We decompose the sum of variance as follows:
	\begin{align*}
		&\summp\sumhm\fV(P^m_h, V^m_{h+1}) = \summp\sumhm\rbr{P^m_h(V^m_{h+1})^2 - V^m_{h+1}(s^m_{h+1})^2}\\
		&\qquad + \summp\sumhm\rbr{ V^m_{h+1}(s^m_{h+1})^2 - V^m_h(s^m_h)^2 } + \summp\sumhm\rbr{V^m_h(s^m_h)^2 - (P^m_hV^m_{h+1})^2}.
		%&\leq\tilO{B\sqrt{\summp\sumhm\fV(P^m_h, V^m_{h+1})} + B^2} + \summp V^m_{H_m+1}(s^m_{H_m+1})^2\\
		%&\qquad + 2B\summp\sumhm (V^m_h(s^m_h) - P^m_hV^m_{h+1})_+. \tag{\pref{lem:freedman}, \pref{lem:var X2}, and $a^2-b^2\leq(a+b)(a-b)_+$}
	\end{align*}
	For the first term, by \pref{lem:freedman} and \pref{lem:var X2}, with probability at least $1-\delta$,
	\begin{align*}
		\summp\sumhm\rbr{P^m_h(V^m_{h+1})^2 - V^m_{h+1}(s^m_{h+1})^2} &= \tilO{ \sqrt{\summp\sumhm\fV(P^m_h, (V^m_{h+1})^2 )} + B^2 }\\ 
		&= \tilO{B\sqrt{\summp\sumhm\fV(P^m_h, V^m_{h+1})} + B^2}.
	\end{align*}
	The second term is clearly upper bounded by $\summp V^m_{H_m+1}(s^m_{H_m+1})^2$, and the third term is upper bounded by $2B\summp\sumhm (V^m_h(s^m_h) - P^m_hV^m_{h+1})_+$ by $a^2-b^2\leq(a+b)(a-b)_+$.
	Putting everything together and solving a quadratic inequality (\pref{lem:quad}) w.r.t $\summp\sumhm\fV(P^m_h, V^m_{h+1})$ completes the proof.
\end{proof}

\begin{lemma}
	\label{lem:sum bV}
	For any value functions $\{V^m_h\}_{m, h}$ such that $\norm{V^m_h}_{\infty}\leq B$, with probability at least $1-\delta$, for any $M'\geq 1$,
	\begin{align*}
		&\summp\sumhm b^m(s^m_h, a^m_h, V^m_{h+1})\\
		&=\tilO{ \sqrt{SA\Lp\summp\sumhm\fV(P^m_h, V^m_{h+1})} + BS^{1.5}A\Lp + B\sqrt{SA\Lp\summp\sumhm\Delta_{P,m} } }.
	\end{align*} 
\end{lemma}
\begin{proof}
	Note that:
	\begin{align*}
		&\summp\sumh b^m(s^m_h, a^m_h, V^m_{h+1}) = \tilO{\summp\sumhm\rbr{\sqrt{\frac{\fV(\P^m_h, V^m_{h+1})}{\N^m_h}} + \frac{B\sqrt{S}}{\N^m_h}}}\\
		&=\tilO{ \sqrt{SA\Lp\summp\sumhm\fV(\P^m_h, V^m_{h+1})} + BS^{1.5}A\Lp}\tag{Cauchy-Schwarz inequality and \pref{lem:sum}}\\
		&= \tilO{\sqrt{SA\Lp\summp\sumhm\fV(P^m_h, V^m_{h+1})} + BS^{1.5}A\Lp + B\sqrt{SA\Lp\summp\sumhm\Delta_{P,m}} }. \tag{\pref{lem:emp var}, \pref{lem:sum}, and $\sqrt{a+b}\leq\sqrt{a}+\sqrt{b}$}
	\end{align*}
\end{proof}

\begin{lemma}
	\label{lem:sum}
	For any $M'\geq 1$, $\summp\sumhm\frac{1}{\M^m_h}=\tilo{SA\Lc}$ and $\summp\sumhm\frac{1}{\N^m_h}=\tilo{SA\Lp}$.
\end{lemma}
\begin{proof}
	This simply follows from the fact that the sum of $\frac{1}{\M^m_h}$ (or $\frac{1}{\N^m_h}$) between consecutive resets of $\M^m_h$ (or $\N^m_h$) is of order $\tilo{SA}$.
\end{proof}

\begin{lemma}
	\label{lem:new interval}
	$\summp\Ind\{H_m<H, s^m_{H_m+1}\neq g\} = \tilo{SAL_{M'}}$ for any $M'\leq M$.
\end{lemma}
\begin{proof}
	This simply follows from the fact that between consecutive resets of $\M$ or $\N$, the number of times that the number of visits to some $(s, a)$ is doubled is $\tilo{SA}$.
\end{proof}

\begin{lemma}
	\label{lem:bound l}
	Suppose $r(m)=\min\{\frac{c_1}{\sqrt{m}} + c_2, c_3\}$, $\Delta\in \fR_+^{\fN_+}$ is a non-stationarity measure, and define $\Delta_{[i,j]}=\sum_{i=1}^{j-1}\Delta(i)$.
	If for a given interval $\calJ$, there is a way to partition $\calJ$ into $\ell$ intervals $\{\calI_i\}_{i=1}^{\ell}$ with $\calI_i=[s_i, e_i]$ such that $\Delta_{[s_i, e_i+1]}>r(|\calI_i|+1)$ for $i\leq\ell-1$ (note that $|\calI_i|=e_i-s_i+1$), then $\ell \leq 1 + (2c_1^{-1}\Delta_{\calJ})^{2/3}|\calJ|^{1/3} + c_3^{-1}\Delta_{\calJ}$.
\end{lemma}
\begin{proof}
	Note that
	\begin{align*}
		\Delta_{\calJ}&\geq \sum_{i=1}^{\ell-1}\Delta_{[s_i, e_i + 1]} > \sum_{i=1}^{\ell-1} r(|\calI_i|+1) \geq \sum_{i=1}^{\ell-1}\min\cbr{c_1(|\calI_i|+1)^{-1/2}, c_3}\\
		&\geq \sum_{i=1}^{\ell-1}\min\cbr{\frac{c_1}{2}|\calI_i|^{-1/2}, c_3} = \sum_{i=1}^{\ell_1}\frac{c_1}{2}|\calI_i|^{-1/2} + \ell_2c_3,
	\end{align*}
	where in the last step we assume $|\calI_i|$ is decreasing in $i$ without loss of generality and $\ell_1+\ell_2=\ell-1$.
	The inequality above implies $\ell_2\leq c_3^{-1}\Delta_{\calJ}$ and
	\begin{align*}
		\ell_1 = \sum_{i=1}^{\ell_1}|\calI_i|^{-\frac{1}{3}}|\calI_i|^{\frac{1}{3}} \leq \rbr{\sum_{i=1}^{\ell_1}|\calI_i|^{-1/2}}^{\frac{2}{3}}\rbr{\sum_{i=1}^{\ell_1}|\calI_i|}^{\frac{1}{3}} \leq \rbr{\frac{2\Delta_{\calJ}}{c_1}}^{\frac{2}{3}}|\calJ|^{\frac{1}{3}} \tag{H\"older's inequality with $p=\frac{3}{2}$ and $q=3$}
	\end{align*}
	Combining them completes the proof.
\end{proof}

\section{Omitted Details in \pref{sec:lb}}
\label{app:lb}
% !TEX root = main.tex

In this section we provide omitted proofs and discussions in \pref{sec:lb}.

\subsection{Optimal Value Change w.r.t Non-stationarity}
\label{app:optV diff}

Below we provide a bound on the change of optimal value functions w.r.t cost and transition non-stationarity.

\begin{lemma}
	\label{lem:value bound}
	For any $k_1,k_2\in[K]$, $\optV_{k_1}(\sinit) - \optV_{k_2}(\sinit)\leq (\Delta_c+\B\Delta_P)\T$.
\end{lemma}
\begin{proof}
	Denote by $\optq_{k_2}(s, a)$ (or $\optq_{k_2}(s)$) the number of visits to $(s, a)$ (or $s$) before reaching $g$ following $\optpi_{k_2}$.
	By the extended value difference lemma \citep[Lemma 1]{efroni2020optimistic} (note that their result is for finite-horizon MDP, but the nature generalization to SSP holds), we have
	\begin{align*}
		&\optV_{k_1}(\sinit) - \optV_{k_2}(\sinit)\\ 
		&= \sum_s\optq_{k_2}(s)(\optV_{k_1}(s) - \optQ_{k_1}(s, \optpi_{k_2}(s))) + \sum_{s, a}\optq_{k_2}(s, a)( \optQ_{k_1}(s, a) - c_{k_2}(s, a) - P_{k_2,s,a}\optV_{k_1} )\\
		&\leq \sum_{s, a}\optq_{k_2}(s, a)( c_{k_1}(s, a) - c_{k_2}(s, a) + (P_{k_1,s,a}-P_{k_2,s,a})\optV_{k_1} ) \leq (\Delta_c+\B\Delta_P)\T.
		%&\leq \sum_{s, a}\optq_{k_2}(s, a)(\Delta_c + \B\Delta_P) \leq (\Delta_c+\B\Delta_P)\T,
	\end{align*}
	where in the last inequality we apply $\norm{c_{k_1}-c_{k_2}}_{\infty}\leq\Delta_c$,$(P_{k_1,s,a}-P_{k_2,s,a})\optV_{k_1}\leq\max_{s, a}\norm{P_{k_1,s,a}-P_{k_2,s,a}}_1\norm{ \optV_{k_1} }_{\infty}\leq \B\Delta_P$, and $\sum_{s, a}\optq_{k_2}(s, a)\leq \T$.
\end{proof}

We also give an example showing that the bound in \pref{lem:value bound} is tight up to a multiplication factor.
Consider an SSP instance with only one state $\sinit$ and one action $a_g$, such that $c(\sinit, a_g)=\frac{\B}{\T}$,  $P(g|\sinit, a_g)=\frac{1}{\T}$, and $P(\sinit|\sinit, a_g)=1-P(g|\sinit, a_g)$ with $1\leq\B\leq\T$.
The optimal value of this instance is clearly $\B$.
Now consider another SSP instance with perturbed cost function $c'(\sinit, a_g)=\frac{\B}{\T}+\Delta_c$ and perturbed transition function $P'(g|\sinit, a_g)=\frac{1}{\T}-\frac{\Delta_P}{2}$, $P'(\sinit|\sinit, a_g)=1-P'(g|\sinit, a_g)$ with $\max\{\Delta_c, \Delta_P\}\leq\frac{1}{\T}$.
The optimal value function in this instance is
\begin{align*}
	\frac{\frac{\B}{\T}+ \Delta_c}{\frac{1}{\T}-\frac{\Delta_P}{2}} &= \frac{\B + \T\Delta_c}{1 - \frac{\T\Delta_P}{2}} \leq (\B+\T\Delta_c)(1 + \T\Delta_P) = \B + (\Delta_c+\B\Delta_P)\T + \T^2\Delta_c\Delta_P\\
	&\leq \B + 2(\Delta_c+\B\Delta_P)\T,
\end{align*}
where in the first inequality we apply $\frac{1}{1-x}\leq 1+2x$ for $x\in[0, \frac{1}{2}]$.
Thus the optimal value difference between these two SSPs is of the same order of the upper bound in \pref{lem:value bound}.

\subsection{\pfref{thm:lb}}
\label{app:pf lb}

For the ease of analysis, in this section we consider SSP instances with different action set at different state similar to \citep{chen2021minimax}.
The meaning of $SA$ is still the total number of state-action pairs in the SSP instance.

For any $\B, \T, SA, K$ with $\B\geq 1$, $\T\geq 3\B$, and $K\geq SA\geq 10$, we define a set of SSP instances $\{\calM^K_{i,j}\}_{i,j}$ with $i,j\in\{0,1,\ldots,N\}$ and $N=SA$.
The instance $\calM^K_{\istar, \jstar}$ is constructed as follows:
\begin{itemize}
	\item There are $N+1$ states $\{\sinit, s_1,\ldots,s_N\}$.
	\item At $\sinit$, there are $N$ actions $a_1,\ldots,a_N$; at $s_i$ for $i\in[N]$ there is only one action $a_g$.
	\item $c(\sinit, a_i)=0$ and $c(s_i, a_g)\sim\bernoulli(\frac{\B+\epsilon_{c,K}\Ind\{i\neq\istar\}}{\T})$ for $i\in[N]$, where $\epsilon_{c,K}=\frac{1-1/N}{4}\sqrt{N\B/K}$.
	\item $P(s_i|\sinit, a_i)=1$, $P(g|s_j, a_g)=\frac{1+\epsilon_{P,K}\Ind\{j=\jstar\}}{\T}$, and $P(s_j|s_j, a_g)=1-P(g|s_j, a_g)$, where $\epsilon_{P,K} = \frac{1-1/N}{4}\sqrt{N/K}$.
\end{itemize}

Note that for any $\calM^K_{i,j}$, the expected hitting time is upper bounded by $\T+1$, the expected cost of optimal policy is upper bounded by $2\B$, and the number of state-action pairs is upper bounded by $2N$.
We then use $\{\calM^K_{i,j}\}_{i,j}$ to prove static regret lower bounds (note that static regret and dynamic regret are the same without non-stationarity, that is, $\Delta_c=\Delta_P=0$) based on cost perturbation and transition perturbation respectively, which serve as the cornerstones of the proof of \pref{thm:lb}.

\begin{theorem}
	\label{thm:lb cost}
	For any $\B, \T, SA, K$ with $\B\geq 1$, $\T\geq 3\B$, $K\geq SA\geq 10$, and any learner, there exists an SSP instance based on cost perturbation such that the regret of the learner after $K$ episodes is at least $\lowo{\sqrt{\B SAK}}$.
\end{theorem}
\begin{proof}
	Consider a distribution of SSP instances which is uniform over $\{\calM^K_{i, 0}\}_i$ for $i\in[N]$.
	Let $\E_i$ be the expectation w.r.t $\calM^K_{i, 0}$, $P_i$ be the distribution of learner's observations w.r.t $\calM^K_{i, 0}$, and $K_i$ the number of visits to state $i$ in $K$ episodes.
	Also let $\epsilon_c=\epsilon_{c,K}$.
	The expected regret over this distribution of SSPs can be lower bounded as
	\begin{align*}
		\E[R_K] &= \frac{1}{N}\sum_{i=1}^N\E_i[R_K] \geq \frac{1}{N}\sum_{i=1}^N\E_i[K-K_i]\epsilon_c = \epsilon_c\rbr{K - \frac{1}{N}\sum_{i=1}^N\E_i[K_i]}.
	\end{align*}
	Note that $\calM^K_{0, 0}$ has no ``good'' state.
	By Pinsker's inequality:
	\begin{align*}
		\E_i[K_i] - \E_0[K_i] \leq K\norm{P_i - P_0}_1 \leq K\sqrt{2\KL(P_0, P_i)}.
	\end{align*}
	By the divergence decomposition lemma \citep[Lemma~15.1]{lattimore2020bandit}, we have:
	% KL divergence: https://osf.io/aqcjh/download
	\begin{align*}
		\KL(P_0, P_i) &= \E_0[K_i]\cdot \T\cdot\KL(\bernoulli((\B+\epsilon_c)/\T), \bernoulli(\B/\T))\\
		&\leq \E_0[K_i]\cdot\T\cdot \frac{\epsilon_c^2/\T^2}{\frac{\B}{\T}(1-\frac{\B}{\T})} \leq \frac{2\epsilon_c^2}{\B}\E_0[K_i]. \tag{\citep[Lemma 6]{gerchinovitz2016refined}}
	\end{align*}
	Therefore, by Cauchy-Schwarz inequality,
	\begin{align*}
		\sum_{i=1}^N\E_i[K_i] \leq \sum_{i=1}^N\rbr{ \E_0[K_i] + 2\epsilon_cK\sqrt{\E_0[K_i]/\B} } \leq K + 2\epsilon_cK\sqrt{NK/\B}.
	\end{align*}
	Plugging this back and by the definition of $\epsilon_c$, we obtain
	\begin{align*}
		\E[R_K] \geq \epsilon_cK\rbr{1 - \frac{1}{N} - 2\epsilon_c\sqrt{\frac{K}{N\B}}} = \frac{(1-1/N)^2}{8}\sqrt{\B NK}=\Omega(\sqrt{\B SAK}).
	\end{align*}
	%Now let $\epsilon=\frac{1-1/N}{4}\sqrt{N\B/K}$, we have $\E[R_K]\geq \frac{(1-1/N)^2}{16}\sqrt{\B NK}=\Omega(\sqrt{\B SAK})$.
	This completes the proof.
\end{proof}

\begin{theorem}
	\label{thm:lb transition}
	For any $\B, \T, SA, K$ with $\B\geq 1$, $\T\geq 3\B$, $K\geq SA\geq 10$, and any learner, there exists an SSP instance based on transition perturbation such that the regret of the learner after $K$ episodes is at least $\lowo{\B\sqrt{SAK}}$.
\end{theorem}
\begin{proof}
	Consider a distribution of SSP instances which is uniform over $\{\calM^K_{0, j}\}_j$ for $j\in[N]$.
	Let $\E_j$ be the expectation w.r.t $\calM^K_{0, j}$, $P_j$ be the distribution of learner's observations w.r.t $\calM^K_{0, j}$, and $K_j$ the number of visits to state $j$ in $K$ episodes.
	Also let $\epsilon_P=\epsilon_{P,K}$.
	The expected regret over this distribution of SSPs can be lower bounded as
	\begin{align*}
		\E[R_K] &= \frac{1}{N}\sum_{j=1}^N\E_j[R_K] \geq \frac{1}{N}\sum_{j=1}^N\E_j[K-K_j]\cdot\B\rbr{1-\frac{1}{1+\epsilon_P}}\\
		&\geq \frac{\B\epsilon_P}{2}\rbr{K - \frac{1}{N}\sum_{j=1}^N\E_j[K_j]}.
	\end{align*}
	Note that $\calM^K_{0, 0}$ has no ``good'' state.
	By Pinsker's inequality:
	\begin{align*}
		\E_j[K_j] - \E_0[K_j] \leq K\norm{P_j - P_0}_1 \leq K\sqrt{2\KL(P_0, P_j)}.
	\end{align*}
	By the divergence decomposition lemma \citep[Lemma~15.1]{lattimore2020bandit}, we have:
	\begin{align*}
		\KL(P_0, P_j) &= \E_0[K_j]\cdot\KL(\geo(1/\T), \geo((1+\epsilon_P)/\T))\\
		&= \E_0[K_j]\cdot \T\cdot\KL(\bernoulli(1/\T), \bernoulli((1+\epsilon_P)/\T))\\
		&\leq \E_0[K_j]\cdot\T\cdot \frac{\epsilon_P^2/\T^2}{\frac{1+\epsilon_P}{\T}(1-\frac{1+\epsilon_P}{\T})} \leq 2\epsilon_P^2\E_0[K_j]. \tag{\citep[Lemma 6]{gerchinovitz2016refined} and $\epsilon_P\leq\frac{1}{4}$}
	\end{align*}
	Therefore, by Cauchy-Schwarz inequality,
	\begin{align*}
		\sum_{j=1}^N\E_j[K_j] \leq \sum_{j=1}^N\rbr{ \E_0[K_j] + 2\epsilon_P K\sqrt{\E_0[K_j]} } \leq K + 2\epsilon_P K\sqrt{NK}.
	\end{align*}
	Plugging this back and by the definition of $\epsilon_P$, we obtain
	\begin{align*}
		\E[R_K] \geq \frac{\B\epsilon_P K}{2}\rbr{1 - \frac{1}{N} - 2\epsilon_P\sqrt{\frac{K}{N}}} \geq \frac{(1-1/N)^2}{16}\B\sqrt{NK}=\lowo{\B\sqrt{SAK}}.
	\end{align*}
	%Now let $\epsilon=\frac{1-1/N}{4}\sqrt{N/K}$, we have $\E[R_K]\geq \frac{(1-1/N)^2}{16}\B\sqrt{NK}=\Omega(\B\sqrt{SAK})$.
	This completes the proof.
\end{proof}

Now we are ready to prove \pref{thm:lb}.
\begin{proof}[\pfref{thm:lb}]
	We construct a hard non-stationary SSP instance as follows: we divide $K$ episodes into $L=L_c+L_P$ epochs.
	Each of the first $L_c$ epochs has length $\frac{K}{2L_c}$, and the corresponding SSP is uniformly sampled from $\{\calM^{K/(2L_c)}_{i,0}\}_{i\in[N]}$ independently; each of the last $L_P$ epochs has length $\frac{K}{2L_P}$, and the corresponding SSP is uniformly sampled from $\{\calM^{K/(2L_P)}_{0,j}\}_{j\in[N]}$ independently.
	%Between consecutive epochs the change in transition or cost function is at most $\frac{4\epsilon_c}{\T}$ or $\frac{4\epsilon_P}{\T}$.
	By \pref{thm:lb cost} and \pref{thm:lb transition}, the regrets in each of the first $L_c$ epochs and each of the last $L_P$ epochs are of order $\lowo{\sqrt{\B SAK/L_c}}$ and $\lowo{\B\sqrt{SAK/L_P}}$ respectively.
	Moreover, the total change in cost and transition functions are upper bounded by $\frac{\epsilon_cL_c}{\T}$ and $\frac{2\epsilon_PL_P}{\T}$ respectively with $\epsilon_c=\epsilon_{c,\frac{K}{2L_c}}$ and $\epsilon_P=\epsilon_{P,\frac{K}{2L_P}}$.
	Now let $\frac{\epsilon_cL_c}{\T}=\Delta_c$ and $\frac{2\epsilon_PL_P}{\T}=\Delta_P$, we have $L_c=(\frac{4\Delta_c\T}{1-1/N})^{2/3}(\frac{K}{2N\B})^{1/3}$ and $L_P=(\frac{2\Delta_P\T}{1-1/N})^{2/3}(\frac{K}{2N})^{1/3}$, and the dynamic regret is of order $\lowo{L_c\cdot\sqrt{\B SAK/L_c} + L_P\cdot\B\sqrt{SAK/L_P}}=\lowo{(\B SA\T(\Delta_c+\B^2\Delta_P))^{1/3}K^{2/3}}$.
	%$\lowo{L_c\cdot\sqrt{\B SAK/L_c} + L_P\cdot\B\sqrt{SAK/L_P}}=\lowo{(\B SA\Delta_c\T)^{1/3}K^{2/3} + \B(SA\Delta_P\T)^{1/3}K^{2/3}}$.
\end{proof}

\section{Omitted Details in \pref{sec:fha}}
\label{app:fha}
% !TEX root = main.tex

\paragraph{Notations} Under the protocol of \pref{alg:fha}, for any $k\in [K]$, denote by $M_k$ the number of intervals in the first $k$ episodes.
Clearly, $M=M_K$.

The following lemma is a more general version of \pref{lem:fha}.

\begin{lemma}
	\label{lem:bound reg}
	For any $K'\in[K]$, $R_{K'}\leq\rR_{M_{K'}}+\B$.
\end{lemma}
\begin{proof}
	Let $\calI_k$ be the set of intervals in episode $k$.
	Then the regret in episode $k$ satisfies
	\begin{align*}
		\sum_{m\in\calI_k}\sumhm c^m_h - \optV_k(s^k_1) &= \sum_{m\in\calI_k}\rbr{\sumhm c^m_h -  V^{\optpi,m}_1(s^m_1)} + \sum_{m\in\calI_k}V^{\optpi,m}_1(s^m_1)  - \optV_k(s^k_1)\\
		&\leq \sum_{m\in\calI_k}(C^m - V^{\optpi,m}_1(s^m_1)) + \frac{\B}{2K},
	\end{align*}
	where the last step is by the definition of $c^m_{H_m+1}$ and $V^{\optpi,m}_1(s^m_1)\leq \optV_k(s^m_1) + \frac{\B}{2K}\leq\frac{3}{2}\B$ by \pref{lem:hitting}.
	Summing up over $k$ completes the proof.
\end{proof}

\begin{lemma}
	\label{lem:bound M}
	Suppose algorithm $\frA$ ensures $\rR_{M'} = \tilo{\gamma_0 + \gamma_1{M'}^{1/3} + \gamma_{\frac{1}{2}}{M'}^{1/2} + \gamma_2{M'}^{2/3} }$ for any number of intervals $M'\leq M$ with cetain probability.
	Then with the same probability, $M_{K'} =\tilo{ K' + \gamma_0/\B + (\gamma_1/\B)^{3/2} + (\gamma_{\frac{1}{2}}/\B)^2 + (\gamma_2/\B)^3}$ and $\rR_{M_{K'}} = \tilo{ \gamma_1{K'}^{1/3} + \gamma_{\frac{1}{2}}{K'}^{1/2} + \gamma_2{K'}^{2/3} + \gamma_1^{3/2}/\B^{1/2} + \gamma_{\frac{1}{2}}^2/\B +  \gamma_2^3/\B^2 + \gamma_0}$ for any $K'\in[K]$.
\end{lemma}
\begin{proof}
	Fix a $K'\in[K]$.
	For any $M'\leq M_{K'}$, let $\calC_g=\{m\in[M']: s^m_{H_m+1}=g\}$.
	Then,
	\begin{align}
		\rR_{M'} &= \sum_{m\in\calC_g}(C^m - V^{\optpi,m}_1(s^m_1)) + \sum_{m\notin\calC_g}(C^m - V^{\optpi,m}_1(s^m_1))\notag\\ 
		&= \tilO{\gamma_0 + \gamma_1{M'}^{1/3} + \gamma_{\frac{1}{2}}{M'}^{1/2} + \gamma_2{M'}^{2/3}}.\label{eq:rR}
	\end{align}
	Note that $V^{\optpi,m}_1(s^m_1)\leq \optV_{k(m)}(s^m_1) + \frac{\B}{2K}\leq\frac{3}{2}\B$ by \pref{lem:hitting}.
	Moreover, $C^m\geq 2\B$ when $m\notin \calC_g$.
	Therefore, $C^m - V^{\optpi,m}_1(s^m_1)\geq -\frac{3\B}{2}$ for $m\in\calC_g$ and $C^m - V^{\optpi,m}_1(s^m_1)\geq \frac{\B}{2}$ for $m\notin\calC_g$.
	Reorganizing terms and by $|\calC_g|\leq K'$, we get:
	\begin{align*}
		\frac{\B M'}{2} \leq 2\B K' + \tilO{\gamma_0 + \gamma_1{M'}^{1/3} + \gamma_{\frac{1}{2}}{M'}^{1/2} + \gamma_2{M'}^{2/3}}.
	\end{align*}
	Solving a quadratic inequality w.r.t.~$M'$, we get $M' =\tilo{ K' + \gamma_0/\B + (\gamma_1/\B)^{3/2} + (\gamma_{\frac{1}{2}}/\B)^2 + (\gamma_2/\B)^3 }$.
	Define $\gamma=\gamma_0/\B + (\gamma_1/\B)^{3/2} + (\gamma_{\frac{1}{2}}/\B)^2 + (\gamma_2/\B)^3$.
	Plugging the bound on $M'$ back to \pref{eq:rR}, we have
	\begin{align*}
		\rR_{M'} &= \tilO{ \gamma_0 + \gamma_1{K'}^{1/3} + \gamma_{\frac{1}{2}}{K'}^{1/2} + \gamma_2{K'}^{2/3} + \gamma_1\gamma^{1/3} + \gamma_{\frac{1}{2}}{\gamma}^{1/2} + \gamma_2{\gamma}^{2/3} }\\
		&= \tilO{ \gamma_0 + \gamma_1{K'}^{1/3} + \gamma_{\frac{1}{2}}{K'}^{1/2} + \gamma_2{K'}^{2/3} + \gamma_1^{3/2}/\B^{1/2} + \gamma_{\frac{1}{2}}^2/\B + \gamma_2^3/\B^2 + \B\gamma }\\
		&= \tilO{ \gamma_0 + \gamma_1{K'}^{1/3} + \gamma_{\frac{1}{2}}{K'}^{1/2} + \gamma_2{K'}^{2/3} + \gamma_1^{3/2}/\B^{1/2} + \gamma_{\frac{1}{2}}^2/\B + \gamma_2^3/\B^2 },
	\end{align*}
%	Plugging this back, we have:
%	\begin{align*}
%		\gamma_1{M'}^{1/3} &= \tilO{\gamma_1\rbr{ K' + \gamma_0/\B + (\gamma_1/\B)^{3/2} + (\gamma_{\frac{1}{2}}/\B)^2 + (\gamma_2/\B)^3 }^{1/3} }\\
%		&= \tilO{ \gamma_1\max\{K', \gamma_0/\B\}^{1/3} + \gamma_1^{3/2}/\B^{1/2} + \gamma_1(\gamma_{\frac{1}{2}}/\B)^{2/3} + \gamma_1\gamma_2/\B},\\
%		\gamma_{\frac{1}{2}}{M'}^{1/2} &= \tilO{ \gamma_{\frac{1}{2}}\rbr{ K' + \gamma_0/\B + (\gamma_1/\B)^{3/2} + (\gamma_{\frac{1}{2}}/\B)^2 + (\gamma_2/\B)^3 }^{1/2} }\\
%		&= \tilO{ \gamma_{\frac{1}{2}}\max\{K', \gamma_0/\B\}^{1/2} + },\\
%		\gamma_2{M'}^{2/3} &= \tilO{\gamma_2\rbr{ K' + \gamma_0/\B + (\gamma_1/\B)^{3/2} + (\gamma_{\frac{1}{2}}/\B)^2 + (\gamma_2/\B)^3 }^{2/3} }\\
%		&= \tilO{\gamma_2\max\{K', \gamma_0/\B\}^{2/3} + \gamma_1\gamma_2/\B + \gamma_2(\gamma_{\frac{1}{2}}/\B)^{4/3} + \gamma_2^3/\B^2 }.
%	\end{align*}
	where in the second last step we apply Young's inequality for product ($xy\leq x^p/p + y^q/q$ for $x\geq0$, $y\geq 0$, $p>1$, $q>1$, and $\frac{1}{p}+\frac{1}{q}=1$).
	%we have $\gamma_1\gamma_2/\B=(\gamma_1/\B^{1/3})(\gamma_2/\B^{2/3})\leq \gamma_1^{3/2}/\B^{1/2} +  \gamma_2^3/\B^2$,
	%$\gamma_1(\gamma_0/\B)^{1/3}\leq \gamma_1^{3/2}/\B^{1/2} + \gamma_0$,
	%and $\gamma_2(\gamma_0/\B)^{2/3} \leq \gamma_2^3/\B^2 + \gamma_0$.
	Putting everything together and setting $M'=M_{K'}$ completes the proof.
\end{proof}

\section{Omitted Details in \pref{sec:subopt}}
\label{app:subopt}
% !TEX root = main.tex

\paragraph{Extra Notations} Let $Q^m_h$, $V^m_h$, $x_m$ be the value of $Q_h$, $V_h$, and $x$ at the beginning of interval $m$, and $Q^m_{H+1}(s, a)=V^m_{H+1}(s)$ for any $(s, a)\in\SA$.

\subsection{\pfref{thm:mvp}}
We first prove two lemmas related to the optimism of $Q^m_h$.
Define the following reference value function: $\rQ^m_h(s, a) = (\hatc^m(s, a) + \P^m_{s, a}\rV^m_{h+1} - b^m(s, a, \rV^m_{h+1}) - \rx_m)_+$ for $h\in[H]$, where $\rV^m_h(s)=\argmin_a\rQ^m_h(s, a)$ for $h\in[H]$, $\rV^m_{H+1}=c_f$, $\rQ^m_{H+1}(s, a)=\rV^m_{H+1}(s)$ for any $(s, a)\in\SA$, and $\rx_m=\Delta_{c,m}+4\B\Delta_{P,m}$.
%We first show two lemmas relating to the optimism of estimated value functions.

\begin{lemma}
	\label{lem:tilopt}
	With probability at least $1-2\delta$, $\rQ^m_h(s, a) \leq Q^{\star,m}_h(s, a)$ for $m \leq M$.
\end{lemma}
\begin{proof}
	We prove this by induction on $h$.
	The base case of $h=H+1$ is clearly true.
	For $h\leq H$, by \pref{lem:mvp}, for any $(s, a)\in\SA$:
	\begin{align*}
		\rQ^m_h(s, a) &= \hatc^m(s, a) + \P^m_{s, a}\rV^m_{h+1} - b^m(s, a, \rV^m_{h+1}) - \rx_m\\
		&\leq \hatc^m(s, a) + \P^m_{s, a}V^{\star, m}_{h+1} - b^m(s, a, V^{\star, m}_{h+1}) - \rx_m \tag{by the induction step}\\
		&= \hatc^m(s, a) + \tilP^m_{s, a}V^{\star, m}_{h+1} + (\P^m_{s,a} - \tilP^m_{s,a})V^{\star,m}_{h+1} - b^m(s, a, V^{\star, m}_{h+1}) - \rx_m \\
		&\overset{\text{(i)}}{\leq} \hatc^m(s, a) + \tilP^m_{s, a}V^{\star, m}_{h+1} - \rx_m \overset{\text{(ii)}}{\leq} c^m(s, a) + P^m_{s, a}V^{\star,m}_{h+1} = Q^{\star,m}_h(s, a),
	\end{align*}
	where in (i) we apply \pref{lem:dPV} with $|\{V^{\star,m}_h\}_{m,h}|\leq HK+1$ to obtain $(\P^m_{s,a} - \tilP^m_{s,a})V^{\star,m}_{h+1} - b^m(s, a, V^{\star, m}_{h+1})\leq 0$; in (ii) we apply \pref{lem:dc}, \pref{lem:bound optQ}, and the definition of $\rx_m$.
\end{proof}

\begin{lemma}
	\label{lem:opt Q}
	With probability at least $1-2\delta$, $Q^m_h(s, a) \leq Q^{\star,m}_h(s, a) + (\Delta_{c,m}+4\B\Delta_{P,m})(H-h+1)$ and $x_m\leq\max\{\frac{1}{mH}, 2(\Delta_{c,m}+4\B\Delta_{P,m})\}$.
\end{lemma}
\begin{proof}
	The second statement simply follows from \pref{lem:tilopt}, $Q^{\star,m}_h(s, a)\leq Q^{\optpi,m}_h(s, a)\leq 4\B=B/4$ by \pref{lem:bound optQ}, and the computing procedure of $x_m$.
	We now prove $Q^m_h(s, a) \leq \rQ^m_h(s, a) + (\Delta_{c,m}+4\B\Delta_{P,m})(H-h+1)$ by induction on $h$, and the first statement simply follows from $\rQ^m_h(s, a) \leq Q^{\star,m}_h(s, a)$ (\pref{lem:tilopt}).
	The statement is clearly true for $h=H+1$.
	For $h\leq H$, by the induction step and $\norm{V^m_{h+1}}_{\infty}\leq B/4$ from the update rule, we have $V^m_{h+1}(s)\leq\min\{B/4, \rV^m_{h+1}(s)+(\Delta_{c,m}+4\B\Delta_{P,m})(H-h)\}\leq\rV^m_{h+1}(s)+y^m_{h+1}\leq B$ for any $s\in\calS_+$, where $y^m_h=\min\{B/4,(\Delta_{c,m}+4\B\Delta_{P,m})(H-h+1)\}$.
	Thus,
	\begin{align*}
		&\P^m_{s, a}V^m_{h+1} - b^m(s, a, V^m_{h+1}) - x_m \leq \P^m_{s, a}(\rV^m_{h+1}+y^m_{h+1}) - b^m(s, a, \rV^m_{h+1}+y^m_{h+1}) \tag{\pref{lem:mvp} and $x_m\geq 0$}\\
		&\leq \P^m_{s, a}\rV^m_{h+1} - b^m(s, a, \rV^m_{h+1}) - \rx_m + (\Delta_{c,m}+4\B^m\Delta_{P,m})(H-h+1),
	\end{align*}
	where in the last inequality we apply definition of $\rx_m$ and $b^m(s, a, \rV^m_{h+1}+y^m_{h+1})=b^m(s, a, \rV^m_{h+1})$ since constant offset does not change the variance.
	Then, $Q^m_h(s, a) \leq \rQ^m_h(s, a) + (\Delta_{c,m}+4\B\Delta_{P,m})(H-h+1)$ by the update rule of $Q^m_h$ and the definition of $\rQ^m_h$.
\end{proof}

We are now ready to prove the main theorem, from which \pref{thm:mvp} is a simple corollary.

\begin{theorem}
	\label{thm:MVP-SSP}
	\pref{alg:MVP-SSP} ensures with probability at least $1-22\delta$, for any $M'\leq M$,  
	$\rR_{M'} = \tilo{\sqrt{\B SAL_{c,M'}M'} + \B\sqrt{SAL_{P,M'}M'} + \B SAL_{c,M'} + \B S^2AL_{P,M'} + \summp (\Delta_{c,m}+\B\Delta_{P,m})H}$.
\end{theorem}
\begin{proof}
	Note that with probability at least $1-2\delta$:
	\begin{align*}
		\rR_{M'} &\leq \sum_{m=1}^{M'}\rbr{\sumhm c^m_h + c^m_{H_m+1} - V^{\star, m}_1(s^m_1)} \tag{$V^{\star,m}_1(s^m_1)\leq V^{\optpi,m}_1(s^m_1)$}\\ 
		&\leq \sum_{m=1}^{M'}\rbr{\sumhm c^m_h + c^m_{H_m+1} - V^m_1(s^m_1)} + \summp(\Delta_{c,m}+4\B\Delta_{P,m})H \tag{\pref{lem:opt Q}}\\
		&\leq \sum_{m=1}^{M'}\sumhm\rbr{ c^m_h + V^m_{h+1}(s^m_{h+1}) - V^m_h(s^m_h) } + \summp(\Delta_{c,m}+4\B\Delta_{P,m})H + \tilO{\B SAL_{M'}} \tag{$c^m_{H_m+1}=\tilo{\B}$ and \pref{lem:new interval}}\\
		&\leq \sum_{m=1}^{M'}\sumhm\rbr{(c^m_h - \hatc^m_h) + (V^m_{h+1}(s^m_{h+1}) - P^m_hV^m_{h+1}) + (P^m_h-\P^m_h)V^m_{h+1} + b^m_h }\\
		&\qquad + 2\summp(\Delta_{c,m}+4\B\Delta_{P,m})H + \tilO{\B SAL_{M'}},
	\end{align*}
	where the last step is by the definitions of $V^m_h(s^m_h)$, $x_m\leq \max\{\frac{1}{mH}, 2(\Delta_{c,m}+4\B\Delta_{P,m})\}$ (\pref{lem:opt Q}), $\max\{a, b\}\leq\frac{a+b}{2}$, and $\summp\sumhm\frac{1}{mH}=\tilo{1}$.
	Now we bound the first three sums separately.
	For the first term, with probability at least $1-4\delta$,
	\begin{align*}
		&\summp\sumhm (c^m_h - \hatc^m_h) = \summp\sumhm(c^m_h - c^m(s^m_h, a^m_h)) + \summp\sumhm(c^m(s^m_h, a^m_h) - \hatc^m_h)\\
		&\leq \tilO{\sqrt{C_{M'}} + \sqrt{SAL_{c,M'}C_{M'}} + SAL_{c,M'}} + 2\summp \Delta_{c,m}H. \tag{\pref{lem:freedman} and \pref{lem:c diff}}
	\end{align*}
	For the second term, by \pref{lem:freedman}, with probability at least $1-\delta$,
	\begin{align*}
		&\sum_{m=1}^{M'}\sumhm (V^m_{h+1}(s^m_{h+1}) - P^m_hV^m_{h+1}) = \tilO{ \sqrt{\summp\sumhm\fV(P^m_h, V^m_{h+1})} + \B }\\
		&=\tilO{ \sqrt{\summp\sumhm\fV(P^m_h, V^{\star,m}_{h+1})} + \sqrt{\summp\sumhm\fV(P^m_h, V^{\star,m}_{h+1}-V^m_{h+1})} + \B}, \tag{$\var[X+Y]\leq 2(\var[X] + \var[Y])$ and $\sqrt{a+b}\leq\sqrt{a}+\sqrt{b}$}
	\end{align*}
	which is dominated by the upper bound of the third term below.
	For the third term, by $P^m_hV^m_{h+1}\leq \tilP^m_hV^m_{h+1} + 4\B(\Delta_{P,m}+\n^m_h)$, with probability at least $1-2\delta$,
	\begin{align*}
		&\summp\sumhm(P^m_h-\P^m_h)V^m_{h+1} \leq \summp\sumhm(\tilP^m_h - \P^m_h)V^m_{h+1} + \summp\sumhm 4\B(\Delta_{P,m} + \n^m_h)\\
		&\leq \summp\sumhm\rbr{(\tilP^m_h - \P^m_h)V^{\star,m}_{h+1} + (\tilP^m_h - \P^m_h)(V^m_{h+1} - V^{\star,m}_{h+1}) + 4\B\n^m_h } + \summp 4\B\Delta_{P,m}H\\
		&= \tilO{\summp\sumhm\rbr{\sqrt{ \frac{ \fV(P^m_h, V^{\star,m}_{h+1}) }{\N^m_h} } + \frac{S\B}{\N^m_h} + \sqrt{\frac{S\fV(P^m_h, V^m_{h+1} - V^{\star,m}_{h+1})}{\N^m_h}} } + \summp \B\Delta_{P,m}H }  \tag{$\n^m_h\leq\frac{1}{\N^m_h}$, \pref{lem:dPV} with $|\{V^{\star,m}_{h+1}\}_{m,h}| \leq HK+1$, and \pref{lem:dPv}}\\
		&= \tilO{\sqrt{SAL_{P,M'}\summp\sumhm\fV(P^m_h, V^{\star,m}_{h+1})} + \sqrt{S^2AL_{P,M'}\summp\sumhm\fV(P^m_h, V^{\star,m}_{h+1} - V^m_{h+1})} }\\
		&\qquad + \tilO{\B S^2AL_{P,M'} + \summp \B\Delta_{P,m}H} \tag{Cauchy-Schwarz inequality and \pref{lem:sum}}.
	\end{align*}
	Moreover, by \pref{lem:sum bV}, with probability at least $1-\delta$,
	\begin{align*}
		\summp\sumhm b^m_h &= \tilO{ \sqrt{SAL_{P,M'}\summp\sumhm\fV(P^m_h, V^m_{h+1})} + \B S^{1.5}A\Lp + \B\sqrt{SAH\Lp\summp\Delta_{P,m} } }\\
		&= \tilO{ \sqrt{SA\Lp\summp\sumhm\fV(P^m_h, V^{\star,m}_{h+1})} + \sqrt{SA\Lp\summp\sumhm\fV(P^m_h, V^m_{h+1}-V^{\star,m}_{h+1})} }\\
		&\qquad + \tilO{\B S^{1.5}A\Lp + \summp\B\Delta_{P,m}H}. \tag{$\var[X+Y]\leq2\var[X]+2\var[Y]$, $\sqrt{a+b}\leq\sqrt{a}+\sqrt{b}$, and AM-GM inequality}
	\end{align*}
	 which is dominated by the upper bound of the third term above.
	 Putting everything together, we have with probability at least $1-11\delta$,
	\begin{align*}
		\rR_{M'} &= \tilO{ \sqrt{SA\Lc C_{M'}} + \B SA\Lc + \sqrt{SA\Lp\summp\sumhm\fV(P^m_h, V^{\star,m}_{h+1})} + \B S^2A\Lp}\\ 
		&\qquad + \tilO{\sqrt{S^2A\Lp\summp\sumhm\fV(P^m_h, V^{\star,m}_{h+1} - V^m_{h+1})} + \summp (\Delta_{c,m}+\B\Delta_{P,m})H }\\
		&= \tilO{ \sqrt{SA\Lc C_{M'}} + \sqrt{\B SA\Lp C_{M'}} + \B SA\Lc + \B S^2A\Lp }\\
		&\qquad + \tilO{ \summp (\Delta_{c,m}+\B\Delta_{P,m})H }. \tag{\pref{lem:optvar sum}, \pref{lem:opt-til var sum} and AM-GM inequality}
	\end{align*}
	Note that $\rR_{M'} = \summp(C^m - V^{\optpi,m}_1(s^m_1)) \geq C_{M'} - 4\B M'$ (\pref{lem:bound optQ}).
	Reorganizing terms and solving a quadratic inequality (\pref{lem:quad}) w.r.t $C_{M'}$ gives $C_{M'}=\tilo{\B M'}$ ignoring lower order terms.
	Plugging this back completes the proof.
\end{proof}

\begin{proof}[\pfref{thm:mvp}]
	Note that by by \pref{line:reset M} and \pref{line:reset N} of \pref{alg:MVP-SSP}, we have $L_c\leq\ceil{\frac{M'}{W_c}}$, $L_P\leq\ceil{\frac{M'}{W_P}}$, and the number of intervals between consecutive resets of $\M$ (or $\N$) are upper bounded by $W_c$ (or $W_P$), which gives
	\begin{align*}
		\summp(\Delta_{c,m}+\B\Delta_{P,m})H \leq \summp(\Delta_{c,f^c(m)}+\B\Delta_{P,f^P(m)})H \leq (W_c\Delta_c+\B W_P\Delta_P)H
	\end{align*}
	Applying \pref{thm:MVP-SSP} completes the proof.
\end{proof}

\subsection{\pfref{thm:any episode}}

We first show that \pref{alg:MVP-SSP-Restart} ensures an anytime regret bound in $\rcalM$.

\begin{theorem}
	\label{thm:double}
	With probability at least $1-22\delta$, \pref{alg:MVP-SSP-Restart} ensures for any $M'\leq M$, $\rR_{M'}= \tilo{(\B SA\Tmax\Delta_c)^{1/3}{M'}^{2/3} + \B(SA\Tmax\Delta_P)^{1/3}{M'}^{2/3} + (\B SA\Tmax\Delta_c)^{2/3}{M'}^{1/3} + \B(S^{2.5}A\Tmax\Delta_P)^{2/3}{M'}^{1/3} + (\Delta_c+\B\Delta_P)\Tmax }$ .
\end{theorem}
\begin{proof}
	It suffices to prove the desired inequality for $M'\in\{2^n-1\}_{n\in\fN_+}$.
	Suppose $M'=2^N-1$ for some $N\geq 1$.
	By the doubling scheme, we run \pref{alg:MVP-SSP} on intervals $[2^{n-1}, 2^n-1]$ for $n=1,\ldots,N$, and the regret on intervals $[2^{n-1}, 2^n-1]$ is of order $\tilo{(\B SA\Tmax\Delta_c)^{1/3}(2^{n-1})^{2/3} + \B(SA\Tmax\Delta_P)^{1/3}(2^{n-1})^{2/3} + (\B SA\Tmax\Delta_c)^{2/3}(2^{n-1})^{1/3} + \B(S^{2.5}A\Tmax\Delta_c)^{2/3}(2^{n-1})^{1/3} + (\Delta_c+\B\Delta_P)\Tmax }$ by \pref{thm:mvp} and the choice of $W_c$ and $W_P$.
	Summing over $n$ completes the proof.
\end{proof}

\begin{proof}[\pfref{thm:any episode}]
	By \pref{lem:bound M} and \pref{thm:double} with $\gamma_0=(\Delta_c+\B\Delta_P)\Tmax$, $\gamma_1=(\B SA\Tmax\Delta_c)^{2/3} + \B(S^{2.5}A\Tmax\Delta_P)^{2/3}$, $\gamma_{\frac{1}{2}}=0$, and $\gamma_2= (\B SA\Tmax\Delta_c)^{1/3} + \B(SA\Tmax\Delta_P)^{1/3}$,
	we have $\gamma_1^{3/2}/\B^{1/2}=\tilo{\B^{1/2}SA\Tmax\Delta_c + \B S^{2.5}A\Tmax\Delta_P}$,
	$\gamma_2^3/\B^2 = \tilo{SA\Tmax\Delta_c/\B^2 + \B SA\Tmax\Delta_P}$,
	and thus $\rR_{M_{K'}} =\tilo{ (\B SA\Tmax\Delta_c)^{1/3}{K'}^{2/3} + \B(SA\Tmax\Delta_P)^{1/3}{K'}^{2/3} + (\B SA\Tmax\Delta_c)^{2/3}{K'}^{1/3} + \B(S^{2.5}A\Tmax\Delta_P)^{2/3}{K'}^{1/3} + \B^{1/2}SA\Tmax\Delta_c + \B S^{2.5}A\Tmax\Delta_P }$ for any $K'\in[K]$.
	Then by \pref{lem:bound reg}, we obtain the same bound as $\rR_{M_{K'}}$ for $R_{K'}$.
\end{proof}

\subsection{Auxiliary Lemmas}

\begin{lemma}
	\label{lem:optvar sum}
	With probability at least $1-2\delta$, 
	$\summp\sumhm\fV(P^m_h, V^{\star,m}_{h+1}) =\tilO{ \B C_{M'} + \B^2}$ for any $M'\leq M$.
\end{lemma}
\begin{proof}
	Applying \pref{lem:sum var} with $\norm{V^{\star,m}_h}_{\infty}\leq 4\B$ (\pref{lem:bound optQ}), with probability at least $1-2\delta$,
	\begin{align*}
		&\summp\sumhm\fV(P^m_h, V^{\star,m}_{h+1})\\ 
		&= \tilO{\summp V^{\star,m}_{H_m+1}(s^m_{H_m+1})^2 + \summp\sumhm \B(V^{\star, m}_h(s^m_h) - P^m_hV^{\star, m}_{h+1})_+ + \B^2}\\
		&= \tilO{\B C_{M'} + \B^2 },
	\end{align*}
	where in the last step we apply
	\begin{align*}
		(V^{\star,m}_h(s^m_h) - P^m_hV^{\star,m}_{h+1})_+ \leq (Q^{\star,m}_h(s^m_h, a^m_h) - P^m_hV^{\star,m}_{h+1})_+ \leq c^m(s^m_h, a^m_h),
	\end{align*}
	and $\summp\sumhm c^m(s^m_h, a^m_h)=\tilo{\summp\sumhm c^m_h}$ by \pref{lem:e2r}.
\end{proof}

\begin{lemma}
	\label{lem:opt-til var sum}
	With probability at least $1-9\delta$, for any $M'\leq M$, $\summp\sumhm\fV(P^m_h, V^{\star,m}_{h+1} - V^m_{h+1}) = \tilo{\B\sqrt{\B SA\Lp C_{M'}} + \B\sqrt{SA\Lc C_{M'}} + \B^2 S^2A\Lp + \B^2SA\Lc + \summp \B(\Delta_{c,m}+\B\Delta_{P,m})H}$.
\end{lemma}
\begin{proof}
	Let $z^m_h=\min\{B/4, (\Delta_{c,m}+4\B\Delta_{P,m})H\}\Ind\{h\leq H\}$.
	By \pref{lem:opt Q} and $\norm{V^m_h}_{\infty}\leq B/4$, we have $V^{\star,m}_h(s) + z^m_h \geq V^m_h(s)$ for all $s\in\calS_+$.
	%Then we apply \pref{lem:sum var} and
	Moreover, by \pref{lem:new interval},
	\begin{align*}
		&\summp (V^{\star,m}_{H_m+1}(s^m_{H_m+1}) + z^m_{H_m+1} - V^m_{H_m+1}(s^m_{H_m+1}))^2\\
		&\leq \summp (z^m_{H_m+1})^2\Ind\{s^m_{H_m+1}=g\} + \tilO{\B^2\summp\Ind\{H_m<H, s^m_{H_m+1}\neq g\} }\\
		&= 4\B\summp(\Delta_{c,m}+4\B\Delta_{P,m})H + \tilO{\B^2SA\L}.
	\end{align*}
	Also note that
	\begin{align*}
		(*)&=\summp \B\sumhm(V^{\star, m}_h(s^m_h) - V^m_h(s^m_h) - P^m_hV^{\star, m}_{h+1} + P^m_hV^m_{h+1} + z^m_h - z^m_{h+1})_+\\
		&\leq \summp \B\sumhm\rbr{c^m(s^m_h, a^m_h) + \tilP^m_hV^m_{h+1} - V^m_h(s^m_h) + 4\B\n^m_h}_+ + 2\summp \B(\Delta_{c,m} + 4\B\Delta_{P,m})H \tag{$V^{\star, m}_h(s^m_h)\leq Q^{\star, m}_h(s^m_h, a^m_h)$, $z^m_h\geq z^m_{h+1}$, and $P^m_hV^m_{h+1}\leq \tilP^m_hV^m_{h+1} + 4\B(\n^m_h+\Delta_{P,m})$}\\
		&\leq \summp \B\sumhm( c^m(s^m_h, a^m_h) - \hatc^m_h + (\tilP^m_h - \P^m_h)V^{\star,m}_{h+1} + (\tilP^m_h - \P^m_h)(V^m_{h+1} - V^{\star,m}_{h+1}) + b^m_h )_+\\ 
		&\qquad + 3\summp \B(\Delta_{c,m} + 4\B\Delta_{P,m})H + 4\B^2\summp\sumhm\n^m_h + \tilO{\B}, 
		%\tag{definition of $V^m_h(s^m_h)$ and \pref{lem:opt Q}}.
	\end{align*}
	where the last step is by the definitions of $V^m_h(s^m_h)$, $x_m\leq \max\{\frac{1}{mH}, 2(\Delta_{c,m}+4\B\Delta_{P,m})\}$ (\pref{lem:opt Q}), $\max\{a, b\}\leq\frac{a+b}{2}$, and $\summp\sumhm\frac{1}{mH}=\tilo{1}$.
	Now by \pref{lem:c diff}, \pref{lem:dPV}, \pref{lem:dPv}, and $\n^m_h\leq\frac{1}{\N^m_h}$, we continue with
	\begin{align*}
		(*)&= \tilO{\B\rbr{\sqrt{SA\Lc C_{M'}} + SA\Lc + \summp\sumhm\rbr{\sqrt{\frac{\fV(P^m_h, V^{\star,m}_{h+1})}{\N^m_h}} + \sqrt{\frac{S\fV(P^m_h, V^m_{h+1} - V^{\star,m}_{h+1})}{\N^m_h}} }  }  }\\ 
		&\qquad + \tilO{ \summp\sumhm\frac{\B^2S}{\N^m_h} + \summp \B(\Delta_{c,m}+\B\Delta_{P,m})H + \B\summp\sumhm b^m_h}\\
		&= \tilO{\B\sqrt{SA\Lc C_{M'}} + \B SA\Lc + \B\sqrt{SA\Lp\summp\sumhm\fV(P^m_h, V^{\star,m}_{h+1})} + \B^2 S^2A\Lp }\\
		&\qquad + \tilO{\B\sqrt{S^2A\Lp\summp\sumhm\fV(P^m_h, V^{\star,m}_{h+1} - V^m_{h+1})} + \summp \B(\Delta_{c,m}+\B\Delta_{P,m})H},
	\end{align*}
	where in the last step we apply Cauchy-Schwarz inequality, \pref{lem:sum}, \pref{lem:sum bV}, $\var[X+Y]\leq2\var[X]+2\var[Y]$, and AM-GM inequality.
	Finally, by \pref{lem:optvar sum}, we continue with
	\begin{align*}
		(*) &= \tilO{\B\sqrt{SA\Lc C_{M'}} + \B SA\Lc + \B\sqrt{\B SA\Lp C_{M'}} + \B^2 S^2A\Lp }\\
		&\qquad + \tilO{\B\sqrt{S^2A\Lp\summp\sumhm\fV(P^m_h, V^{\star,m}_{h+1} - V^m_{h+1})} + \summp \B(\Delta_{c,m}+\B\Delta_{P,m})H}.
	\end{align*}
	Applying \pref{lem:sum var} on value functions $\{ V^{\star,m}_h + z^m_h - V^m_h \}_{m,h}$ (constant offset does not change the variance) and plugging in the bounds above, we have
	\begin{align*}
		&\summp\sumhm\fV(P^m_h, V^{\star,m}_{h+1} - V^m_{h+1}) = \summp\sumhm\fV(P^m_h, V^{\star,m}_{h+1} + z^m_h - V^m_{h+1})\\
		&=\tilO{\B\sqrt{SA\Lc C_{M'}} + \B^2SA\Lc + \B\sqrt{\B SA\Lp C_{M'}} + \B^2 S^2A\Lp }\\
		&\qquad + \tilO{\B\sqrt{S^2A\Lp\summp\sumhm\fV(P^m_h, V^{\star,m}_{h+1} - V^m_{h+1})} + \summp \B(\Delta_{c,m}+\B\Delta_{P,m})H}.
	\end{align*}
	Then solving a quadratic inequality w.r.t $\summp\sumhm\fV(P^m_h, V^{\star,m}_{h+1} - V^m_{h+1})$ (\pref{lem:quad}) completes the proof.
\end{proof}

\subsection{Minimax Optimal Bound in Finite-Horizon MDP}
\label{app:finite horizon}

Here we give a high level arguments on why \pref{alg:MVP-SSP} implies a minimax optimal dynamic regret bound in the finite-horizon setting.
To adapt \pref{alg:MVP-SSP} to the non-homogeneous finite-horizon setting, we maintain empirical cost and transition functions for each layer $h\in[H]$ and let $c_f(s)=0$.
Following similar arguments and substituting $\B$, $\Tmax$ by horizon $H$, \pref{thm:mvp} implies (ignoring lower order terms)
\begin{align*}
	\rR_{M'} &= \tilO{\sqrt{SAH^2/W_c}M' + \sqrt{SAH^3/W_P}M' + (\Delta_cW_c + H\Delta_PW_P)H}\\
	&=\tilO{ H(SA\Delta_c)^{1/3}{M'}^{2/3} + (SAH^5\Delta_P)^{1/3}{M'}^{2/3} },
\end{align*}
where the extra $\sqrt{H}$ dependency in the first two terms comes from estimating the cost and transition functions of each layer independently, and we set $W_c=(SA)^{1/3}(M'/\Delta_c)^{2/3}$, $W_P=(SA/H)^{1/3}(M'/\Delta_P)^{2/3}$.
Note that the lower bound construction in \citep{mao2020model} only make use of non-stationary transition.
The lower bound they prove is $\lowo{(SA\Delta)^{1/3}(HT)^{2/3}}$ (their Theorem 5), which actually matches our upper bound $\tilo{(SAH^5\Delta_P)^{1/3}{M'}^{2/3}}$ for non-stationary transition since $T=M'H$ and $\Delta=H\Delta_P$ by their definition of non-stationarity.
It is also straightforward to show that the lower bound for non-stationary cost matches our upper bound following similar arguments in proving \pref{thm:lb}.

\section{Omitted Details in \pref{sec:opt}}
\label{app:opt}
% !TEX root = main.tex

\paragraph{Notations} 
Denote by $\rho^c_m$ and $\rho^P_m$ the values of $\rho^c$ and $\rho^P$ at the beginning of interval $m$ respectively, that is, $\rho^c_m=g^c(\nu^c_m)$ and $\rho^P_m=g^P(\nu^P_m)$, where $g^c(m)=\min\{\frac{c_1}{\sqrt{m}}, \frac{1}{2^8H}\}$ and $g^P(m)=\min\{\frac{c_2}{\sqrt{m}}, \frac{1}{2^8H}\}$.
Denote by $\cc^m$ the value of $\cc$ at the beginning of interval $m$ and define $\cc^m_h=\cc(s^m_h, a^m_h)$.
Define $\cQ^{\optpi,m}_h$ and $\cV^{\optpi,m}_h$ as the action-value function and value function w.r.t cost $c^m + 8\eta_m$, transition $P^m$, and policy $\optpi_{k(m)}$; and $C_{[i,j]}=\sum_{m=i}^j\sumhm C^m$.
Let $\cQ^{\star,m}_h$ and $\cV^{\star,m}_h$ be the optimal value functions w.r.t cost function $c^m + 8\eta_m$ and transition function $P^m$.
It is not hard to see that they can be defined recursively as follows: $\cV^{\star,m}_{H+1}=c_f$ and for $h\leq H$,
\begin{equation*}
	\cQ^{\star,m}_h(s, a) = c^m(s, a) + 8\eta_m + P^m_{s, a}\cV^{\star,m}_{h+1}, \qquad \cV^{\star,m}_h(s) = \min_aQ^{\star,m}_h(s, a).
\end{equation*}
For notational convenience, define $\cQ^m_{H+1}(s, a)=\cV^m_{H+1}(s)$, $\cQ^{\optpi,m}_{H+1}(s, a)=\cV^{\optpi,m}_{H+1}(s)$, and $\cQ^{\star,m}_{H+1}(s, a)=\cV^{\star,m}_{H+1}(s)$ for any $(s, a)\in\SA$; let $L_c=L_{c,[1,K]}$ and $L_P=L_{P,[1,K]}$.

\paragraph{Proof Sketch of \pref{thm:mvp-test}}
We give a high level idea on the analysis of the main theorem and also point out the key technical challenges.
%Also define $L_c$ and $L_P$ as the number of resets of $\M$ and $\N$ within $K$ intervals respectively.
We decompose the regret as follows:
\begin{align*}
	\rR_K &= \summk(C^m - \cV^m_1(s^m_1)) + \summk(\cV^m_1(s^m_1) - \cV^{\optpi,m}_1(s^m_1)) + 8\T\summk\eta_m\\
	&\lesssim \summk\sumh\rbr{c^m_h - \hatc^m_h + \cV^m_{h+1}(s^m_{h+1}) - \P^m_h\cV^m_{h+1} + b^m_h - 8\eta_m} \tag{definition of $\cV^m_h(s^m_h)$}\\
	&\qquad + \summk(\cV^m_1(s^m_1) - \cV^{\optpi,m}_1(s^m_1)) + 8\T\summk\eta_m.
\end{align*}
We bound the three terms above separately.
For the second term, we first show that $\cV^m_1(s^m_1) - \cV^{\optpi,m}_1(s^m_1)\leq (\Delta_{c,m}+B\Delta_{P,m})\T$, where $\Delta_{c,m}=\Delta_{c,[i^c_m,m]}$, $\Delta_{P,m}=\Delta_{P,[i^P_m,m]}$ are the accumulated cost and transition non-stationarity since the last reset respectively.
Although proving such a bound is straightforward when $\cV^m_h$ is indeed a value function (similar to \pref{lem:value bound}), it is non-trivial under the UCBVI update rule as the bonus term $b$ depends on the next-step value function and can not be simply treated as part of the cost function.
A key step here is to make use of the monotonic property (\pref{lem:mvp}) of the bonus function; see \pref{lem:Q diff} for more details.
Now by the periodic resets of cost and transition counters (\pref{line:period c} and \pref{line:period p}), the number of intervals between consecutive resets of cost and transition estimation is upper bounded by $W_c$ and $W_P$ respectively.
Thus,
\begin{align*}
	&\summk(\Delta_{c,m}+B\Delta_{P,m})\T \leq \summk(\Delta_{c,f^c(m)}+B\Delta_{P,f^P(m)})\T \leq (W_c\Delta_c+BW_P\Delta_P)\T\\
	&=\tilO{(\B SA\T\Delta_c)^{1/3}K^{2/3} + \B(SA\T\Delta_P)^{1/3}K^{2/3} + (\Delta_c + \B\Delta_P)\T}.
\end{align*}
where the last step is simply by the chosen values of $W_c$ and $W_P$.

For the third term, we have:
\begin{align*}
	\T\summk\eta_m & \leq \T\summk\rbr{\frac{c_1}{\sqrt{\nu^c_m}} + \frac{Bc_2}{\sqrt{\nu^P_m}}} = \tilO{\T\rbr{c_1\sum_{i=1}^{L_c}\sqrt{M^c_i} + \B c_2\sum_{i=1}^{L_P}\sqrt{M^P_i} } }\\ 
	&= \tilO{ \T (c_1\sqrt{L_cK} + \B c_2\sqrt{L_PK}) } = \tilO{ \sqrt{\B SAL_cK} + \B\sqrt{SAL_PK} },
\end{align*}
where $M^c_i$ (or $M^P_i$) is the number of intervals between the $i$-th and $(i+1)$-th reset of cost (or transition) estimation, and the second last step is by Cauchy-Schwarz inequality.
Finally we bound the first term, simply by \textbf{Test 1} and \textbf{Test 2}, we have (only keeping the dominating terms)
\begin{align*}
	&\summk\sumh\rbr{c^m_h - \hatc^m_h + \cV^m_{h+1}(s^m_{h+1}) - \P^m_h\cV^m_{h+1} + b^m_h - 8\eta_m}\\
	&=\sum_{i=1}^{L_c}\sum_{m\in\calI^c_i}\sumhm(c^m_h - \hatc^m_h) + \sum_{i=1}^{L_P}\sum_{m\in\calI^P_i}\sumhm(\cV^m_{h+1}(s^m_{h+1}) - \P^m_h\cV^m_{h+1}) + \summp\sumhm(b^m_h - 8\eta_m)\\
	&\lesssim \summp\sumhm\rbr{\sqrt{\frac{\barc^m_h}{\M^m_h}} + \sqrt{\frac{\fV(\P^m_h, \cV^m_{h+1})}{\N^m_h}}} = \tilO{\sqrt{\B SAL_cK} + \B\sqrt{SAL_PK}}.
\end{align*}
where $\{\calI^c_i\}_{i=1}^{L_c}$ (or $\{\calI^P_i\}_{i=1}^{L_P}$) is a partition of $K$ episodes such that $\M$ (or $\N$) is reseted in the last interval of each $\calI^c_i$ (or $\calI^P_i$) for $i<L_c$ (or $i<L_P$) and the last interval of $\calI^c_{L_c}$ (or $\calI^P_{L_P}$) is $K$, and in the second last step we apply the definition of $\chi^c_m$ (\pref{lem:chic}) and $\chi^P_m$ (\pref{lem:chip}).
Note that the regret of non-stationarity along the learner's trajectory is cancelled out by the negative correction term $-8\eta_m$.
%A main difficulty here is that $\cV^m_h$ is affected by both cost and transition non-stationarity.
Now it suffices to bound $L_c$ and $L_P$.
It can be shown that the reset rules of the non-stationarity tests guarantee that
\begin{align*}
	L_c = \tilO{K/W_c + \B K/W_P}, \quad L_P = \tilO{K/W_P + K/(\B W_c)}.
\end{align*}
Details are deferred to \pref{lem:bound L}.
Putting everything together completes the proof.

Next, we present three lemmas related to the optimism and magnitude (\textbf{Test 3}) of estimated value function.

\begin{lemma}
	\label{lem:bound cQ star}
	With probability at least $1-2\delta$, for all $m\leq K$, $\cQ^m_h(s, a)\leq \cQ^{\star,m}_h(s, a) + (\Delta_{c,m} + B\Delta_{P,m})(H-h+1)$.
\end{lemma}
\begin{proof}
	%This is clearly true when $\Delta_{[1,m]}\leq\rho_m$.
	%When $\Delta_{[1,m]}>\rho_m$, 
	We prove this by induction on $h$.
	The base case of $h=H+1$ is clearly true.
	For $h\leq H$, by \textbf{Test 3} and the induction step, we have $\cV^m_{h+1}(s)\leq\min\{B/2, \cV^{\star, m}_{h+1}(s) + (\Delta_{c,m} + B\Delta_{P,m})(H-h)\}\leq \cV^{\star, m}_{h+1}(s) + x^m_{h+1}\leq B$ where $x^m_h=\min\{B/2, (\Delta_{c,m} + B\Delta_{P,m})(H-h+1)\}$ and $\cV^{\star,m}_h(s)\leq \cV^{\optpi,m}_h(s)\leq\frac{B}{4}+8H\eta_m\leq \frac{B}{3}$.
	Thus, with probability at least $1-2\delta$,
	\begin{align*}
		&\cc^m(s, a) + \P^m_{s, a}\cV^m_{h+1} - b^m(s, a, \cV^m_{h+1})\\
		&\leq \cc^m(s, a) + \P^m_{s, a}(\cV^{\star, m}_{h+1} + x^m_{h+1}) - b^m(s, a, \cV^{\star, m}_{h+1} + x^m_{h+1}) \tag{\pref{lem:mvp}}\\
		&\overset{\text{(i)}}{\leq} \cc^m(s, a) + \tilP^m_{s, a}(\cV^{\star, m}_{h+1} + x^m_{h+1}) \tag{\pref{lem:dPV}}\\
		&\leq c^m(s, a) + 8\eta_m + \Delta_{c,m} + P^m_{s, a}(\cV^{\star, m}_{h+1} + x^m_{h+1}) + B\Delta_{P,m} \tag{\pref{lem:dc}}\\
		&\leq \cQ^{\star,m}_h(s, a) + (\Delta_{c,m}+B\Delta_{P,m})(H-h+1).
	\end{align*}
	Note that in (i) we use the fact that $|\{ \cV^{\star, m}_h + x^m_h \}_{m,h}| \leq (HK+1)^6$ since $|\{(c^m, P^m)\}_m|\leq K$, $|\{\rho^c_m\}_m|\leq K$, $|\{\rho^P_m\}_m|\leq K$, $|\{\Delta_{c,m}\}_m|\leq K+1$, and $|\{\Delta_{P,m}\}_m|\leq K+1$ ($\Delta_{c,m}=\Delta_{P,m}=0$ when $m$ is not the first interval of some episode).
\end{proof}

\begin{lemma}
	\label{lem:Q diff}
	With probability at least $1-2\delta$, for all $m\leq K$, $\cQ^m_h(s, a)\leq \cQ^{\optpi,m}_h(s, a) + (\Delta_{c,m} + B\Delta_{P,m})T^{\optpi,m}_h(s, a)$.
\end{lemma}
\begin{proof}
	%This is clearly true when $\Delta_{[1,m]}\leq\rho_m$.
	%When $\Delta_{[1,m]}>\rho_m$, 
	We prove this by induction on $h$.
	The base case of $h=H+1$ is clearly true.
	For $h\leq H$, by \textbf{Test 3} and the induction step, we have $\cV^m_{h+1}(s)\leq\min\{B/2, \cV^{\optpi, m}_{h+1}(s) + (\Delta_{c,m} + B\Delta_{P,m})T^{\optpi,m}_{h+1}(s)\}\leq \cV^{\optpi, m}_{h+1}(s) + x^m_{h+1}(s)\leq B$ where $x^m_h(s)=\min\{B/2, (\Delta_{c,m} + B\Delta_{P,m})T^{\optpi,m}_h(s)\}$ and $\cV^{\optpi,m}_h(s)\leq\frac{B}{4}+8\eta_mT^{\optpi,m}_h(s)\leq\frac{B}{4}+8H\eta_m\leq \frac{B}{3}$.
	Thus, with probability at least $1-2\delta$,
	\begin{align*}
		&\cc^m(s, a) + \P^m_{s, a}\cV^m_{h+1} - b^m(s, a, \cV^m_{h+1})\\
		&\leq \cc^m(s, a) + \P^m_{s, a}(\cV^{\optpi, m}_{h+1} + x^m_{h+1}) - b^m(s, a, \cV^{\optpi, m}_{h+1} + x^m_{h+1}) \tag{\pref{lem:mvp}}\\
		&\overset{\text{(i)}}{\leq} \cc^m(s, a) + \tilP^m_{s, a}(\cV^{\optpi, m}_{h+1} + x^m_{h+1}) \tag{\pref{lem:dPV}}\\
		&\leq c^m(s, a) + 8\eta_m + \Delta_{c,m} + P^m_{s, a}(\cV^{\optpi, m}_{h+1} + x^m_{h+1}) + B\Delta_{P,m} \tag{\pref{lem:dc}}\\
		&\leq \cQ^{\optpi,m}_h(s, a) + (\Delta_{c,m}+B\Delta_{P,m})T^{\optpi,m}_h(s, a).
	\end{align*}
	Note that in (i) we use the fact that $|\{ \cV^{\optpi, m}_h + x^m_h \}_{m,h}| \leq (HK+1)^6$ since $|\{V^{\optpi,m}_h\}_{m, h}|\leq HK+1$, $|\{\rho^c_m\}_m|\leq K$, $|\{\rho^P_m\}_m|\leq K$, $|\{\Delta_{c,m}\}_m|\leq K+1$, $|\{\Delta_{P,m}\}_m|\leq K+1$ ($\Delta_{c,m}=\Delta_{P,m}=0$ when $m$ is not the first interval of some episode), and $|\{T^{\optpi,m}_h\}_{m,h}|\leq HK+1$. 
\end{proof}

\begin{lemma}
	\label{lem:bound V}
	With probability at least $1-2\delta$, for all $m\leq K$,
	if $\Delta_{c, m}\leq\rho^c_m$ and $\Delta_{P, m}\leq\rho^P_m$, then $\cQ^m_h(s, a)\leq \cQ^{\optpi,m}_h(s, a) + \eta_mT^{\optpi,m}_h(s, a)\leq B/2$.
	Moreover, if \textbf{Test 3} fails in interval $m$, then $\Delta_{c, [i^c_m,m+1]}>g^c(\nu^c_m+1)$ or $\Delta_{P, [i^P_m, m+1]}>g^P(\nu^P_m+1)$.%\tc{purple}{(be careful)}
\end{lemma}
\begin{proof}
	First note that $\cQ^{\optpi,m}_h(s, a)\leq\frac{B}{4}+8\eta_mT^{\optpi,m}_h(s, a)\leq\frac{B}{4}+8H\eta_m\leq \frac{B}{3}$.
	We prove the first statement by induction on $h$.
	The base case of $h=H+1$ is clearly true.
	For $h\leq H$, note that:
	\begin{align*}
		&\cc^m(s, a) + \P^m_{s, a}\cV^m_{h+1} - b^m(s, a, \cV^m_{h+1})\\
		&\leq \cc^m(s, a) + \P^m_{s, a}(\cV^{\optpi, m}_{h+1} + \eta_mT^{\optpi,m}_{h+1}) - b^m(s, a, \cV^{\optpi, m}_{h+1} + \eta_mT^{\optpi,m}_{h+1}) \tag{induction step and \pref{lem:mvp}}\\
		&\overset{\text{(i)}}{\leq} \cc^m(s, a) + \tilP^m_{s, a}(\cV^{\optpi, m}_{h+1} + \eta_mT^{\optpi,m}_{h+1}) \tag{\pref{lem:dPV}}\\
		&\leq c^m(s, a) + 8\eta_m + \rho^c_m + P^m_{s, a}(\cV^{\optpi, m}_{h+1} + \eta_mT^{\optpi,m}_{h+1}) + \rho^P_m(B/3 + H\eta_m) \tag{\pref{lem:dc}, $\Delta_{c,m}\leq\rho^c_m$, and $\Delta_{P,m}\leq\rho^P_m$}\\
		&\leq \cQ^{\optpi,m}_h(s, a) + \eta_mT^{\optpi,m}_h(s, a). \tag{$H\eta_m\leq B/12$}
	\end{align*}
	Note that in (i) we use the fact that $|\{ \cV^{\optpi, m}_h + \eta_mT^{\optpi,m}_h \}_{m,h}| \leq (HK+1)^6$ since $|\{V^{\optpi,m}_h\}_{m, h}|\leq HK+1$, $|\{\rho^c_m\}_m|\leq K$, $\{\rho^P_m\}_m\leq K$, and $|\{T^{\optpi,m}_h\}_{m,h}|\leq HK+1$. 
	The second statement is simply by the contraposition of the first statement.
\end{proof}

The next two lemmas are about \textbf{Test 1} and \textbf{Test 2}.

\begin{lemma}
	\label{lem:chic}
	With probability at least $1-4\delta$, for any $M'\leq K$, if $\Delta_{c,M'}\leq \rho^c_{M'}$, then
	\begin{align*}
		\sumic\sumhm(c^m_h - \hatc^m_h) \leq \tilO{ \sqrt{C_{[i^c_{M'}, M']}} + \sumic\sumhm\rbr{\sqrt{\frac{\barc^m_h}{\M^m_h}} + \frac{1}{\M^m_h} } } + \sumic\sumhm\rho^c_m \triangleq \chi^c_{M'}.
	\end{align*}
	Moreover, if \textbf{Test 1} fails in interval $M'$, then $\Delta_{c,M'} > \rho^c_{M'}$.
\end{lemma}
\begin{proof}
	Note that for any given $M'\leq M$, without loss of generality, we can offset the intervals and assume $i^c_{M'}=1$. 
	Then with probability at least $1-4\delta$, for any $M'\leq K$, assuming $i^c_{M'}=1$ we have
	\begin{align*}
		\summp\sumhm(c^m_h - \hatc^m_h) &= \summp\sumhm (c^m_h - c^m(s^m_h, a^m_h)) + \summp\sumhm (c^m(s^m_h, a^m_h) - \hatc^m_h) \\
		&\leq \tilO{\sqrt{C_{M'}}} + \summp\sumhm(c^m(s^m_h, a^m_h) - \hatc^m_h) \tag{\pref{lem:freedman} and \pref{lem:e2r}}\\
		&\leq \tilO{\sqrt{C_{M'}} + \summp\sumhm\rbr{\sqrt{\frac{\barc^m_h}{\M^m_h}} + \frac{1}{\M^m_h}} } + \summp\sumhm\rho^c_m. \tag{\pref{lem:c diff}, and $\Delta_{c,m}\leq \Delta_{c,M'}\leq \rho^c_{M'} \leq \rho^c_m$}
	\end{align*}
	The first statement is then proved by noting $i^c_{M'}=1$.
	The second statement is simply by the contraposition of the first statement.
\end{proof}

\begin{lemma}
	\label{lem:chip}
	With probability at least $1-16\delta$, for any $M'\leq K$, if $\Delta_{c,[i^P_{M'}, M']}\leq \bar{\rho}^c_{M'}\triangleq\min\{\frac{\B^{1.5}c_1}{\sqrt{\nu^P_{M'}}}, \frac{1}{2^8H}\}$ and $\Delta_{P,M'}\leq\rho^P_{M'}$, then
	\begin{align*}
		&\sumip\sumhm\rbr{\cV^m_{h+1}(s^m_{h+1}) - \P^m_h\cV^m_{h+1}} \leq \tilO{\sqrt{\sumip\sumhm\fV(\P^m_h, \cV^m_{h+1})} + \sumip\sumhm \sqrt{\frac{\fV(\P^m_h, \cV^m_{h+1})}{\N^m_h}} } \\
		&+ \tilO{ \sqrt{SA(\B + L_{c,[i^P_{M'}, M']})C_{[i^P_{M'}, M']}} + \sqrt{\B SA\nu^P_{M'}} + \B^{2.5}S^2AHL_{c,[i^P_{M'},M']} } + 4\sumip\sumhm\eta_m \triangleq \chi^P_{M'}.
		%&\summp\sumhm\rbr{\cV^m_{h+1}(s^m_{h+1}) - \P^m_h\cV^m_{h+1}} \leq \tilO{\sqrt{\summp\sumhm\fV(\P^m_h, \cV^m_{h+1})} + \sqrt{SAM'} }\\
		%& + \tilO{\sqrt{SAL_{c,[i^P_{M'}, M']}C_{M'}} + \summp\sumhm b^m_h + \B^{2.5}S^2A + SAL_{c,[i^P_{M'}, M']} } + \summp\sumhm\eta_m \triangleq \chi^P_{M'},
	\end{align*}
	%where $L_{c,M'}$ is the number of resets of cost counter within $[i^P_{M'}, M']$.
	Moreover, if \textbf{Test 2} fails in interval $M'$, then $\Delta_{c,[i^P_{M'}, M']} > \bar{\rho}^c_{M'}$ or $\Delta_{P,M'}>\rho^P_{M'}$.
\end{lemma}
\begin{proof}	
	For any $M'\leq K$, without loss of generality, we can offset the intervals and assume $i^P_{M'}=1$.
	Moreover, for any $m\leq M'$, we have $\Delta_{P,m}\leq\Delta_{P,M'}\leq\rho^P_{M'}\leq\rho^P_m$.
	Thus, with probability at least $1-2\delta$,
	\begin{align*}
		&\summp\sumhm\rbr{\cV^m_{h+1}(s^m_{h+1}) - \P^m_h\cV^m_{h+1}}\\
		&\leq \summp\sumhm(\cV^m_{h+1}(s^m_{h+1}) - P^m_h\cV^m_{h+1}) + \summp\sumhm(\tilP^m_h - \P^m_h)\cV^m_{h+1} + \summp\sumhm B(\rho^P_m + \n^m_h) \tag{$P^m_h\cV^m_{h+1}\leq \tilP^m_h\cV^m_{h+1} + B(\Delta_{P,m} + \n^m_h)$ and $\Delta_{P,m}\leq\rho^P_m$}\\
		&\leq \tilO{\sqrt{\summp\sumhm\fV(P^m_h, \cV^m_{h+1})} + \B SA} + \summp\sumhm(\tilP^m_h - \P^m_h)\cV^m_{h+1} + \summp\sumhm B\rho^P_m\tag{\pref{lem:freedman} and $\summp\sumhm\n^m_h\leq\summp\sumhm\frac{1}{\N^m_h}\leq SA$ by $L_{P,M'}=1$}\\
		&\leq \tilO{\sqrt{\summp\sumhm\fV(\P^m_h, \cV^m_{h+1})} + \B SA} + \summp\sumhm(\tilP^m_h - \P^m_h)\cV^m_{h+1} + \summp\sumhm 2B\rho^P_m,
	\end{align*}
	where the last inequality is by
	\begin{align*}
		\fV(P^m_h, \cV^m_{h+1}) &\leq P^m_h(\cV^m_{h+1} - \P^m_h\cV^m_{h+1})^2 \leq \tilP^m_h(\cV^m_{h+1} - \P^m_h\cV^m_{h+1})^2 + B^2(\Delta_{P,m} + \n^m_h) \tag{$\frac{\sum_ip_ix_i}{\sum_ip_i}=\argmin_z\sum_ip_i(x_i-z)^2$}\\
		&\leq 2\fV(\P^m_h, \cV^m_{h+1}) + \tilO{\frac{SB^2}{\N^m_h}} + B^2\rho^P_m, \tag{$\tilP^m_h(s')\leq 2\P^m_h(s') + \frac{1}{\N^m_h}$ by \pref{lem:e2r}, $\n^m_h\leq\frac{1}{\N^m_h}$, and $\Delta_{P,m}\leq\rho^P_m$}
	\end{align*}
	\pref{lem:sum}, $L_{P,M'}=1$, and AM-GM inequality.
	Now note that with probability at least $1-3\delta$,
	\begin{align*}
		&\summp\sumhm (\tilP^m_h - \P^m_h)\cV^m_{h+1} = \summp\sumhm \rbr{(\tilP^m_h - \P^m_h)\cV^{\star,m}_{h+1} + (\tilP^m_h - \P^m_h)(\cV^m_{h+1} - \cV^{\star,m}_{h+1})}\\
		&\leq \tilO{\summp\sumhm\rbr{\sqrt{\frac{\fV(\P^m_h, \cV^{\star, m}_{h+1})}{\N^m_h}} + \frac{S\B}{\N^m_h}} + \sqrt{S^2A\summp\sumhm\fV(P^m_h, \cV^m_{h+1} - \cV^{\star, m}_{h+1})} } + \summp\sumhm \frac{B\rho^P_m}{32} \tag{\pref{lem:dPV}, \pref{lem:dPv}, Cauchy-Schwarz inequality, \pref{lem:sum}, and $\Delta_{P,m}\leq\rho^P_m$}\\
		&\leq \tilO{\summp\sumhm\rbr{ \sqrt{\frac{\fV(\P^m_h, \cV^m_{h+1})}{\N^m_h}} + \frac{S\B}{\N^m_h}} + \sqrt{S^2A\summp\sumhm\fV(P^m_h, \cV^m_{h+1} - \cV^{\star, m}_{h+1})} } + \summp\sumhm\frac{B\rho^P_m}{16},
	\end{align*}
	where in the last step we apply 
	$$\summp\sumhm \sqrt{\frac{\fV(\P^m_h, \cV^{\star,m}_{h+1})}{\N^m_h}}\leq \summp\sumhm \rbr{\sqrt{\frac{\fV(\P^m_h, \cV^m_{h+1})}{\N^m_h}} + \sqrt{\frac{\fV(\P^m_h, \cV^m_{h+1} - \cV^{\star,m}_{h+1})}{\N^m_h}} }$$
	by $\sqrt{\var[X+Y]}\leq\sqrt{\var[X]}+\sqrt{\var[Y]}$ \citep[Lemma E.3]{cohen2021minimax} and
	\begin{align*}
		&\summp\sumhm\sqrt{\frac{\fV(\P^m_h, \cV^m_{h+1} - \cV^{\star,m}_{h+1})}{\N^m_h}} \leq \summp\sumhm\sqrt{ \frac{\P^m_h((\cV^m_{h+1} - \cV^{\star,m}_{h+1}) - P^m_h(\cV^m_{h+1} - \cV^{\star,m}_{h+1}))^2}{\N^m_h}} \tag{$\frac{\sum_ip_ix_i}{\sum_ip_i}=\argmin_z\sum_ip_i(x_i-z)^2$}\\
		&\leq \summp\sumhm\sqrt{ \frac{2\tilP^m_h((\cV^m_{h+1} - \cV^{\star,m}_{h+1}) - P^m_h(\cV^m_{h+1} - \cV^{\star,m}_{h+1}))^2}{\N^m_h}} + \tilO{\summp\sumhm\frac{B\sqrt{S}}{\N^m_h}} \tag{$\P^m_h(s')\leq 2\tilP^m_h(s') + \tilO{\frac{1}{\N^m_h}}$ by \pref{lem:e2r}}\\
		&\leq \summp\sumhm\sqrt{ \frac{2\fV(P^m_h, \cV^m_{h+1} - \cV^{\star,m}_{h+1})}{\N^m_h}} + \tilO{\summp\sumhm\frac{B\sqrt{S}}{\N^m_h} + \summp\sumhm B\sqrt{\frac{\Delta_{P,m}}{\N^m_h}} }\\
		&\leq \tilO{\sqrt{SA\summp\sumhm\fV(P^m_h, \cV^m_{h+1} - \cV^{\star,m}_{h+1})} + \summp\sumhm\frac{B\sqrt{S}}{\N^m_h}} + \summp\sumhm\frac{B\rho^P_m}{32}. \tag{Cauchy-Schwarz inequality, \pref{lem:sum}, $L_{P,M'}=1$, AM-GM inequality, and $\Delta_{P,m}\leq\rho^P_m$}
	\end{align*}
	%and finally apply Cauchy-Schwarz inequality on $\summp\sumhm\sqrt{ \frac{2\fV(P^m_h, \cV^m_{h+1} - \cV^{\star,m}_{h+1})}{\N^m_h}}$, AM-GM inequality on $\summp\sumhm B\sqrt{\frac{\Delta_{P,m}}{\N^m_h}}$, and $\Delta_{P,m}\leq\rho^P_m$.
	Now by \pref{lem:opt-check var sum}, $L_{P,M'}=1$, and AM-GM inequality, we have with probability $1-10\delta$,
	\begin{align*}
		&\sqrt{S^2A\summp\sumhm\fV(P^m_h, \cV^m_{h+1} - \cV^{\star, m}_{h+1})} \leq \tilO{\sqrt{SA\Lc C_{M'}} + \sqrt{\B SA (C_{M'}+M')} }\\ 
		&\qquad + \tilO{\B S^2A + \B S^{1.5}A\Lc + \sqrt{\B S^2A\summp (\Delta_{c,m}+\B\Delta_{P,m})H} }.
	\end{align*}
	Moreover, by $i^c_m\geq i^P_m$ and $\nu^c_m\leq \nu^P_m$ due to the reset rules, we have $\Delta_{c,m}\leq\Delta_{c,[i^P_{M'}, m]}\leq \Delta_{c,[i^P_{M'}, M']} \leq \bar{\rho}^c_{M'} \leq \bar{\rho}^c_m\leq \B^{1.5}\min\{\frac{c_1}{\sqrt{\nu^P_m}}, \frac{1}{2^8H}\} \leq \B^{1.5}\min\{\frac{c_1}{\sqrt{\nu^c_m}}, \frac{1}{2^8H}\} \leq \B^{1.5}\rho^c_m$. 
	Therefore, by $\Delta_{P,m}\leq\rho^P_m$ and AM-GM inequality,
	\begin{align*}
		\sqrt{\B S^2A\summp\sumhm (\Delta_{c,m}+\B\Delta_{P,m})} \leq \sqrt{\B^{2.5}S^2AH\summp(\rho^c_m + \B\rho^P_m)} \leq \B^{2.5}S^2AH + \summp\eta_m.
	\end{align*}
	Plugging these back, and by \pref{lem:sum}, $L_{P,M'}=1$, we obtain
	\begin{align*}
		&\summp\sumhm (\tilP^m_h - \P^m_h)\cV^m_{h+1} \leq \tilO{\summp\sumhm\rbr{ \sqrt{\frac{\fV(\P^m_h, \cV^m_{h+1})}{\N^m_h}} } + \sqrt{\B SA (C_{M'}+M')} }\\
		&\qquad + \tilO{ \sqrt{SA\Lc C_{M'}} + \B S^{1.5}A\Lc + \B^{2.5}S^2AH} + 2\summp\sumhm\eta_m.
		%&+ \sqrt{S^2A\rbr{ \B^2\sqrt{SAM'} + \B\sqrt{SAL_{c,M'}C_{M'}} + \B^2S^2A + \B^2SAL_{c,M'} + \B\summp\sumhm b^m_h + \B^{2.5}\summp\sumhm\eta_m }} \tag{\pref{lem:opt-til var sum} and $i^c_m\geq i^P_m$}\\
		%&\lesssim \summp\sumhm\rbr{ \sqrt{\frac{\fV(\P^m_h, \cV^m_{h+1})}{\N^m_h}} + \frac{S\B}{\N^m_h} + \eta_m + b^m_h} + \B^{2.5}S^2A + SAL_{c,M'} + \sqrt{SAM'} + \sqrt{SAL_{c,M'}C_{M'}}.
	\end{align*}
	Plugging this back and noting $i^P_{M'}=1$ completes the proof of the first statement.
	The second statement is simply by the contraposition of the first statement.
%	Plugging these back, we obtain:
%	\begin{align*}
%		&\summp\rbr{\sumhm c^m_h + c^m_{H_m+1} -  \cV^m_1(s^m_1)}\\
%		&\lesssim \summp\sumhm(\Ind_{s^m_{h+1}}-\tilP^m_h)\cV^m_{h+1} + \summp\sumhm\rbr{ \sqrt{\frac{\fV(\P^m_h, \cV^m_{h+1})}{\N^m_h}} + \frac{S\B}{\N^m_h} + \sqrt{\frac{\barc^m_h}{\M^m_h}} + \frac{1}{\M^m_h} + b^m_h - \frac{\eta_m}{4}}\\
%		&\qquad + \B^{1.5}S^2A + \sqrt{\B SAM'} + \sqrt{SAC_{M'}}\\
%		&\lesssim \summp\sumhm\rbr{ \sqrt{\frac{\barc^m_h}{\M^m_h}} + b^m_h} + \summp\sumhm\rbr{(\Ind_{s^m_{h+1}}-\tilP^m_h)\cV^m_{h+1} - \frac{\eta_m}{8}}\\
%		&\qquad + \B^{1.5}S^2A + \sqrt{\B SAM'} + \sqrt{SAC_{M'}}. \tag{\pref{lem:var sum}}
%	\end{align*}
	%This proves that \textbf{Test 2} will not be triggered.
	%Finally, applying \pref{lem:sum b} completes the proof, and it proves that \textbf{Test 2} will not be triggered.
\end{proof}

\subsection{\pfref{thm:mvp-test}}
\begin{proof}
	%Let $M'=K$.
	By $s^m_1=\sinit$, we decompose the regret as follows, with probability at least $1-2\delta$,
	\begin{align*}
		\rR_K &= \summk\rbr{\sumhm c^m_h + c^m_{H_m+1} - V^{\optpi, m}_1(s^m_1)}\\ 
		&= \summk\rbr{\sumhm c^m_h + c^m_{H_m+1} -  \cV^m_1(s^m_1)} + \summk\rbr{\cV^m_1(s^m_1) - \cV^{\optpi,m}_1(s^m_1)} + 8\T\summk\eta_m\\
		&\leq \summk\rbr{\sumhm c^m_h + c^m_{H_m+1} -  \cV^m_1(s^m_1)} + \summk (\Delta_{c,m} + B\Delta_{P,m})\T + 8\T\summk\eta_m \tag{\pref{lem:Q diff}}\\
		%&\leq \chi_{M'} + \summp (\Delta_{c,m} + B_m\Delta_{P,m})\T + \T\summp\eta_m. \tag{By \textbf{Test 2}}
	\end{align*}
	We first bound the first and the third term above separately.
	For the third term, we have:
	\begin{align*}
		\T\summk\eta_m & \leq \T\summk\rbr{\frac{c_1}{\sqrt{\nu^c_m}} + \frac{Bc_2}{\sqrt{\nu^P_m}}} = \tilO{\T\rbr{c_1\sum_{i=1}^{L_c}\sqrt{M^c_i} + \B c_2\sum_{i=1}^{L_P}\sqrt{M^P_i} } } \tag{$\sum_{i=1}^j\frac{1}{\sqrt{i}}=\bigo{\sqrt{j}}$}\\
		&= \tilO{ \T (c_1\sqrt{L_cK} + \B c_2\sqrt{L_PK}) } = \tilO{ \sqrt{\B SAL_cK} + \B\sqrt{SAL_PK} },
	\end{align*}
	where $M^c_i$ (or $M^P_i$) is the number of intervals between the $i$-th and $(i+1)$-th reset of cost (or transition) estimation, and the second last step is by Cauchy-Schwarz inequality.
	For the first term, define $\{\calI^c_i\}_{i=1}^{L_c}$ (or $\{\calI^P_i\}_{i=1}^{L_P}$) as a partition of $K$ episodes such that $\M$ (or $\N$) is reset in the last interval of each $\calI^c_i$ (or $\calI^P_i$) for $i<L_c$ (or $i<L_P$) and the last interval of $\calI^c_{L_c}$ (or $\calI^P_{L_P}$) is $K$.
	Also let $L=L_c+L_P$.
	Then with probability at least $1-20\delta$,
	\begin{align*}
		&\summk\rbr{\sumhm c^m_h + c^m_{H_m+1} -  \cV^m_1(s^m_1)} \leq \summk\sumhm \rbr{c^m_h + \cV^m_{h+1}(s^m_{h+1}) - \cV^m_h(s^m_h)} + \tilO{\B SAL} \tag{\pref{lem:new interval}}\\
		&\leq \summk\sumhm\rbr{ c^m_h - \hatc^m_h + \cV^m_{h+1}(s^m_{h+1}) - \P^m_h\cV^m_{h+1} + b^m_h - 8\eta_m} + \tilO{\B SAL} \tag{definition of $\cV^m_h(s^m_h)$}\\
		&= \sum_{i=1}^{L_c}\sum_{m\in\calI^c_i}\sumhm(c^m_h - \hatc^m_h) + \sum_{i=1}^{L_P}\sum_{m\in\calI^P_i}\sumhm(\cV^m_{h+1}(s^m_{h+1}) - \P^m_h\cV^m_{h+1}) + \summk\sumhm(b^m_h - 8\eta_m) + \tilO{\B SAL}\\
		&=\tilO{ \sqrt{L_cC_K} + \summk\sumhm\rbr{\sqrt{\frac{\barc^m_h}{\M^m_h}} + \frac{1}{\M^m_h} } + \sqrt{L_P\summk\sumhm\fV(\P^m_h, \cV^m_{h+1})} + \summk\sumhm b^m_h }\\
		&\qquad + \tilO{\B^{2.5}S^2AHL_c + \sqrt{\B SAL_P(C_K+K)} + \sqrt{SAL_cC_K} + HL_c + \B HL_P}, \tag{\textbf{Test 1} (\pref{lem:chic}), \textbf{Test 2} (\pref{lem:chip}), and Cauchy-Schwarz inequality}
	\end{align*}
	where $\tilo{HL_c + \B HL_P}$ is upper bound of the costs in intervals where \textbf{Test 1} fails or \textbf{Test 2} fails.
	By \pref{lem:c diff} and AM-GM inequality, with probability at least $1-3\delta$,
	\begin{align*}
		\summk\sumhm\rbr{\sqrt{\frac{\barc^m_h}{\M^m_h}} + \frac{1}{\M^m_h} } = \tilO{SAHL_c + \sqrt{SAL_cC_K}} + \summk\Delta_{c,m}.
	\end{align*}
	Following the proof of \pref{lem:sum bV}, we have $\sqrt{L_P\summk\sumhm\fV(\P^m_h, \cV^m_{h+1})}$ is dominated by the upper bound of $\summp\sumhm b^m_h$. 
	Thus with probability at least $1-\delta$,
	\begin{align*}
		&\sqrt{L_P\summk\sumhm\fV(\P^m_h, \cV^m_{h+1})} + \summk\sumhm b^m_h\\
		&=\tilO{ \sqrt{SAL_P\summk\sumhm\fV(P^m_h, \cV^m_{h+1})} + \B S^{1.5}AL_P + \B\sqrt{SAL_P\summk\sumhm\Delta_{P,m} } }\\
		&= \tilO{\sqrt{\B SAL_P(C_K+K)} + \sqrt{SAL_cC_K} + \B S^{1.5}AHL} + \summk(\Delta_{c,m}+\B\Delta_{P,m}),
	\end{align*}
	where in the last inequality we apply AM-GM inequality on $\B\sqrt{SAL_P\summk\sumhm\Delta_{P,m} } $, and note that with probability at least $1-11\delta$,
	\begin{align*}
		&\sqrt{SAL_P\summk\sumhm\fV(P^m_h, \cV^m_{h+1})}\\ 
		&=\tilO{ \sqrt{SAL_P\summk\sumhm\fV(P^m_h, \cV^{\star,m}_{h+1})} + \sqrt{SAL_P\summk\sumhm\fV(P^m_h, \cV^m_{h+1} - \cV^{\star,m}_{h+1})} }\tag{$\var[X+Y]\leq2\var[X]+2\var[Y]$ and $\sqrt{a+b}\leq\sqrt{a}+\sqrt{b}$}\\
		&= \tilO{\sqrt{\B SAL_P(C_K+K)} + \sqrt{SAL_cC_K} + \B S^{1.5}AHL} + \summp(\Delta_{c,m}+\B\Delta_{P,m}). \tag{\pref{lem:optvar check sum}, \pref{lem:opt-check var sum}, and AM-GM inequality}
	\end{align*}
	Putting everything together, we have
	\begin{align*}
		\rR_K &= \tilO{\sqrt{SA(L_c+\B L_P)(C_K+\B K)} + \B^{2.5} S^2AHL + \summk(\Delta_{c,m}+\B\Delta_{P,m})\T}.
	\end{align*}
	Now by $\rR_K\geq C_K - 4\B K$, solving a quadratic inequality (\pref{lem:quad}) w.r.t $C_K$ and plugging the bound on $C_K$ back, we obtain
	\begin{align*}
		\rR_K = \tilO{ \sqrt{\B SAL_cK} + \B\sqrt{SAL_PK}+ \B^{2.5} S^2AHL + \summk(\Delta_{c,m}+\B\Delta_{P,m})\T }.
	\end{align*}
	It suffices to bound the last term above.
	By the periodic resets of $\M$ and $\N$ (\pref{line:period c} and \pref{line:period p} of \pref{alg:mvp-test}), the number of intervals between consecutive resets of $\M$ and $\N$ are upper bounded by $W_c$ and $W_P$ respectively.
	Thus,
	\begin{align*}
		&\summk(\Delta_{c,m}+\B\Delta_{P,m})\T \leq \summk(\Delta_{c,f^c(m)}+\B\Delta_{P,f^P(m)})\T \leq (W_c\Delta_c+\B W_P\Delta_P)\T\\
		&=\tilO{(\B SA\T\Delta_c)^{1/3}K^{2/3} + \B(SA\T\Delta_P)^{1/3}K^{2/3} + (\Delta_c + \B\Delta_P)\T},
	\end{align*}
	where the last step is simply by the chosen values of $W_c$ and $W_P$.
	%where $f^c(m)$ (or $f^P(m)$) is the earliest interval at or after interval $m$ in which the learner resets $\M$ (or $\N$), and the last step is simply by the chosen values of $W_c$ and $W_P$.
	Plugging this back and applying \pref{lem:bound L} completes the proof.
\end{proof}

\begin{lemma}
	\label{lem:bound L}
	With probability at least $1-2\delta$, \pref{alg:mvp-test} with $p=1/\B$ ensures
	\begin{align*}
		L_c &= \tilO{(\B SA)^{-1/3}(\T\Delta_c)^{2/3}K^{1/3} + \B(SA)^{-1/3}(\T\Delta_P)^{2/3}K^{1/3} + H(\Delta_c + \B \Delta_P)},\\
		L_P &= \tilO{(\B SA)^{-1/3}(\T\Delta_c)^{2/3}K^{1/3}/\B + (SA)^{-1/3}(\T\Delta_P)^{2/3}K^{1/3} + H(\Delta_c + \Delta_P) }.
	\end{align*}
\end{lemma}
\begin{proof}
	We consider the number of resets of $\M$ and $\N$ from each test separately.
	By \pref{lem:chic} and \pref{lem:bound l}, there are at most $\tilo{(c_1^{-1}\Delta_c)^{2/3}K^{1/3} + H\Delta_c}$ resets of $\M$ triggered by \textbf{Test 1}.
	By \pref{lem:chip} and \pref{lem:bound l}, there are at most $\tilo{((\B^{-1.5}c_1^{-1}\Delta_c)^{2/3} + (c_2^{-1}\Delta_P)^{2/3})K^{1/3} + H(\Delta_c+\Delta_P)}$ resets of $\M$ and $\N$ triggered by \textbf{Test 2}.
	
	Next, we consider \textbf{Test 3}.
	Define $\Ind^c_m=\Ind\{\Delta_{c, [i^c_m, m+1]}>g^c(\nu^c_m+1)\}$ and $\Ind^P_m=\Ind\{\Delta_{P, [i^P_m, m+1]}>g^P(\nu^P_m+1)\}$.
	Note that whenever \textbf{Test 3} fails in interval $m$, we have $\Ind^c_m=1$ or $\Ind^P_m=1$ by \pref{lem:bound V}.
	We partition $K$ intervals into segments $\calI_1,\ldots,\calI_{N_c}$, such that in the last interval of each $\calI_i$ with $i<N_c$ denoted by $m$, \textbf{Test 3} fails and $\Ind^c_m=1$.
	Since $\nu^c$ is reset whenever \textbf{Test 3} fails, we have $\Delta_{\calI_i\cup\{m+1\}}\geq\Delta_{[i^c_m, m+1]}>g^c(\nu^c_m+1)\geq g^c(|\calI_i|+1)$.
	By \pref{lem:bound l}, we obtain $N_c=\tilo{(c_1^{-1}\Delta_c)^{2/3}K^{1/3} + H\Delta_c}$.
	
	Now define $\fA_m$ as the indicator that \textbf{Test 3} fails in interval $m$ and $\Ind^P_m=1$.
	Also define $\fA'_m$ as the indicator that \textbf{Test 3} fails and $\N$ is reset in interval $m$, and $\Ind^P_m=1$.
	We then partition $K$ intervals into segments $\calI'_1,\ldots,\calI'_{N_P}$, such that in the last interval of each $\calI'_i$ with $i<N_P$ denoted by $m$, $\fA'_m=1$.
	Since $\nu^P$ is reset in interval $m$ when $\fA'_m=1$, we have $\Delta_{\calI'_i\cup\{m+1\}}\geq\Delta_{[i^P_m, m+1]}>g^P(\nu^P_m+1)\geq g^P(|\calI'_i|+1)$.
	By \pref{lem:bound l}, we have $N_P=\tilo{(c_2^{-1}\Delta_P)^{2/3}K^{1/3} + H\Delta_P}$.
	Moreover, by \pref{lem:e2r} and the reset rule of \textbf{Test 3}, we have $p\sum_m\fA_m=\tilo{\sum_m\fA'_m}$ with probability at least $1-\delta$, which gives $\sum_m\fA_m=\tilo{N_P/p}$.
	
	Since $\Ind^c_m=1$ or $\Ind^P_m=1$ when \textbf{Test 3} fails in interval $m$, the total number of times that \textbf{Test 3} fails $N_3 \leq N_c + \sum_m\fA_m=\tilo{ (c_1^{-1}\Delta_c)^{2/3}K^{1/3} + \B(c_2^{-1}\Delta_P)^{2/3}K^{1/3} + H(\Delta_c+\B\Delta_P) }$.
	Now by the reset rule of \textbf{Test 3}, the number of times $\M$ is reset due to \textbf{Test 3} is upper bounded by $N_3$, and the number of times $\N$ is reset due to \textbf{Test 3} is upper bounded by $\tilo{pN_3}$ with probability at least $1-\delta$ by \pref{lem:e2r}.
	Finally, by \pref{line:period c} and \pref{line:period p} of \pref{alg:mvp-test}, there are at most $\frac{K}{W_c}$ resets of $\M$ and $\frac{K}{W_P}$ resets of $\N$ respectively due to periodic restarts.
	Putting all cases together, we have
	\begin{align*}
		L_c &=\tilO{ (c_1^{-1}\Delta_c)^{2/3}K^{1/3} + \B(c_2^{-1}\Delta_P)^{2/3})K^{1/3} + H(\Delta_c+\B\Delta_P) + K/W_c }\\
		%&=\tilO{ \frac{M'}{W_c} + H\Delta_c + \B\rbr{ \frac{M'}{W_P} + H\Delta_P }},\\
		&=\tilO{ (\B SA)^{-1/3}(\T\Delta_c)^{2/3}K^{1/3} + \B(SA)^{-1/3}(\T\Delta_P)^{2/3}K^{1/3} + H(\Delta_c + \B \Delta_P)},
	\end{align*}
	and
	\begin{align*}
		L_P &=\tilO{ \frac{1}{\B}(c_1^{-1}\Delta_c)^{2/3}K^{1/3} + (c_2^{-1}\Delta_P)^{2/3}K^{1/3} + H(\Delta_c+\Delta_P) + K/W_P} \\
		%&=\tilO{ \rbr{\frac{M'}{W_P} + H\Delta_P + p\rbr{\frac{M'}{W_c} + H\Delta_c}} + \rbr{\frac{1}{\B}\frac{M'}{W_c} + H\Delta_c + \frac{M'}{W_P} + H\Delta_P} }\\
		%&\lesssim \frac{M'}{W_P} + H\Delta_P + \frac{1}{\B}\frac{M'}{W_c} + H\Delta_c\\
		&= \tilO{\frac{(\B SA)^{-1/3}(\T\Delta_c)^{2/3}K^{1/3}}{\B} + (SA)^{-1/3}(\T\Delta_P)^{2/3}K^{1/3} + H(\Delta_c+\Delta_P) }.
	\end{align*}
	This completes the proof.
\end{proof}

\subsection{Auxiliary Lemmas}

\begin{lemma}
	\label{lem:optvar check sum}
	With probability at least $1-\delta$, for any $M'\leq K$,
	$\summp\sumhm\fV(P^m_h, \cV^{\star,m}_{h+1}) =\tilO{ \B C_{M'} + \B M' + \B^2}$.
\end{lemma}
\begin{proof}
	Applying \pref{lem:sum var} with $\norm{\cV^{\star,m}_h}_{\infty}\leq B$, with probability at least $1-\delta$,
	\begin{align*}
		&\summp\sumhm\fV(P^m_h, \cV^{\star,m}_{h+1})\\ 
		&= \tilO{\summp \cV^{\star,m}_{H_m+1}(s^m_{H_m+1})^2 + \summp\sumhm \B(\cV^{\star, m}_h(s^m_h) - P^m_h\cV^{\star, m}_{h+1})_+ + \B^2}\\
		&= \tilO{\B C_{M'} + \B M' + \B^2 },
	\end{align*}
	where in the last step we apply
	\begin{align*}
		(\cV^{\star,m}_h(s^m_h) - P^m_h\cV^{\star,m}_{h+1})_+ &\leq (\cQ^{\star,m}_h(s^m_h, a^m_h) - P^m_h\cV^{\star,m}_{h+1})_+ \leq c^m(s^m_h, a^m_h) + 8\eta_m\\ 
		&\leq c^m(s^m_h, a^m_h) + 1/H,
	\end{align*}
	and also \pref{lem:e2r}.
\end{proof}

\begin{lemma}
	\label{lem:opt-check var sum}
	With probability at least $1-10\delta$, for any $M'\leq K$, $\summp\sumhm\fV(P^m_h, \cV^{\star,m}_{h+1} - \cV^m_{h+1}) = \tilo{\B\sqrt{SA\Lc C_{M'}} + \B\sqrt{\B SA\Lp (C_{M'}+M')} + \B^2 S^2A\Lp + \B^2SA\Lc + \summp \B(\Delta_{c,m}+\B\Delta_{P,m})H}$.
\end{lemma}
\begin{proof}
	Let $z^m_h=\min\{B/2, (\Delta_{c,m}+B\Delta_{P,m})H\}\Ind\{h\leq H\}$.
	By \pref{lem:bound cQ star}, we have $\cV^{\star,m}_h(s) + z^m_h \geq \cV^m_h(s)$.
	%Then we apply \pref{lem:sum var} and
	Moreover, by \pref{lem:new interval}, we have
	\begin{align*}
		&\summp (\cV^{\star,m}_{H_m+1}(s^m_{H_m+1}) + z^m_{H_m+1} - \cV^m_{H_m+1}(s^m_{H_m+1}))^2\\
		&\leq \summp (z^m_{H_m+1})^2\Ind\{s^m_{H_m+1}=g\} + 64\B^2\summp\Ind\{H_m<H, s^m_{H_m+1}\neq g\}\\
		&= \tilO{\B \summp (\Delta_{c,m}+\B\Delta_{P,m})H + \B^2SA\L}.
		%= \tilO{\B^2\summp\Ind\{H_m<H\}} = \tilO{\B^2SA\L}.
	\end{align*}
	and
	\begin{align*}
		(*)&=\summp \B\sumhm(\cV^{\star, m}_h(s^m_h) - \cV^m_h(s^m_h) - P^m_h\cV^{\star, m}_{h+1} + P^m_h\cV^m_{h+1} + z^m_h - z^m_{h+1})_+\\
		&\leq \summp \B\sumhm\rbr{c^m(s^m_h, a^m_h) + 8\eta_m + \tilP^m_h\cV^m_{h+1} - \cV^m_h(s^m_h) + B(\Delta_{P,m}+\n^m_h)}_+ + \B\summp(z^m_1 - z^m_{H_m+1}) \tag{$\cV^{\star, m}_h(s^m_h)\leq \cQ^{\star, m}_h(s^m_h, a^m_h)$, $z^m_h\geq z^m_{h+1}$, and $P^m_{h+1}\cV^m_{h+1}\leq \tilP^m_{h+1}\cV^m_{h+1} + B(\Delta_{P,m}+\n^m_h)$}\\
		&\leq \summp \B\sumhm( c^m(s^m_h, a^m_h) - \hatc^m_h + (\tilP^m_h - \P^m_h)\cV^{\star,m}_{h+1} + (\tilP^m_h - \P^m_h)(\cV^m_{h+1} - \cV^{\star,m}_{h+1}) + b^m_h )_+\\ 
		&\qquad + \tilO{\summp\B(\Delta_{c,m} + \B\Delta_{P,m})H + \B^2\summp\sumhm\n^m_h } \tag{definition of $\cV^m_h(s^m_h)$}.
	\end{align*}
	Now by \pref{lem:c diff}, \pref{lem:dPV}, \pref{lem:dPv}, and $\n^m_h\leq\frac{1}{\N^m_h}$, we continue with
	\begin{align*}
		(*)&= \tilO{\B\rbr{\sqrt{SA\Lc C_{M'}} + SA\Lc + \summp\sumhm\rbr{\sqrt{\frac{\fV(P^m_h, \cV^{\star,m}_{h+1})}{\N^m_h}} + \sqrt{\frac{S\fV(P^m_h, \cV^m_{h+1} - \cV^{\star,m}_{h+1})}{\N^m_h}} }  }  }\\ 
		&\qquad + \tilO{ \summp\sumhm\frac{\B^2S}{\N^m_h} + \summp\B(\Delta_{c,m}+\B\Delta_{P,m})H + \B\summp\sumhm b^m_h}\\
		&= \tilO{\B\sqrt{SA\Lc C_{M'}} + \B SA\Lc + \B\sqrt{SA\Lp\summp\sumhm\fV(P^m_h, \cV^{\star,m}_{h+1})} + \B^2 S^2A\Lp }\\
		&\qquad + \tilO{\B\sqrt{S^2A\Lp\summp\sumhm\fV(P^m_h, \cV^{\star,m}_{h+1} - \cV^m_{h+1})} + \summp\B(\Delta_{c,m}+\B\Delta_{P,m})H},
	\end{align*}
	where in the last step we apply Cauchy-Schwarz inequality, \pref{lem:sum}, \pref{lem:sum bV}, $\var[X+Y]\leq2\var[X]+2\var[Y]$, and AM-GM inequality.
	Finally, by \pref{lem:optvar check sum}, we continue with
	\begin{align*}
		(*) &= \tilO{\B\sqrt{SA\Lc C_{M'}} + \B SA\Lc + \B\sqrt{\B SA\Lp (C_{M'}+M')} + \B^2 S^2A\Lp }\\
		&\qquad + \tilO{\B\sqrt{S^2A\Lp\summp\sumhm\fV(P^m_h, \cV^{\star,m}_{h+1} - \cV^m_{h+1})} + \summp \B(\Delta_{c,m}+\B\Delta_{P,m})H}.
	\end{align*}
	%Applying \pref{lem:sum var} and solving a quadratic inequality (\pref{lem:quad}) completes the proof.
	Applying \pref{lem:sum var} on value functions $\{ \cV^{\star,m}_h + z^m_h - \cV^m_h \}_{m,h}$ (constant offset does not change the variance) and plugging in the bounds above, we have
	\begin{align*}
		&\summp\sumhm\fV(P^m_h, \cV^{\star,m}_{h+1} - \cV^m_{h+1}) = \summp\sumhm\fV(P^m_h, \cV^{\star,m}_{h+1} + z^m_h - \cV^m_{h+1})\\
		&=\tilO{\B\sqrt{SA\Lc C_{M'}} + \B^2SA\Lc + \B\sqrt{\B SA\Lp (C_{M'}+M')} + \B^2 S^2A\Lp }\\
		&\qquad + \tilO{\B\sqrt{S^2A\Lp\summp\sumhm\fV(P^m_h, \cV^{\star,m}_{h+1} - \cV^m_{h+1})} + \summp \B(\Delta_{c,m}+\B\Delta_{P,m})H}.
	\end{align*}
	Then solving a quadratic inequality w.r.t $\summp\sumhm\fV(P^m_h, \cV^{\star,m}_{h+1} - \cV^m_{h+1})$ (\pref{lem:quad}) completes the proof.
\end{proof}

\subsection{\pfref{thm:two phase}}
We first prove a general regret guarantee of \pref{alg:two phase}, from which \pref{thm:two phase} is a direct corollary.

\begin{theorem}
	\label{thm:two phase general}
	Suppose $\frA_1$ ensures $\rR_K\leq R^1$ when $s^m_1=\sinit$ for $m\leq K$, and $\frA_2$ ensures $R_{K'}\leq R^2(K')$ for any $K'\leq K$ such that $R^2(k)$ is sub-linear w.r.t $k$.
	Then \pref{alg:two phase} ensures $R_K = \tilo{R^1}$ (ignoring lower order terms).
\end{theorem}
\begin{proof}
	Let $\calI_k$ be the set of intervals in episode $k$, and $m^k_i$ be the $i$-th interval of episode $k$ (if exists).
	The regret is decomposed as:
	\begin{align*}
		R_K = \sumk\sbr{\sum_{h=1}^{H_{m^k_1}}c^{m^k_1}_h + c^{m^k_1}_{H_{m^k_1}+1} - \optV_k(s^k_1) } + \sumk\sbr{\sum_{m\in\calI_k\setminus \{m^k_1\}}\sum_{h=1}^{H_m}c^m_h - c^{m^k_1}_{H_{m^k_1}+1}}.
	\end{align*}
	Note that $V^{\optpi,m^k_1}_1(s^{m^k_1}_1) \leq \optV_k(s^k_1) + \B/K$ by \pref{lem:hitting}.
	Therefore,
	\begin{align*}
		\sumk\sbr{\sum_{h=1}^{H_{m^k_1}}c^{m^k_1}_h + c^{m^k_1}_{H_{m^k_1}+1} - \optV_k(s^k_1) } &\leq  \sumk\sbr{\sum_{h=1}^{H_{m^k_1}}c^{m^k_1}_h + c^{m^k_1}_{H_{m^k_1}+1} - V^{\optpi,m^k_1}_1(s^{m^k_1}_1) } + \B\\ 
		&\leq R^1 + \B.
		%&= \tilO{  \B(SA\T\Delta_P)^{1/3}K^{2/3} + (\B SA\T\Delta_c)^{1/3}K^{2/3} }. \tag{\pref{thm:mvp-test} ignoring lower order terms}
	\end{align*}
	For the second term, note that $c^{m^k_1}_{H_{m^k_1}+1}=2\B$ if $s^{m^k_2}_1$ exists.
	Define $K_f=\sumk\Ind\{|\calI_k|>1\}$, we have (define $s_1^{m^k_2}=g$ if $m^k_2$ does not exist)
	\begin{align*}
		\sumk\sbr{\sum_{m\in\calI_k\setminus \{m^k_1\}}\sumhm c^m_h - c^{m^k_1}_{H_{m^k_1}+1}} &\leq \sumk\rbr{\sum_{m\in\calI_k\setminus\{m^k_1\}}\sumhm c^m_h - \optV_k(s^{m^k_2}_1) } - \B K_f\\
		&\leq R^2(K_f) - \B K_f,
		%&\leq \tilO{(\B SA\Delta_c\Tmax)^{1/3}{K_f}^{2/3} + \B(SA\Delta_P\Tmax)^{1/3}{K_f}^{2/3}}  - \B K_f \tag{\pref{thm:any episode}}\\
		%&= \tilO{ SA\Tmax\Delta_c/\B^2 + SA\Tmax\Delta_P}.
	\end{align*}
	which is a lower order term since $R^2(K_f)$ is sub-linear w.r.t $K_f$.
	Putting everything together completes the proof.
\end{proof}

We are now ready to prove \pref{thm:two phase}.
\begin{proof}
	We simply apply \pref{thm:two phase general} with $R^1$ determined by \pref{thm:mvp-test} and $R^2$ determined by  \pref{thm:any episode}.
\end{proof}

\section{Omitted Details in \pref{sec:unknown}}
\label{app:unknown}
% !TEX root = main.tex

In this section, we present all proofs and details of learning without the knowledge of non-stationarity.
We first provide a base algorithm in \pref{app:base}.
The rest of this section then discusses the meta algorithm MASTER adopted from \citep{wei2021non}, and its regret guarantee combining with the base algorithm.

\subsection{Base Algorithm}
\label{app:base}

\setcounter{AlgoLine}{0}
\begin{algorithm}[t]
	\caption{MVP-Base}
	\label{alg:mvp-base}
	\SetKwFunction{update}{Update}
	\SetKwProg{proc}{Procedure}{}{}
	
	\textbf{Parameters:} failure probability $\delta$.
	
	\textbf{Initialize:} $\hatchi\leftarrow 0$, and for all $(s, a, s')$, $\C(s, a)\leftarrow 0$, $\M(s, a)\leftarrow 0$, $\N(s, a)\leftarrow 0$, $\N(s, a, s')\leftarrow 0$.
	
	\textbf{Initialize:} \update{$1$}.
	
	\For{$m=1,\ldots,M$}{
		\For{$h=1,\ldots,H$}{
			Play action $a^m_h=\argmin_a\cQ_h(s^m_h, a)$, receive cost $c^m_h$ and next state $s^m_{h+1}$.
			
			$\C(s^m_h, a^m_h)\leftarrow c^m_h$, $\M(s^m_h, a^m_h)\overset{+}{\leftarrow}1$, $\N(s^m_h, a^m_h)\overset{+}{\leftarrow}1$, $\N(s^m_h, a^m_h, s^m_{h+1})\overset{+}{\leftarrow}1$.
			
			\If{$s^m_{h+1}=g$ or $\M(s^m_h, a^m_h)=2^l$ or $\N(s^m_h, a^m_h)=2^l$ for some integer $l\geq 0$}{
				\textbf{break} (which starts a new interval).
			}
		}
		
		$\hatchi\overset{+}{\leftarrow} C^m - \cV_1(s^m_1)$.
		
		\nl\lIf{$\hatchi > \chi_m$ (defined in \pref{lem:chi mvp})}{terminate. \textbf{(Test 1)}}\label{line:test base}
		
		\update{$m+1$}.
	
		\nl\lIf{$\norm{\cV_h}>B/2$ for some $h$ \textbf{(Test 2)}}{terminate. }\label{line:test2 base}
	}
	
	\proc{\update{$m$}}{
		$\cV_{H+1}(s)\leftarrow2\B\Ind\{s\neq g\}$, $\cV_h(g)\leftarrow0$ for all $h\leq H$, and $\iota\leftarrow 2^{11}\cdot\ln\big(\frac{2SAHKm}{\delta}\big)$.
		
		\nl $\eta\leftarrow \min\{\frac{\B S\sqrt{A}}{\T\sqrt{m}}, \frac{1}{2^8H}\}$.\label{line:correct base}
		
		\For{all $(s, a)$}{
			$\N^+(s, a)\leftarrow\max\{1, \N(s, a)\}$, $\M^+(s, a)\leftarrow\max\{1, \M(s, a)\}$, $\barc(s, a)\leftarrow \frac{\C(s, a)}{\M^+(s, a)}$,  
			
			$\P_{s, a}(\cdot)\leftarrow\frac{\N(s, a, \cdot)}{\N^+(s, a)}$, 
			$\hatc(s, a)\leftarrow \max\Big\{0, \barc(s, a) - \sqrt{\frac{\barc(s, a)\iota}{\M^+(s, a)}} - \frac{\iota}{\M^+(s, a)}\Big\}$, 
			
			\nl $\cc(s, a)\leftarrow\hatc(s, a) + 8\eta$.\label{line:compute}
		}
		
		\For{$h=H,\ldots,1$}{
		    $b_h(s, a)\leftarrow \max\cbr{7\sqrt{\frac{\fV(\P_{s, a}, \cV_{h+1})\iota}{\Np(s, a)}}, \frac{49B\sqrt{S}\iota}{\Np(s, a)}}$ for all $(s, a)$.
		
			$\cQ_h(s, a)=\max\{0, \cc(s, a) + \P_{s, a}\cV_{h+1} - b_h(s, a)\}$ all $(s, a)$.
			
			$\cV_h(s)=\argmin_a\cQ_h(s, a)$ for all $s$.
		}
	}
	
%	\proc{\update{$m$}}{
%		$\cV_{H+1}(s)=2\B\Ind\{s\neq g\}$ and $\cV_h(g)=0$ for all $h\leq H$.
%		
%		\nl $\eta\leftarrow \min\{\frac{\B S\sqrt{A}}{\T\sqrt{m}}, \frac{1}{2^8H}\}$.\label{line:correct base}
%		
%		\For{all $(s, a)$}{
%			$\N^+(s, a)\leftarrow\max\{1, \N(s, a)\}$, $\M^+(s, a)\leftarrow\max\{1, \M(s, a)\}$, $\barc(s, a)\leftarrow \frac{\C(s, a)}{\M^+(s, a)}$, $\iota\leftarrow 2^{11}\cdot\ln\frac{2SAHKm}{\delta}$, $\hatc(s, a)\leftarrow \max\{0, \barc(s, a) - \sqrt{\frac{\barc(s, a)\iota}{\M^+(s, a)}} - \frac{\iota}{\M^+(s, a)}\}$, $\P_{s, a}(\cdot)\leftarrow\frac{\N(s, a, \cdot)}{\N^+(s, a)}$, $b(s, a)\leftarrow \max\cbr{7\sqrt{\frac{\fV(\P_{s, a}, \cV_{h+1})\iota}{\Np(s, a)}}, \frac{49B\sqrt{S}\iota}{\Np(s, a)}}$, $\cc(s, a)\leftarrow\hatc(s, a) + 8\eta$.
%		}
%		
%		\For{$h=H,\ldots,1$ and all $(s, a)$}{
%			$\cQ_h(s, a)=\max\{0, \cc(s, a) + \P_{s, a}\cV_{h+1} - b(s, a)\}$.
%			
%			$\cV_h(s)=\argmin_a\cQ_h(s, a)$.
%		}
%	}

\end{algorithm}

We first present the base algorithm used in MASTER (\pref{alg:mvp-base}).
The main idea is again incorporating a correction term to penalize long horizon policy and has the effect of cancelling the non-stationarity along the learner's trajectory when it is not too large (\pref{line:correct base}).
When the non-stationarity is large, on the other hand, we detect it through two non-stationary tests (\pref{line:test base} and \pref{line:test2 base}), and reset the knowledge of the environment (more details to follow).

\textbf{Test 1} is a combination of the first two tests of \pref{alg:mvp-test}, which directly checks whether the estimated regret is too large.
This is also similar to the second test of the MASTER algorithm \citep{wei2021non}.
\textbf{Test 2} is the same as the third test of \pref{alg:mvp-test}, which guards the magnitude of the estimated value function.
%We also incorporate a simpler non-stationarity test inspecting the estimated regret of the learner (\pref{line:test base}).
When tests fail, the algorithm directly terminate instead of resetting some accumulators.
Note that the status of $\M$ and $\N$ are completely identical in this algorithm, but we still maintain them separately so that the auxiliary lemmas in \pref{app:pre} are still applicable.
The rest of the algorithm largely follows the design of \pref{alg:MVP-SSP}.

\paragraph{Notations} Note that here $\M$ and $\N$ are only reset at the initialization step.
Thus, $i^c_m=i^P_m=1$, $L_{c,m}=L_{P,m}=1$, $\Delta_{c,m}=\Delta_{c,[1,m]}$ and $\Delta_{P,m}=\Delta_{P,[1,m]}$.
Let $\Delta'_{m}=(\Delta_{c,m}+ B\Delta_{P,m})$ and denote by $\eta_m$, $\cQ^m_h$, $\cV^m_h$ the value of $\eta$, $\cQ_h$, and $\cV_h$ at the beginning of interval $m$.
Denote by $\cc^m$ the value of $\cc$ at the beginning of interval $m$ and define $\cc^m_h=\cc(s^m_h, a^m_h)$.
Also define $\cQ^{\optpi,m}_h$ and $\cV^{\optpi,m}_h$ as the action-value function and value function w.r.t cost $c^m(s, a) + 8\eta_m$, transition $P^m$, and policy $\optpi_{k(m)}$.

\begin{lemma}
	\label{lem:tiloptQ base}
	With probability at least $1-2\delta$, if \pref{alg:mvp-base} does not terminate up to interval $m\leq K$, then $\cQ^m_h(s, a)\leq \cQ^{\optpi,m}_h(s, a) + \Delta'_{m}T^{\optpi,m}_h(s, a)$.
\end{lemma}
\begin{proof}
	%This is clearly true when $\Delta_{[1,m]}\leq r_m$.
	%When $\Delta_{[1,m]}> r_m$, 
	We prove this by induction on $h$.
	The base case of $h=H+1$ is clearly true.
	For $h\leq H$, by \textbf{Test 2} and the induction step, we have $\cV^m_{h+1}(s)\leq\min\{B/2, \cV^{\optpi, m}_{h+1}(s) + \Delta'_mT^{\optpi,m}_{h+1}(s)\}\leq \cV^{\optpi, m}_{h+1}(s) + x^m_{h+1}(s)\leq B$ where $x^m_h(s)=\min\{B/2, \Delta'_mT^{\optpi,m}_h(s)\}$.
	Thus,
	\begin{align*}
		&\cc^m(s, a) + \P^m_{s, a}\cV^m_{h+1} - b^m(s, a, \cV^m_{h+1})\\
		&\leq \cc^m(s, a) + \P^m_{s, a}(\cV^{\optpi, m}_{h+1} + x^m_{h+1}) - b^m(s, a, \cV^{\optpi, m}_{h+1} + x^m_{h+1}) \tag{\pref{lem:mvp}}\\
		&\overset{\text{(i)}}{\leq} \cc^m(s, a) + \tilP^m_{s, a}(\cV^{\optpi, m}_{h+1} + x^m_{h+1}) \tag{\pref{lem:dPV}}\\
		&\leq c^m(s, a) + 8\eta_m + \Delta_{c,m} + P^m_{s, a}(\cV^{\optpi, m}_{h+1} + x^m_{h+1}) + \Delta_{P,m}B \tag{\pref{lem:dc}}\\
		&\leq \cQ^{\optpi,m}_h(s, a) + \Delta'_mT^{\optpi,m}_h(s, a).
	\end{align*}
	Note that in (i) we use the fact that $|\{ \cV^{\optpi, m}_h + x^m_h \}_{m,h}| \leq (HK+1)^6$ since $|\{V^{\optpi,m}_h\}_{m, h}|\leq HK+1$, $|\{\eta_m\}_m|\leq K+1$, $|\{\Delta'_m\}_m|\leq K+1$, and $|\{T^{\optpi,m}_h\}_{m,h}|\leq HK+1$.
\end{proof}

\begin{lemma}
	\label{lem:optQ base}
	With probability at least $1-2\delta$, for all $m\leq K$, if $\Delta'_{m}\leq\eta_m$, then $\cQ^m_h(s, a)\leq \cQ^{\optpi,m}_h(s, a) + \eta_mT^{\optpi,m}_h(s, a)\leq B/2$.
	Moreover, if \textbf{Test 2} fails in interval $m$, then $\Delta'_{m+1}>\eta_{m+1}$.
\end{lemma}
\begin{proof}
	First note that $\cQ^{\optpi,m}_h(s, a)\leq\frac{B}{4}+8\eta_mT^{\optpi,m}_h(s, a)\leq\frac{B}{4}+8H\eta_m\leq \frac{B}{3}$.
	We prove the first statement by induction on $h$.
	The base case of $h=H+1$ is clearly true.
	For $h\leq H$, note that:
	\begin{align*}
		&\cc^m(s, a) + \P^m_{s, a}\cV^m_{h+1} - b^m(s, a, \cV^m_{h+1})\\
		&\leq \cc^m(s, a) + \P^m_{s, a}(\cV^{\optpi, m}_{h+1} + \eta_mT^{\optpi,m}_{h+1}) - b^m(s, a, \cV^{\optpi, m}_{h+1} + \eta_mT^{\optpi,m}_{h+1}) \tag{induction step and \pref{lem:mvp}}\\
		&\overset{\text{(i)}}{\leq} \cc^m(s, a) + \tilP^m_{s, a}(\cV^{\optpi, m}_{h+1} + \eta_mT^{\optpi,m}_{h+1}) \tag{\pref{lem:dPV}}\\
		&\leq c^m(s, a) + 8\eta_m + \Delta_{c,m} + P^m_{s, a}(\cV^{\optpi, m}_{h+1} + \eta_mT^{\optpi,m}_{h+1}) + \Delta_{P,m}(B/3 + H\eta_m) \tag{\pref{lem:dc}}\\
		&\leq \cQ^{\optpi,m}_h(s, a) + \eta_mT^{\optpi,m}_h(s, a). \tag{$H\eta_m\leq B/12$ and $\Delta'_m\leq\eta_m$}
	\end{align*}
	Note that in (i) we use the fact that $|\{ \cV^{\optpi, m}_h + \eta_mT^{\optpi,m}_h \}_{m,h}| \leq (HK+1)^6$ since $|\{V^{\optpi,m}_h\}_{m, h}|\leq HK+1$, $|\{\eta_m\}_m|\leq K+1$, and $|\{T^{\optpi,m}_h\}_{m,h}|\leq HK+1$.
	The second statement is simply by the contraposition of the first statement.
\end{proof}

\begin{lemma}
	\label{lem:chi mvp}
	With probability at least $1-12\delta$, for any $M'\leq K$, if $\Delta'_{M'}\leq \eta_{M'}$, then
	\begin{align*}
		\summp\rbr{\sumhm c^m_h + c^m_{H_m+1} -  \cV^m_1(s^m_1)} = \tilO{\B S\sqrt{AM'} + \B S^2A} \triangleq \chi_{M'}.
		%&\summp\rbr{\sumhm c^m_h + c^m_{H_m+1} -  \cV^m_1(s^m_1)}\\
		%&\lesssim \summp\sumhm\rbr{ \sqrt{\frac{\barc^m_h}{M^m_h}} + b^m_h} + \B\sqrt{M'} + \B^{1.5}S^2A + \sqrt{\B SAM'} + \sqrt{SAC_{M'}} \triangleq\chi_{M'}.
	\end{align*}
	Moreover, if \textbf{Test 1} fails in interval $m$, then $\Delta'_m>\eta_m$.
\end{lemma}
\begin{proof}
	By $\Delta'_{M'}\leq \eta_{M'}$ and \pref{lem:optQ base}, the algorithm will not terminate by \textbf{Test 2} before interval $M'$ with probability at least $1-2\delta$.
	Then with probability at least $1-4\delta$,
	\begin{align*}
		&\summp\rbr{\sumhm c^m_h + c^m_{H_m+1} -  \cV^m_1(s^m_1)}\\
		&\leq \summp\sumhm \rbr{c^m_h + \cV^m_{h+1}(s^m_{h+1}) -  \cV^m_h(s^m_h)} + \tilO{\B SA} \tag{\pref{lem:new interval} and $L_{M'}=\bigo{1}$}\\
		&\leq \summp\sumhm\rbr{ c^m_h - \hatc^m_h + \cV^m_{h+1}(s^m_{h+1}) - P^m_h\cV^m_{h+1} + (P^m_h - \P^m_h)\cV^m_{h+1} + b^m_h - 8\eta_m} \tag{definition of $\cV^m_h(s^m_h)$} + \tilO{\B SA}\\
		&\leq \tilO{\sqrt{SAC_{M'}} + \sqrt{\summp\sumhm\fV(P^m_h, \cV^m_{h+1})} + \B SA}\\
		 &\qquad + \summp\sumhm\rbr{ (\tilP^m_h - \P^m_h)\cV^m_{h+1} + B\n^m_h + b^m_h - 5\eta_m},
	\end{align*}
	where in the last inequality we apply \pref{lem:c diff}, $i^c_{M'}=i^P_{M'}=1$, $\Delta'_{M'}\leq\eta_{M'}$, $P^m_h\cV^m_{h+1}\leq \tilP^m_h\cV^m_{h+1} + B(\n^m_h+\Delta_{P,m})$, \pref{lem:freedman} and \pref{lem:e2r} on both $\summp\sumhm (c^m_h-c^m(s^m_h, a^m_h))$, and \pref{lem:freedman} on $\summp\sumhm(\cV^m_{h+1}(s^m_{h+1}) - P^m_h\cV^m_{h+1})$.
	Now note that with probability at least $1-6\delta$,
	\begin{align*}
		&\summp\sumhm ((\tilP^m_h - \P^m_h)\cV^m_{h+1} + b^m_h + B\n^m_h) + \tilO{\sqrt{\summp\sumhm\fV(P^m_h, \cV^m_{h+1})}}\\
		&= \tilO{ \sqrt{S^2A\summp\sumhm\fV(P^m_h, \cV^m_{h+1})} + \B S^2A} + \summp\sumhm b^m_h + \summp\sumhm \frac{B\Delta_{P,m}}{64} \tag{$\n^m_h\leq\frac{1}{\N^m_h}$, \pref{lem:dPv}, Cauchy-Schwarz inequality, \pref{lem:sum}, and $L_{P,M'}=1$}\\
		&= \tilO{ \sqrt{S^2A\summp\sumhm\fV(P^m_h, \cV^m_{h+1})} + \B S^2A } + \summp\sumhm \frac{B\Delta_{P,m}}{32}. \tag{\pref{lem:sum bV}, $L_{P,M'}=1$, and AM-GM inequality}\\
		&= \tilO{ \sqrt{\B S^2A(C_{M'}+M')} + \B S^2A } + \summp\sumhm\frac{\Delta'_m}{16}. \tag{\pref{lem:sum var V} and AM-GM inequality}
	\end{align*}
	Plugging this back and by $\Delta'_{M'}\leq \eta_{M'}$, we have
	\begin{align*}
		C_{M'} - \summp\cV^m_1(s^m_1) = \tilO{ \sqrt{\B S^2A(C_{M'}+M')} + \B S^2A }.
	\end{align*}
	Solving a quadratic inequality w.r.t $C_{M'}$ (\pref{lem:quad}), we have $C_{M'}=\tilo{\B M' + \sqrt{\B S^2AM'} + \B S^2A}$.
	Plugging this back completes the proof of  the first statement.
	The second statement is simply by the contraposition of the first statement.
\end{proof}

\begin{theorem}
	\label{thm:mvp-unknown}
	Suppose \pref{alg:mvp-test} does not terminate up to interval $M'\leq K$ (including $M'$) and $s^m_1=\sinit$ for $m\leq M'$.
	Then with probability at least $1-2\delta$, $\rR_{M'} = \tilo{\B S\sqrt{AM'} + \B S^2A + \summp \Delta'_{m}\T}$.
\end{theorem}
\begin{proof}
	We decompose the regret as follows:
	\begin{align*}
		\rR_{M'} &= \summp\rbr{\sumhm c^m_h + c^m_{H_m+1} - V^{\optpi, m}_1(s^m_1)}\\ 
		&= \summp\rbr{\sumhm c^m_h + c^m_{H_m+1} -  \cV^m_1(s^m_1)} + \summp\rbr{\cV^m_1(s^m_1) - \cV^{\optpi,m}_1(s^m_1)} + 8\T\summp\eta_m\\
		%&\leq \summp\rbr{\sumhm c^m_h + c^m_{H_m+1} -  \cV^m_1(s^m_1)} + \summp (\Delta_{c,m} + B\Delta_{P,m})\T + \T\summp\eta_m \tag{\pref{lem:}}\\
		&\leq \chi_{M'} + \summp \Delta'_m\T + 8\T\summp\eta_m. \tag{\textbf{Test 2} and \pref{lem:tiloptQ base}}
	\end{align*}
	Plugging in the definition of $\chi_{M'}$ and $\eta_m$ completes the proof.
\end{proof}

\begin{lemma}
	\label{lem:sum var V}
	With probability at least $1-4\delta$, $\summp\sumhm\fV(P^m_h, \cV^m_{h+1}) = \tilo{ \B (C_{M'} + M') + \B^2S^2A + \B\summp\sumhm\Delta'_m }$ for any $M'\leq K$.
\end{lemma}
\begin{proof}
	Applying \pref{lem:sum var} with $\norm{\cV^m_h}_{\infty}\leq B$ (\textbf{Test 2}), with probability at least $1-\delta$,
	\begin{align*}
		&\summp\sumhm\fV(P^m_h, \cV^m_{h+1})\\ 
		&= \tilO{\summp \cV^m_{H_m+1}(s^m_{H_m+1})^2 + \summp\sumhm \B(\cV^m_h(s^m_h) - P^m_h\cV^m_{h+1})_+ + \B^2}\\
		&= \tilO{\B (C_{M'} + M') + \B\sqrt{S^2A\summp\sumhm\fV(P^m_h, \cV^m_{h+1})} + \B^2S^2A + \B\summp\sumhm\Delta'_m},
	\end{align*}
	where in the last step we apply
	\begin{align*}
		&\summp\sumhm(\cV^m_h(s^m_h) - P^m_h\cV^m_{h+1})_+ = \summp\sumhm(\cQ^m_h(s^m_h, a^m_h) - P^m_h\cV^m_{h+1})_+\\ 
		&\leq \summp\sumhm( \cc^m_h + (\P^m_h - \tilP^m_h)\cV^m_{h+1} +  B\Delta_{P,m})_+ \tag{$(a)_+ - (b)_+\leq (a-b)_+$, definition of $\cQ^m_h$, and $b^m_h\geq 0$}\\
		&\leq \summp\sumhm c^m(s^m_h, a^m_h) + M' + \tilO{\sqrt{S^2A\summp\sumhm\fV(P^m_h, \cV^m_{h+1})} + \B S^2A} + 2\summp\sumhm\Delta'_m \tag{\pref{lem:dc}, $8\eta_m\leq\frac{1}{H}$, \pref{lem:dPv}, Cauchy-Schwarz inequality, and \pref{lem:sum}}\\
		&\leq \tilO{\summp\sumhm c^m_h + M' + \sqrt{S^2A\summp\sumhm\fV(P^m_h, \cV^m_{h+1})} + \B S^2A} + 2\summp\sumhm\Delta'_m. \tag{\pref{lem:e2r}}
	\end{align*}
	Solving a quadratic inequality w.r.t $\summp\sumhm\fV(P^m_h, \cV^m_{h+1})$ (\pref{lem:quad}) completes the proof.
\end{proof}

\subsection{Preliminaries}

Here we adopt the MASTER algorithm in \citep{wei2021non} to our finite-horizon approximation scheme.
There are several issues we need to address: 1) under the protocol of \pref{alg:fha}, the total number of intervals and the non-stationarity in each interval are not fixed before learning start;
besides, we need to prove an anytime regret guarantee, so that it can translate back to a regret guarantee on the original SSP (see \pref{lem:bound M}); 
2) when the base algorithm has a regret guarantee $\rR_{m}\leq\min\{c_1\sqrt{m}+c_2,c_3m\}$ without non-stationarity, the original MASTER algorithm ensures a dynamic regret whose dominating term scale with $c_1+c_2c_3/c_1$;
this is undesirable as $c_3=\tilo{\Tmax}$ in our case, and ideally we want $c_3=\tilo{\B}$;
3) when base algorithms incorporate correction term, the original analysis of the non-stationarity tests breaks as discussed in \pref{sec:unknown}.
Our modified MASTER algorithm (\pref{alg:master}) manages to address all these issues.

\paragraph{Setup} To give a general result, we define the dynamic regret for the first $M'$ intervals as $\tilR_{M'}=\sum_{m=1}^{M'}(C^m - f^{\star}_m)$, where the choice of benchmark $\{f^{\star}_m\}_{m=1}^{M'}$ is flexible depending on the problem and the algorithm.

\paragraph{Notations} For any interval $\calI=[s, e]$, define $\Delta_{\calI}=\sum_{m=s}^{e-1}\Delta(m)$ and $L_{\calI}=1 + \sum_{m=s}^{e-1}\Ind\{\Delta(m)\neq 0\}$, where $\Delta(m)\in\fR_+^{\fN_+}$ is some non-stationarity measure satisfying $|\fstar_{m+1} - \fstar_m|\leq\Delta(m)$.

%We let $M'$ be the current interval where the algorithms below are going to initialize.
%To derive an anytime regret bound, we always set the failure probability of an algorithm depending on the last interval in its scheduled range.
We make the following assumption on the base algorithm used in the MASTER algorithm, and then show two algorithms satisfying the assumption.

\begin{assumption}
	\label{assum:base}
	Base algorithm $\frA$ with failure probability $\delta$ on intervals $[1, M']$ outputs an estimate $\tilf_m$ at the beginning of interval $m\leq M'$ if it does not terminate before interval $m$.
	Moreover, there exists a non-decreasing function $R(m)=\min\{c_1\sqrt{m} + c_2, c_3m\}$ with $c_3\geq 1$ and non-stationarity measure $\Delta$ such that $r(m)=R(m)/m$ is non-increasing, $r(m)\geq\frac{1}{\sqrt{m}}$, $\tilf_m\leq c_4\leq c_3$ for all $m$, and with probability at least $1-\delta$, for any $m\leq M'$, as long as $\Delta_{[1, m]}\leq r(m)$ and $\frA$ does not terminate up to interval $m$ (including $m$), without knowing $\Delta_{[1, m]}$ we have:
	\begin{align*}
		&\tilf_m \leq \fstar_m + r(m), \quad \sum_{\tau=1}^m\rbr{C^{\tau} - \tilf_{\tau}} \leq R(m),\; and\quad \sum_{\tau=1}^{m}(\fstar_{\tau} - C^{\tau}) \leq R(m).
	\end{align*}
\end{assumption}

\begin{lemma}
	\label{lem:mvp-ssp assum}
	\pref{alg:MVP-SSP} with arbitrary initial state for each interval satisfies \pref{assum:base} with $\fstar_m=V^{\star,m}_1(s^m_1)$, $\tilf_m=V^m_1(s^m_1)$, $\Delta(m)=\tilo{(\Delta_{c,[m,m+1]}+B\Delta_{P,[m,m+1]})H}$,
	$R(m) =\tilo{\min\{ \B S\sqrt{Am} + \B S^2A, Hm\} }$, and $c_4=\tilo{\B}$.
\end{lemma}
\begin{proof}
	The first two properties are simply by \pref{lem:opt Q} and \pref{thm:MVP-SSP} with $L_{c,m}=L_{P,m}=1$ and $\Delta_{[1,m]}\leq r(m)$ with a large enough constant hidden in $\tilo{\cdot}$ in the definition of $\Delta(m)$.
	For the third property, with high probability,
	\begin{align*}
		\sum_{\tau=1}^m(\fstar_{\tau}-C^{\tau}) &= \sum_{\tau=1}^m(V^{\star,\tau}_1(s^m_1) - C^{\tau}) = \tilO{\sqrt{\B\sum_{\tau=1}^mC^{\tau}} + \B}. \tag{\pref{lem:V-C}}\\
		&= \tilO{\sqrt{\B\rbr{\sum_{\tau=1}^m\tilf_{\tau} + R(m)}} + \B} \leq \tilO{\B\sqrt{m} + \B} + \frac{1}{2}R(m). \tag{the second property, $V^m_1(s^m_1)=\tilo{\B}$, and AM-GM inequality}
	\end{align*}
	Plugging in the definition of $R(m)$ completes the proof (again with a large enough constant hidden in $\tilo{\cdot}$ in the definition of $R(m)$).
\end{proof}

\begin{lemma}
	\label{lem:mvp-base assum}
	\pref{alg:mvp-base} with $m\leq K$ and $s^m_1=\sinit$ satisfies \pref{assum:base} with $\fstar_m=V^{\optpi,m}_1(\sinit)$, $\tilf_m=\cV^m_1(s^m_1)$, $\Delta(m)=\tilo{(\Delta_{c,[m,m+1]}+B\Delta_{P,[m,m+1]})\T}$,
	$R(m)=\tilo{\min\{\B S\sqrt{Am} + \B S^2A, Hm\}}$, and $c_4=\tilo{\B}$.
\end{lemma}
\begin{proof}
	For the first property, by \pref{lem:tiloptQ base}, $\Delta_{[1,m]}\leq r(m)$ and a large enough constant hidden in $\tilo{\cdot}$ in the definition of $\Delta(m)$, we have
	\begin{align*}
		\tilf_m=\cV^m_1(s^m_1) \leq \cV^{\optpi,m}_1(s^m_1) + \Delta'_{m}\T \leq \fstar_m + 8\T\eta_m + \tilO{\Delta_{[1,m]}} \leq \fstar_m + r(m).
	\end{align*}
	The second property is simply by \textbf{Test 2} (\pref{lem:chi mvp}) of \pref{alg:mvp-base} (again with a large enough constant hidden in $\tilo{\cdot}$ in the definition of $R(m)$).
	For the third property,
	\begin{align*}
		\sum_{\tau=1}^m(\fstar_{\tau}-C^{\tau}) &= \sum_{\tau=1}^m(V^{\optpi,\tau}_1(\sinit) - C^{\tau}) \leq \frac{\B m}{K} + \sum_{\tau=1}^m(V^{\star,\tau}_1(\sinit) - C^{\tau})\leq R(m),
	\end{align*}
	where the first inequality is by \pref{lem:hitting} and the last step follows similar arguments as in \pref{lem:mvp-ssp assum}.
\end{proof}

\begin{lemma}
	\label{lem:V-C}
	With probability at least $1-3\delta$, for any $m\leq M$, $\sum_{\tau=1}^m(V^{\star,\tau}_1(s^{\tau}_1) - C^{\tau})=\tilo{ \sqrt{\B\sum_{\tau=1}^mC^{\tau}} + \B }$.
\end{lemma}
\begin{proof}
	With probability at least $1-3\delta$,
	\begin{align*}
		\sum_{\tau=1}^m(V^{\star,\tau}_1(s^m_1) - C^{\tau}) &\leq \sum_{\tau=1}^m\sum_{h=1}^{H_{\tau}}(V^{\star,\tau}_h(s^{\tau}_h) - V^{\star,\tau}_{h+1}(s^{\tau}_{h+1}) - c^{\tau}_h) \tag{$V^{\star,\tau}_{H_{\tau}+1}(s^{\tau}_{H_{\tau}+1}) \leq c^{\tau}_{H_{\tau}+1}$}\\
		&\leq \sum_{\tau=1}^m\sum_{h=1}^{H_{\tau}}(P^{\tau}_hV^{\star,\tau}_{h+1} - V^{\star,\tau}_{h+1}(s^{\tau}_{h+1})) =\tilO{ \sqrt{\sum_{\tau=1}^m\sum_{h=1}^{H_{\tau}}\fV(P^{\tau}_h, V^{\star,\tau}_{h+1})} + \B } \tag{$V^{\star,\tau}_h(s^{\tau}_h)\leq Q^{\star,\tau}_h(s^{\tau}_h, a^{\tau}_h)$ and \pref{lem:freedman}}\\
		&= \tilO{\sqrt{\B\sum_{\tau=1}^mC^{\tau}} + \B}. \tag{\pref{lem:optvar sum}}
	\end{align*}
\end{proof}

\subsection{MALG: Multi-Scale Learning with Base Algorithm}

\begin{algorithm}[t]
	\caption{MALG}
	\label{alg:MALG}
	\textbf{Input:} order $n$, regret density function $r$.
	
	\For{$l=0,\ldots,n$}{
		\For{$m\in\{0, 2^l, 2\cdot 2^l,\ldots, 2^n-2^l\}$}{
			With probability $\frac{r(2^n)}{r(2^m)}$, assigns a new base algorithm on intervals $[m+1, m+2^l]$.
		}
	}
	
	\For{each interval $m$}{
		Let $\frA$ be the algorithm that covers interval $m$ with shortest scheduled length, output $\tilg_m=\tilf^{\frA}_m$ (which is the $\tilf_m$ output by $\frA$), follow $\frA$'s decision, and update $\frA$ with environment's feedback.
		
		\lIf{$\frA$ terminates}{terminate.}
	}
\end{algorithm}

Following \citep[Section 3]{wei2021non}, we first introduce MALG (\pref{alg:MALG}), which runs multiple instances of base algorithms in a multi-scale manner.
We then combine MALG with non-stationarity detection to obtain the MASTER algorithm in \pref{app:ns detection}.
We always run MALG on a segment (an interval of intervals) of length $2^n$ for some integer $n$, which we call a \textit{block}.
Since we want to obtain an anytime regret guarantee, the failure probability of base algorithms and MALG need to be adjusted adaptively.
Specifically, if an MALG instance is scheduled on intervals $[\dM-2^n+1, \dM]$, then the regret guarantee of this MALG instance and the failure probability of base algorithms it maintains depends on $\dM$.
However, we ignore the dependency on $\dM$ in algorithms and analysis since the regret bound only has logarithmic dependency on $\dM$.

We show that MALG ensures a multi-scale regret guarantee in the following lemma.
Below we say an algorithm is of order $l$ if it is scheduled on a segment of length $2^l$.
Also denote by $\tilf^{\frA}_m$ the $\tilf_m$ output by $\frA$.

\begin{lemma}
	\label{lem:MALG}
	For a given $\dM\geq 1$, let $\hatn=\log_2\dM+1$ and $\hatR(m)=2^{10}\hatn\ln(2\dM/\delta)R(m)$.
	\pref{alg:MALG} scheduled on $[\dM-2^n+1, \dM]$ with input $n\leq\log_2\dM$ guarantees for any $\frA$ it maintains and any $m\in[\frA.s, \frA.e]$, as long as $\Delta_{[\frA.s, m]}\leq r(m')$ where $m'=m-\frA.s+1$ and all base algorithms it maintains do not terminate up to interval $m$ (including $m$), we have with high probability:
	\begin{align*}
		\tilg_m \leq \fstar_m + r(m''), \quad \sum_{\tau=\frA.s}^{m}(C^{\tau} - \tilg_{\tau}) \leq \hatR(m'),\quad \text{and}\quad \sum_{\tau=\frA.s}^m(\fstar_{\tau} - C^{\tau}) \leq \hatR(m'),
	\end{align*}
	where $m''$ is the number of intervals that $\frA'$ is active up to interval $m$, and $\frA'$ is the active algorithm in interval $m$.
\end{lemma}
\begin{proof}
	Fix a base algorithm $\frA$ and $m\in[\frA.s, \frA.e]$.
	Suppose $\frA'$ is active in interval $m$, which implies $[\frA'.s, \frA'.e]\subseteq[\frA.s, \frA.e]$.
	For the first statement, note that $\Delta_{[\frA'.s, m]} \leq \Delta_{[\frA.s, m]} \leq r(m') \leq r(m'')$ since $ r$ is non-increasing.
	Thus, by the guarantee of $\frA'$ (\pref{assum:base}), we have
	\begin{align*}
		\tilg_m \leq \fstar_m + r(m''). 
		%\leq \max_{\tau\in[\frA.s, m]}\fstar_{\tau} + r(m').
	\end{align*}
	For the second statement, first note that:
	\begin{align*}
		\sum_{\tau=\frA.s}^m(C^{\tau} - \tilg_{\tau}) &= \sum_{l=0}^n\sum_{\frA'\in\calS_l}\sum_{\tau=\frA.s}^m (C^{\tau} - \tilf^{\frA'}_{\tau})\Ind\{\frA'\text{ is active at }\tau\},
	\end{align*}
	where $\calS_l$ is the set of base algorithms of order $l$ which starts within $[\frA.s, m]$.
	For a fix $l$, suppose $\calS_l=\{\frA'_1,\ldots,\frA'_N\}$, and define $\calI_i=[\frA.s, m]\cap[\frA'_i.s, \frA'_i.e]$.
	Note that $\{\calI_i\}_{i=1}^N$ are disjoint, and $\Delta_{\calI_i}\leq\Delta_{[\frA.s, m]}\leq r(m')\leq r(|\calI_i|)$.
	Moreover, $[\frA'_i.s, \frA'_i.e]\subseteq [\frA.s, \frA.e]$ if $\frA'_i$ is active at some interval within $[\frA.s, m]$.
	Therefore, by the the guarantee of $\frA'_i$ (\pref{assum:base}) we have:
	\begin{align*}
		\sum_{i=1}^N\sum_{\tau=\frA.s}^m (C^{\tau} - \tilf^{\frA'_i}_{\tau})\Ind\{\frA'_i\text{ is active at }\tau\} &\leq \sum_{i=1}^N R(|\calI_i|) \leq N\cdot R(\min\{2^l, m'\}).
	\end{align*}
	Now we need to bound $N$.
	Note that $\E[N]\leq \frac{ r(2^n)}{ r(2^l)}(\frac{m'}{2^l}+1)$ by the scheduling rule.
	By \pref{lem:e2r}, with probability at least $1-\frac{\delta}{(2\dM)^6}$ (simply choose a small enough failure probability such that the failure probability over all $\dM\geq 1$ and all base algorithms is bounded), $N\leq 2\E[N] + 2^8\ln(2\dM/\delta)\leq \frac{2 r(2^n)}{ r(2^l)}\frac{m'}{2^l}+258\ln(2\dM/\delta)$ and
	\begin{align*}
		N\cdot R(\min\{2^l, m'\}) &\leq \rbr{\frac{2 r(2^n)}{ r(2^l)}\frac{m'}{2^l}+258\ln(2\dM/\delta)}R(\min\{2^l, m'\})\\
		&\leq \rbr{\frac{2R(m')}{R(2^l)} + 258\ln(2\dM/\delta)}R(\min\{2^l, m'\}) \leq 2^9\ln(2\dM/\delta)R(m'). \tag{$ r(2^n)\leq r(m')$}
	\end{align*}
	Summing over $l$ and by $n+1\leq\hatn$ proves the second statement.
	For the third statement, by \pref{lem:V-C},
	\begin{align*}
		\sum_{\tau=\frA.s}^m(\fstar_{\tau} - C^{\tau}) &= \tilO{\sqrt{\B\sum_{\tau=\frA.s}^mC^{\tau}} + \B} = \tilO{\sqrt{\B\rbr{\sum_{\tau=\frA.s}^m\tilg_{\tau} + \hatR(m')}} + \B}\tag{the second statement}\\
		&\leq \tilO{\B\sqrt{m'}} + \frac{1}{2}\hatR(m') \leq\hatR(m'). \tag{$\tilg_{\tau}\leq c_4=\tilo{\B}$ and AM-GM inequality}
	\end{align*}
	This completes the proof.
\end{proof}

%\begin{remark}
%	\label{rem:time stamp}
%	Note that the failure probability of the base algorithms in \pref{alg:MALG} has a log-polynomial dependency on the time stamp $M'$.
%	Moreover, these quantities are increasing w.r.t its time stamp $M'$.
%	Thus, we omit the dependency on $M'$ for simplicity and assume that all these quantities have log-polynomial dependency on $M$, which give an upper bound of the original quantities.
%\end{remark}

\subsection{Non-stationarity Detection: Single Block Regret Analysis}
\label{app:ns detection}

\DontPrintSemicolon
\setcounter{AlgoLine}{0}
\begin{algorithm}[t]
	\caption{MASTER}
	\label{alg:master}
	\textbf{Input:} $\hatr(\cdot)$ (defined in \pref{app:ns detection}).
	
	\textbf{Initialize:} $m\leftarrow 1$.
	
	\nl \For{$n=0,1,\ldots$}{\label{line:for}
		Set $m_n\leftarrow m$, and initialize a MALG (\pref{alg:MALG}) instance on $[m_n, m_n + 2^n - 1]$.
		
		\While{$m < m_n + 2^n$}{
			Receive $\tilg_m$ from MALG, follow MALG's decision, and suffer $C^m$.
			
			\nl Update MALG and set $U^l_m=\max_{\tau\in[m_n+2^l-1, m]}\tilg_{\tau}^l$ for all $0\leq l \leq n$, where $\tilg^l_{\tau}=\frac{1}{2^l}\sum_{\tau'=\tau-2^l+1}^{\tau}\tilg_{\tau'}$ and $U^l_m=0$ if $m<m_n+2^l-1$. \label{line:U}
			
			Perform \textbf{Test 1} and \textbf{Test 2}, and increment $m\leftarrow m + 1$.
			
			\lIf{either test fails or MALG terminates}{restart from \pref{line:for}}
		}
		
	}
	
	\nl \textbf{Test 1}: If $m=\frA.e$ for some order-$l$ $\frA$ and $\frac{1}{2^l}\sum_{\tau=\frA.s}^{\frA.e}C^{\tau}\leq U^l_m - 9\hatr(2^l)$, return \textit{fail}.\label{line:master test 1}
		
	\textbf{Test 2}: If $\frac{1}{m-m_n+1}\sum_{\tau=m_n}^m(C^{\tau} - \tilg_{\tau}) \geq 3\hatr(m-m_n+1)$. return \textit{fail}.
\end{algorithm}

Now we introduce the MASTER algorithm (\pref{alg:master}) that performs non-stationarity tests and restarts.
We first show the regret bound on a single block of order $n$ (of length $2^n$) that starts from $m_n$ and ends on $E_n$.
Clearly $E_n\leq m_n + 2^n - 1$ since it may terminate earlier than planned.
Also let $\dM=m_n+2^n-1$ be the planned last interval.
Define $\hatr(m)=\hatR(m)/m$, $\alpha_l= r(2^l)$, $\hatalpha_l=\hatr(2^l)$, and $l_0=\max_l\{12\hatalpha_{l-1} > c_4\}$.
We divide the whole block $[m_n, E_n]$ into near-stationary segments $\calI_1,\ldots,\calI_{\ell}$ with $\calI_i=[s_i, e_i]$, such that $\Delta_{\calI_i}\leq r(|\calI_i|)$ and $\Delta_{[s_i, e_i+1]}>r(|\calI_i|+1)$ for $i<\ell$.
Note that the partition depends on the learner's behavior, but whether $m\in\calI_i$ is determined at the beginning of interval $m$ before interaction starts.
In the following lemma we give a bound on $\ell$.

\begin{lemma}
	\label{lem:bound l unknown}
	Let $\calJ=[m_n, E_n]$.
	We have $\ell\leq L_{\calJ}$ and $\ell \leq 1 + (2c_1^{-1}\Delta_{\calJ})^{2/3}|\calJ|^{1/3} + c_3^{-1}\Delta_{\calJ}$.
\end{lemma}
\begin{proof}
	The first statement is clearly true.
	For the second statement follows from \pref{lem:bound l}.
\end{proof}

We also define $\tilg^l_{\tau}=\frac{1}{2^l}\sum_{\tau'=\tau-2^l+1}^{\tau}\tilg_{\tau'}$ and $f^{\star,l}_{\tau}=\frac{1}{2^l}\sum_{\tau'=\tau-2^l+1}^{\tau}\fstar_{\tau'}$ for $\tau \geq m_n+2^l-1$.
We first show a running average version of the first statement in \pref{lem:MALG}.

\begin{lemma}
	\label{lem:g2f l}
	For any $\tau\geq m_n+2^l-1$, if for any $m\in[\tau-2^l+1,\tau]$, $\Delta_{[\frA.s, m]}\leq r(m-\frA.s+1)$ where $\frA$ is the base algorithm of MALG active in interval $m$, then $\tilg^l_{\tau} \leq f^{\star,l}_{\tau} + \hatalpha_l$ with high probability.
\end{lemma}
\begin{proof}
	The case of $l=0$ is clearly true by \pref{lem:MALG}.
	For $l>0$, we have
	\begin{align*}
		\tilg^l_{\tau} &= \frac{1}{2^l}\sum_{\tau'=\tau-2^l+1}^{\tau}\tilg_{\tau'} = \frac{1}{2^l}\sum_{\tau'=\tau-2^l+1}^{\tau}\sum_{l'=0}^n\sum_{\frA'\in\calS_{l'}}\tilf^{\frA'}_{\tau'}\Ind\{\frA'\text{ is active at }\tau'\}\\
		&\leq \frac{1}{2^l}\sum_{\tau'=\tau-2^l+1}^{\tau}\sum_{l'=0}^n\sum_{\frA'\in\calS_{l'}}(\fstar_{\tau'} + r(m^{\frA'}_{\tau'}))\Ind\{\frA'\text{ is active at }\tau'\} \tag{$\Delta_{[\frA'.s, m]}\leq r(m-\frA'.s+1)\leq r(m^{\frA'}_{\tau'})$ and \pref{assum:base}} \\
		&\leq f^{\star,l}_{\tau} + \frac{1}{2^l}\sum_{l'=0}^n\sum_{\frA'\in\calS_{l'}}\sum_{\tau'=\tau-2^l+1}^{\tau}r(m^{\frA'}_{\tau'})\Ind\{\frA'\text{ is active at }\tau'\}\\
		&\leq f^{\star,l}_{\tau} + \frac{2}{2^l}\sum_{l'=0}^n|\calS_{l'}|R(\min\{2^l, 2^{l'}\}),
	\end{align*}
	where $m^{\frA'}_{\tau'}$ is the number of intervals that $\frA'$ is active up to $\tau'$, $\calS_{l'}$ is the set of order $l'$ base algorithms that intersect with $[\tau-2^l+1, \tau]$, and in the last inequality we use the fact that for any $m\geq 1$,
	\begin{align*}
		\sum_{\tau=1}^mr(\tau) = \sum_{\tau=1}^m\min\cbr{\frac{c_1}{\sqrt{\tau}} + \frac{c_2}{\tau}, c_3} \leq \min\cbr{\sum_{\tau=1}^m\rbr{ \frac{c_1}{\sqrt{\tau}} + \frac{c_2}{\tau} }, c_3m} \leq 2R(m).
	\end{align*}
	For $l'\geq l$, we have $|\calS_{l'}|\leq 2$.
	For $l'<l$, note that $\E[|\calS_{l'}|] \leq \frac{r(2^n)}{r(2^{l'})}(2^{l-l'}+1)$.
	By \pref{lem:e2r}, with high probability, $|\calS_{l'}| \leq 2\E[|\calS_{l'}|] + 2^8\ln(2\dM/\delta) \leq \frac{2R(2^l)}{R(2^{l'})} + 258\ln(2\dM/\delta)$.
	Plugging these back, we obtain
	\begin{align*}
		\frac{2}{2^l}\sum_{l'=0}^n|\calS_{l'}|R(\min\{2^l, 2^{l'}\}) &\leq \frac{2}{2^l}\sum_{l'=0}^{l-1}\rbr{\frac{2R(2^l)}{R(2^{l'})} + 258\ln(2\dM/\delta)}R(2^{l'}) + 4\sum_{l'=l}^n\alpha_l \leq \hatalpha_l.
	\end{align*}
	This completes the proof.
\end{proof}

Now we show the guarantee of non-stationarity detection on a single block $[m_n, E_n]$.
Define $\tau_i(l)$ as the smallest interval $\tau\in\calI^{l}_i\triangleq [s_i+2^l-1, e_i]$ ($\tau_i(l)=e_i+1$ if such an interval does not exist) such that $\tilg^l_{\tau} - f^{\star,l}_{\tau} > 12\hatalpha_l$, and $\xi_i(l)=e_i - \tau_i(l) + 1$.

\begin{lemma}
	\label{lem:reg block}
	Let the event in \pref{lem:MALG} hold.
	Then with high probability,
	\begin{align*}
		&\sum_{\tau=m_n}^{E_n}(C^{\tau} - \tilg_{\tau}) \leq 3\hatR(E_n - m_n + 1) + c_3,\\
		&\sum_{\tau=m_n}^{E_n}(\tilg_{\tau} - \fstar_{\tau}) \leq \tilO{\sum_{i=1}^{\ell}\hatR(|\calI_i|)} + 2^{10}\sum_{l=l_0}^n\frac{\alpha_l}{\alpha_n}\hatR(2^l)\ln(2\dM/\delta).
	\end{align*}
\end{lemma}
\begin{proof}
	The first statement trivially holds by \textbf{Test 2} and the estimated regret in a single interval is at most $c_3$ (\pref{assum:base}).
	For the second statement, define $d^l_{\tau}=\tilg^l_{\tau}-f^{\star,l}_{\tau}$.
	For a particular $\calI_i$ and any $l\geq 0$, let $\calI'_i=\calI_i\cap[\tau_i(l)-1]$. 
	If $|\calI'_i|\leq 2\cdot 2^{l+1}$, then clearly $\sum_{\tau\in\calI^l_i,\tau<\tau_i(l)}d^l_{\tau}\leq |\calI'_i|\cdot12\hatalpha_l\leq \min\{|\calI_i|, 2\cdot 2^{l+1}\}\cdot 12\hatalpha_l$.
	%note that $\calI'_i$ is a union of three segments $\calH^0_i=[s_i,s_i+2^{l+1}-1]\cap\calI'_i$, $\calH^1_i=[\tau_i(l)-2^{l+1}, \tau_i(l)-1]\cap\calI'_i$, and $\calH^2_i=[s_i+2^{l+1}, \tau_i(l)-2^{l+1}-1]$.
	If $|\calI'_i| > 2\cdot 2^{l+1}$, then $\calI'_i$ can be partitioned into three segments $\calH^0_i=[s_i,s_i+2^{l+1}-1]$, $\calH^1_i=[\tau_i(l)-2^{l+1}, \tau_i(l)-1]$, and $\calH^2_i=[s_i+2^{l+1}, \tau_i(l)-2^{l+1}-1]$.
	Note that for $\tau\in\calH^2_i$, the weight of $d^0_{\tau}$ within the sum $\sum_{\tau\in\calI^{l+1}_i,\tau<\tau_i(l)}d^{l+1}_{\tau}$ is $1$.
	Therefore, $\sum_{\tau\in\calH^2_i}d^0_{\tau}\leq \sum_{\tau\in\calI^{l+1}_i,\tau<\tau_i(l)}d^{l+1}_{\tau}$.
	Moreover, $\sum_{\tau\in\calH^0_i\cup\calH^1_i}d^0_{\tau}=2^l(d^l_{s_i+2^l-1}+d^l_{s_i+2^{l+1}-1}+d^l_{\tau_i(l)-1}+d^l_{\tau_i(l)-2^l-1})\leq 2\cdot2^{l+1}\cdot12\hatalpha_l= \min\{|\calI_i|, 2\cdot 2^{l+1}\}\cdot 12\hatalpha_l$.
	This gives
	\begin{align*}
		\sum_{\tau\in\calI^l_i,\tau<\tau_i(l)}d^l_{\tau} &\leq \sum_{\tau\in\calI'_i}d^0_{\tau} = \rbr{ \sum_{\tau\in\calH^2_i} + \sum_{\tau\in\calH^0_i\cup\calH^1_i} }d^0_{\tau}\\
		&\leq \sum_{\tau\in\calI^{l+1}_i,\tau<\tau_i(l)}d^{l+1}_{\tau} + \min\{|\calI_i|, 2\cdot 2^{l+1}\}\cdot 12\hatalpha_l\\
		&\leq \sum_{\tau\in\calI^{l+1}_i,\tau<\tau_i(l+1)}d^{l+1}_{\tau} + 12\hatalpha_l\xi_i(l+1) + \min\{|\calI_i|, 2^l\}\cdot 48\hatalpha_l. \tag{$d^{l+1}_{\tau}=\frac{1}{2}(d^l_{\tau}+d^l_{\tau-2^l})$}\\
		&\leq \sum_{\tau\in\calI^{l+1}_i,\tau<\tau_i(l+1)}d^{l+1}_{\tau} + 24\hatalpha_{l+1}\xi_i(l+1) + \min\{|\calI_i|, 2^l\}\cdot 48\hatalpha_l. \tag{$\hatalpha_l = \frac{\hatR(2^l)}{2^l}\leq \frac{2\hatR(2^{l+1})}{2^{l+1}}\leq 2\hatalpha_{l+1}$}
	\end{align*}
	Combining the two cases above, we have
	\begin{align*}
		\sum_{\tau\in\calI^l_i,\tau<\tau_i(l)}d^l_{\tau} \leq \sum_{\tau\in\calI^{l+1}_i,\tau<\tau_i(l+1)}d^{l+1}_{\tau} + 24\hatalpha_{l+1}\xi_i(l+1) + \min\{|\calI_i|, 2^l\}\cdot 48\hatalpha_l.
	\end{align*}
%	\begin{align*}
%		\sum_{\tau\in\calI^l_i,\tau<\tau_i(l)}d^l_{\tau} &\leq \sum_{\tau\in\calI_i,\tau<\tau_i(l)}d^0_{\tau} \leq \sum_{\tau\in\calI^{l+1}_i,\tau<\tau_i(l)}d^{l+1}_{\tau} + \min\{|\calI_i|, 2\cdot 2^{l+1}\}\cdot 12\hatalpha_l\\
%		&\leq \sum_{\tau\in\calI^{l+1}_i,\tau<\tau_i(l+1)}d^{l+1}_{\tau} + 12\hatalpha_l\xi_i(l+1) + \min\{|\calI_i|, 2^l\}\cdot 48\hatalpha_l. \tag{$d^{l+1}_{\tau}=\frac{1}{2}(d^l_{\tau}+d^l_{\tau-2^l})$}\\
%		&\leq \sum_{\tau\in\calI^{l+1}_i,\tau<\tau_i(l+1)}d^{l+1}_{\tau} + 24\hatalpha_{l+1}\xi_i(l+1) + \min\{|\calI_i|, 2^l\}\cdot 48\hatalpha_l. \tag{$\hatalpha_l = \frac{\hatR(2^l)}{2^l}\leq \frac{2\hatR(2^{l+1})}{2^{l+1}}\leq 2\hatalpha_{l+1}$}
%	\end{align*}
	Applying this recursively, we have for a given $\calI_i$,
	\begin{align*}
		\sum_{\tau\in\calI_i}(\tilg_{\tau} - \fstar_{\tau}) &= \sum_{\tau\in\calI^0_i, \tau<\tau_i(0)}d^0_{\tau} \leq \sum_{\tau\in\calI^n_i,\tau<\tau_i(n)}d^n_{\tau} + 24\sum_{l=0}^{n-1}\hatalpha_{l+1}\xi_i(l+1) + 48\sum_{l=0}^{n-1}\hatR(\min\{|\calI_i|, 2^l\}) \tag{$\min\{|\calI_i|, 2^l\}\hatalpha_l\leq \hatr(\min\{|\calI_i|, 2^l\})\min\{|\calI_i|, 2^l\}=\hatR(\min\{|\calI_i|, 2^l\})$}\\
		&\leq 12|\calI_i|\hatalpha_n + 24\sum_{l=1}^n\hatalpha_l\xi_i(l) + \tilO{\hatR(|\calI_i|)} \leq 24\sum_{l=1}^n\hatalpha_l\xi_i(l) + \tilO{\hatR(|\calI_i|)}.
	\end{align*}
%	\begin{align*}
%		\sum_{\tau\in\calI'_i}(\tilg_{\tau} - \fstar_{\tau}) &= \sum_{\tau\in\calI'_i}(\tilg^0_{\tau} - f^{\star,0}_{\tau}) \leq \sum_{\tau\in\calI'_i,\tau <\tau_i(0)}(\tilg^0_{\tau} - f^{\star,0}_{\tau}) + 24\hatalpha_0\xi_i(0)\\
%		&\leq 2^0\cdot 12\hatalpha_0 + \sum_{\tau\in\calI'_i,\tau <\tau_i(0)}(\tilg^1_{\tau} - f^{\star,1}_{\tau}) + 24\hatalpha_0\xi_i(0)\\
%		&\leq 2^0\cdot 12\hatalpha_0 + \sum_{\tau\in\calI'_i,\tau <\tau_i(1)}(\tilg^1_{\tau} - f^{\star,1}_{\tau}) + 24\hatalpha_1\xi_i(1) + 24\hatalpha_0\xi_i(0)
%	\end{align*}
%	\begin{align*}
%		\sum_{\tau\in\calI'_i}(\tilg_{\tau} - \fstar_{\tau}) &\leq 12\sum_{\tau\in\calI'_i}\rbr{ \hatalpha_n\Ind\{ \tilg_{\tau} - \fstar_{\tau} \leq 12\hatalpha_n \} + \sum_{l=l_0}^n\hatalpha_{l-1}\Ind\{ 12\hatalpha_l < \tilg_{\tau} - \fstar_{\tau} \leq 12\hatalpha_{l-1} \} }\\
%		&\leq 12\hatalpha_n|\calI'_i| + 12\sum_{l=l_0}^n\hatalpha_{l-1}\xi_i(l) \leq 12\hatalpha_n|\calI'_i| + 24\sum_{l=l_0}^n\hatalpha_l\xi_i(l). \tag{$\hatalpha_{l-1} = \frac{\hatR_{M'}(2^{l-1})}{2^{l-1}}\leq \frac{2\hatR_{M'}(2^l)}{2^l}\leq 2\hatalpha_l$}
%	\end{align*}
	Summing over all $i$ and by $l < l_0\implies 12\hatalpha_l > c_4 \implies\xi_i(l)=0$, we have:
	\begin{align*}
		\sum_{\tau=m_n}^{E_n}(\tilg_{\tau} - \fstar_{\tau}) \leq \tilO{\sum_{i=1}^{\ell}R(|\calI_i|)} + 24\sum_{l=l_0}^n\sum_{i=1}^{\ell}\hatalpha_l\xi_i(l).
	\end{align*}
	Now note that for any fixed $l$, we have:
	\begin{align*}
		\sum_{i=1}^{\ell}\hatalpha_l\xi_i(l) &= \hatalpha_l\sum_{i=1}^{\ell}\min\{\xi_i(l), 4\cdot 2^l\} + \hatalpha_l\sum_{i=1}^{\ell}(\xi_i(l) - 4\cdot 2^l)_+\\
		&\leq 4\sum_{i=1}^{\ell}\rbr{\hatR(|\calI_i|) + \frac{4\alpha_l}{\alpha_n}\hatR(2^l)\ln(2\dM/\delta) }. \tag{\pref{lem:bound length} and $\hatalpha_l\min\{\xi_i(l), 4\cdot 2^l\}\leq 4\hatr(\min\{\xi_i(l), 2^l\})\min\{\xi_i(l), 2^l\} = 4\hatR(\min\{\xi_i(l), 2^l\})$}
	\end{align*}
	Putting everything together completes the proof.
\end{proof}

\begin{lemma}
	\label{lem:bound length}
	For any $l\leq n$, $\sum_{i=1}^{\ell}\hatalpha_l(\xi_i(l) - 4\cdot 2^l)_+ \leq \frac{4\alpha_l}{\alpha_n}\hatR(2^l)\ln(2\dM/\delta)$ with high probability.
\end{lemma}
\begin{proof}
	Denote by $A_l$ the number of candidate starting points of an order-$l$ algorithm in $[\tau_i(l), e_i - 2\cdot 2^l]$ for some $i$.
	Note that this quantity is lower bounded by $\sum_{i=1}^{\ell}(\xi_i(l)-4\cdot2^l)_+/2^l$.
	Moreover, if in interval $m\in[\tau_i(l), e_i - 2\cdot 2^l]$, an order-$l$ algorithm $\frA$ starts, then \textbf{Test 1} is performed at $m+2^l-1\leq e_i$, and \textbf{Test 1} returns \textit{fail} with high probability because
	\begin{align*}
		\frac{1}{2^l}\sum_{\tau=\frA.s}^{\frA.e}C^{\tau} &\leq \frac{1}{2^l}\sum_{\tau=\frA.s}^{\frA.e}\tilg_{\tau} + \hatalpha_l \tag{$\Delta_{[\frA.s, \frA.e]}\leq \Delta_{\calI_i} \leq r(|\calI_i|)\leq r(2^l)$ and \pref{lem:MALG}}\\
		&\leq \frac{1}{2^l}\sum_{\tau=\frA.s}^{\frA.e}\fstar_{\tau} + 2\hatalpha_l \tag{\pref{lem:g2f l}, $\Delta_{\calI_i}\leq r(|\calI_i|)$, and $[\frA'.s, \frA'.e]\subseteq[\frA.s, \frA.e]$ if $\frA'$ is active within $[\frA.s, \frA.e]$}\\
		%\tag{\pref{lem:MALG} and $[\frA.s, \frA.e]\subseteq\calI_i$}\\
		&\leq f^{\star,l}_{\tau_i(l)} + 2\hatalpha_l + \Delta_{\calI_i}\\
		&\leq \tilg^l_{\tau_i(l)} - 12\hatalpha_l + 3\hatalpha_l \leq \tilg^l_{\tau_i(l)} - 9\hatalpha_l. \tag{$\Delta_{\calI_i}\leq r(|\calI_i|) \leq r(2^l)\leq\hatalpha_l$}
	\end{align*}
	This is a contradiction by the definition of $E_n$.
	Therefore, all candidate starting points of order-$l$ algorithm in $[\tau_i(l), e_i - 2\cdot 2^l]$ does not instantiate an order-$l$ algorithm.
	Let $X_m=\{ m\in [\tau_i(l), e_i-2\cdot2^l]\text{ for some }i\}$, $X'_m=\{ m\in [\tau_i(l), e_i]\text{ for some }i\}$, $Y_m=\{(m-m_n) \mod 2^l=0\}$ and $Z_m=\{\nexists \text{ order-$l$ } \frA' \text{ such that }\frA'.s=m\}$, we have
	\begin{align*}
		A_l = \sum_{m=m_n}^{m_n+2^m-1}\Ind\{ X_m, Y_m \} = \sum_{m=m_n}^{m_n+2^m-1}\Ind\{ X_m, Y_m, Z_m \} \leq \sum_{m=m_n}^{m_n+2^m-1}\Ind\{ X'_m, Y_m, Z_m \}.
	\end{align*}
	Note that conditioned on $X'_m\cap Y_m$, the event $Z_m$ happens with a constant probability $1-\frac{\alpha_n}{\alpha_l}$.
	Moreover, $Z_m=0$ implies $X'_{m'}=0$ for $m'>m$.
	Therefore, $\sum_{m=m_n}^{m_n+2^m-1}\Ind\{ X'_m, Y_m, Z_m \}$ counts the number of trials up to the first success with success probability $\frac{\alpha_n}{\alpha_l}$ of each trial.
	%Note that for each of these starting points, with probability $\frac{\alpha_n}{\alpha_l}$ an order-$l$ algorithm is instantiated.
	%Therefore, the probability that there exists consecutive $\frac{4\alpha_l}{\alpha_n}\ln(2\dM/\delta)$ starting points without any order-$l$ algorithm is at most $\delta/(2\dM^2)$.
	Then with probability at least $1-\delta/(2\dM^2)$, we have $A_l\leq \frac{4\alpha_l}{\alpha_n}\ln(2\dM/\delta)$.
	Thus,
	\begin{align*}
		\sum_{i=1}^{\ell}\hatalpha_l(\xi_i(l) - 4\cdot 2^l)_+ \leq \hatalpha_l2^l\cdot\frac{4\alpha_l}{\alpha_n}\ln(2\dM/\delta) \leq \frac{4\alpha_l}{\alpha_n}\hatR(2^l)\ln(2\dM/\delta).
	\end{align*}
	This completes the proof.
\end{proof}

Now we present the regret guarantee in a single block.
\begin{lemma}
	\label{lem:block reg}
	Within a single block $\calJ=[m_n, E_n]$, we have
	$$\sum_{m\in\calJ}(C^m - \fstar_m) = \tilO{ c_1\sqrt{\ell|\calJ|} + c_2\ell + \rbr{c_1 + \frac{c_2c_4}{c_1}}2^{n/2} + \frac{c_2^2}{c_3} + c_3 }.$$
\end{lemma}
\begin{proof}
	By \pref{lem:reg block}, we have
	\begin{align*}
		\sum_{m\in\calJ}\rbr{C^m - \fstar_m} &= \sum_{m\in\calJ}\rbr{C^m - \tilg_m} + \sum_{m\in\calJ}\rbr{\tilg_m - \fstar_m}\\
		&=\tilO{\hatR(|\calJ|) + \sum_{i=1}^{\ell}\hatR(|\calI_i|) + \sum_{l=l_0}^n\frac{\alpha_l}{\alpha_n}\hatR(2^l) + c_3}.
	\end{align*}
	Note that by Cauchy-Schwarz inequality:
	\begin{align*}
		\hatR(|\calJ|) + \sum_{i=1}^{\ell}\hatR(|\calI_i|) &=\tilO{\rbr{c_1\sqrt{|\calJ|} + c_2} + \sum_{i=1}^{\ell}\rbr{c_1\sqrt{|\calI_i|} + c_2} } = \tilO{ c_1\sqrt{\ell|\calJ|} + c_2\ell}.
		%&\lesssim c_1\sqrt{|\calJ|} + \Delta_{\calJ}^{1/3}(c_1|\calJ|)^{2/3} + c_1\sqrt{c_3^{-1}\Delta_{\calJ}|\calJ|} + c_2 + . \tag{\pref{lem:bound l}}
	\end{align*}
	Moreover, by the definition of $l_0$, we have $12\hatalpha_{l_0}=2^{10}\hatn\ln(2\dM/\delta)\min\{\frac{c_1}{\sqrt{2^{l_0}}} + \frac{c_2}{2^{l_0}}, c_3\} \leq c_4$, which implies $c_2 \leq c_42^{l_0}$ by $c_4\leq c_3$.
	Now for any $l\geq l_0$,
	%and assuming $c_1\leq 2^{n/2}$ (otherwise, $c_12^{n/2}> 2^n$ and the bound is vacuous): \tc{purple}{(should be $H2^n$ here?)}
	\begin{align*}
		\frac{\alpha_l}{\alpha_n}\hatR(2^l) &=\tilO{ \frac{R(2^l)^2}{R(2^n)}2^{n-l} } = \tilO{ \frac{c_1^22^{l} + c_2^2}{c_12^{n/2}+c_2}2^{n-l} + \frac{ c_1^22^{l} + c_2^2 }{c_32^n}2^{n-l} }\\
		&=\tilO{ c_12^{n/2} + \frac{c_2c_42^{l_0}}{c_1}2^{n/2-l} + \frac{c_1^2}{c_3} + \frac{c_2^2}{c_3}2^{-l} }\\
		&= \tilO{\rbr{c_1 + \frac{c_2c_4}{c_1}}2^{n/2} + \frac{c_1^2}{c_3} + \frac{c_2^2}{c_3}2^{-l} } = \tilO{ \rbr{c_1 + \frac{c_2c_4}{c_1}}2^{n/2} + \frac{c_2^2}{c_3}2^{-l} },
	\end{align*}
	where in the last inequality we assume $c_1\leq c_32^{n/2}$ without loss of generality and have $\frac{c_1^2}{c_3}\leq c_12^{n/2}$ (note that if $c_1 > c_32^{n/2}$, then $c_12^{n/2}>c_32^n$ and the regret bound is vacuous).
	Summing over $l$ and putting everything together, we obtain:
	\begin{align*}
		\sum_{m\in\calJ}\rbr{C^m - \fstar_m} = \tilO{ c_1\sqrt{\ell|\calJ|} + c_2\ell + \rbr{c_1 + \frac{c_2c_4}{c_1}}2^{n/2} + \frac{c_2^2}{c_3} + c_3 }.
	\end{align*}
\end{proof}

\subsection{Single Epoch Regret Analysis}

%Next, we derive the regret guarantee of a single epoch.
We call $[m_0, E]$ an \textit{epoch} if $m_0$ is the first interval after restart from \pref{line:for} or $m_0=1$, and $E$ is the first interval where a restart after interval $m$ is triggered.
The regret guarantee in a single epoch is shown in the following lemma.

\begin{lemma}
	\label{lem:reg epoch}
	Let $\calE$ be an epoch, then $\sum_{m\in\calE}(C^m - \fstar_m) = \tilo{c_1\sqrt{\ell_{\calE}|\calE|} + c_2\ell_{\calE} + (c_1 + \frac{c_2c_4}{c_1})\sqrt{|\calE|} + \frac{c_2^2}{c_3} + c_3 }$, where $\ell_{\calE} =\tilo{ 1 + (c_1^{-1}\Delta_{\calE})^{2/3}|\calE|^{1/3} + c_3^{-1}\Delta_{\calE} }$ and $\ell_{\calE} =\tilo{ L_{\calE} }$.
\end{lemma}
\begin{proof}
	Suppose $\calE$ consists of blocks $\calJ_1,\ldots,\calJ_n$ and the number of near stationary segments (as discussed in \pref{app:ns detection}) in $\calJ_i$ is $\ell_i$.
	Then, $|\calE|=\Theta(2^n)$, and by \pref{lem:block reg} and Cauchy-Schwarz inequality,
	\begin{align*}
		\sum_{m\in\calE}(C^m - \fstar_m) &=\tilO{ c_1\sum_{i=1}^n\sqrt{\ell_i|\calJ_i|} + c_2\sum_{i=1}^n\ell_i + \rbr{c_1 + \frac{c_2c_4}{c_1}}2^{n/2} + \frac{c_2^2}{c_3} + c_3}\\
		&= \tilO{c_1\sqrt{\sum_{i=1}^n\ell_i|\calE|} + c_2\sum_{i=1}^n\ell_i + \rbr{c_1 + \frac{c_2c_4}{c_1}}\sqrt{|\calE|} + \frac{c_2^2}{c_3} + c_3}.
	\end{align*}
	Finally by \pref{lem:bound l unknown} and H\"older's inequality, $\sum_{i=1}^n\ell_i =\tilo{ 1 + (c_1^{-1}\Delta_{\calE})^{2/3}|\calE|^{1/3} + c_3^{-1}\Delta_{\calE} }$ and $\sum_{i=1}^n\ell_i =\tilo{ L_{\calE}}$.
\end{proof}

\subsection{Full Regret Guarantee}

To derive the full regret guarantee of the MASTER algorithm (\pref{alg:master}), we first bound the number of epochs by the following two lemmas.
Define $\frN_{[1,M']}$ as the number of times MALG terminates within $[1,M']$
\begin{lemma}
	\label{lem:epoch end}
	Let $m$ be in an epoch starting from interval $m_0$.
	If $\Delta_{[m_0, m]}\leq  r(m-m_0+1)$, then no restart would be triggered by \textbf{Test 1} or \textbf{Test 2} in interval $m$ with high probability.
\end{lemma}
\begin{proof}
	We first show that \textbf{Test 1} would not fail.
	Let $m=\frA.e$ where $\frA$ is any order-$l$ base algorithm in a block of order $n$ starting from $m_n$.
	Then with high probability,
	\begin{align*}
		U^l_m &= \max_{\tau\in[m_n+2^l-1, m]}\tilg^l_{\tau} \leq \max_{\tau\in[m_n+2^l-1, m]}f^{\star,l}_{\tau} + \hatr(2^l) \tag{\pref{lem:g2f l}}\\
		&\leq \frac{1}{2^l}\sum_{\tau=\frA.s}^m\fstar_{\tau} + \hatr(2^l) + \Delta_{[m_n, m]}\\
		&\leq \frac{1}{2^l}\sum_{\tau=\frA.s}^m C^{\tau} + 2\hatr(2^l) + \Delta_{[m_n, m]} \leq \frac{1}{2^l}\sum_{\tau=\frA.s}^m C^{\tau} + 3\hatr(2^l). \tag{\pref{lem:MALG} and $\Delta_{[m_n, m]} \leq \Delta_{[m_0, m]}\leq r(m-m_0+1) \leq r(2^l)$}
	\end{align*}
	Thus, \textbf{Test 1} would not fail.
	For \textbf{Test 2}, by \pref{lem:MALG} and $\Delta_{[m_n,m]}\leq\Delta_{[m_0, m]}\leq r(m-m_0+1)\leq r(m-m_n+1)$:
	\begin{align*}
		\sum_{\tau=m_n}^m(C^{\tau} - \tilg_{\tau}) \leq \hatR(m-m_n+1),
	\end{align*}
	Thus, \textbf{Test 2} also would not fail.
\end{proof}

\begin{lemma}
	\label{lem:bound epoch}
	Assuming that MALG does not terminate without non-stationarity, with high probability, the number of epochs within $[1,M']$ is upper bounded by $L_{[1,M']}$ and  $1 + (2c_1^{-1}\Delta_{[1,M']})^{2/3}{M'}^{1/3} + c_3^{-1}\Delta_{[1,M']} + \frN_{[1,M']}$.
\end{lemma}
\begin{proof}
	The first upper bound is clearly true by partitioning $[1, M']$ into segments without non-stationarity.
	For the second upper bound, by \pref{lem:epoch end}, if an epoch $[m_0, E]$ is not the last epoch, then $\Delta_{[m_0, E]}> r(E-m_0+1)$ or MALG terminates with high probability.
	Applying \pref{lem:bound l} completes the proof.
\end{proof}

\begin{theorem}
	\label{thm:master}
	If \pref{assum:base} holds, then MASTER (\pref{alg:master}) ensures with high probability (ignoring lower order terms), for any $M'\geq 1$:
	\begin{align*}
		\tilR_{M'} &= \tilO{\rbr{c_1 + \frac{c_2c_4}{c_1}}\sqrt{L_{[1,M']}M'} }\text{ and }\\
		\tilR_{M'} &= \tilO{ \rbr{c_1 + \frac{c_2c_4}{c_1}}\sqrt{(\frN_{[1,M']}+1)M'} + \rbr{c_1^{2/3} + \frac{c_2c_4}{c_1^{4/3}}}\Delta_{[1,M']}^{1/3}{M'}^{2/3} }.
	\end{align*}
\end{theorem}
\begin{proof}
	Let $\calE_1,\ldots,\calE_N$ be epochs in $[1, M']$ and $\calE=\bigcup_{i=1}^N\calE_i$.
	Then by \pref{lem:reg epoch} and Cauchy-Schwarz inequality, we have:
	\begin{align*}
		\tilR_{M'} &=\tilO{ \sum_{i=1}^N\rbr{ c_1\sqrt{\ell_{\calE_i}|\calE_i|} + c_2\ell_{\calE_i} + \rbr{c_1 + \frac{c_2c_4}{c_1}}\sqrt{|\calE_i|} + \frac{c_2^2}{c_3} + c_3 } }\\
		&=\tilO{ c_1\sqrt{\ell_{\calE}M'} + c_2\ell_{\calE} + \rbr{c_1 + \frac{c_2c_4}{c_1}}\sqrt{NM'} + \rbr{\frac{c_2^2}{c_3}+c_3}N},
	\end{align*}
	where $\ell_{\calE}=\sum_{i=1}^N\ell_{\calE_i}$.
	%where $\ell_{\calE}=\sum_{i=1}^N\ell_{\calE_i}\leq N + L_{[1,M']} \leq 2L_{[1,M']}$.
	Below we assume sub-linear $L_{[1,M']}, \Delta_{[1,M']}$ and only write down dominating terms.
	For $L$-dependent bound, note that $N\leq L_{[1,M']}$ by \pref{lem:bound epoch} and $\ell_{\calE}\leq N + L_{[1,M']} = \tilo{L_{[1,M']}}$ by \pref{lem:reg epoch}.
	Thus, $c_2\ell_{\calE} + (\frac{c_2^2}{c_3}+c_3)N$ is a lower order term, and
	\begin{align*}
		\tilR_{M'} =\tilO{ \rbr{c_1 + \frac{c_2c_4}{c_1}}\sqrt{L_{[1,M']}M'} }.
	\end{align*}
	For $\Delta$-dependent bound, note that by \pref{lem:reg epoch}, H\"older's inequality, and \pref{lem:bound epoch},
	\begin{align*}
		\ell_{\calE} &= \tilO{N + (c_1^{-1}\Delta_{[1,M']})^{2/3}{M'}^{1/3} + c_3^{-1}\Delta_{[1,M']} }\\ 
		&= \tilO{\frN_{[1,M']} + 1 + (c_1^{-1}\Delta_{[1,M']})^{2/3}{M'}^{1/3} + c_3^{-1}\Delta_{[1,M']} }.
	\end{align*}
	Ignoring lower order term of the form $\sqrt{\Delta_{[1,M']}M'}$, we have
	\begin{align*}
		&c_1\sqrt{\ell_{\calE}M'} + \rbr{c_1 + \frac{c_2c_4}{c_1}}\sqrt{NM'}\\ 
		&=\tilO{ \rbr{c_1 + \frac{c_2c_4}{c_1}}\sqrt{\rbr{\frN_{[1,M']} + 1 + (c_1^{-1}\Delta_{[1,M']})^{2/3}{M'}^{1/3} + c_3^{-1}\Delta_{[1,M']}}M'} }\\
		%&=\tilO{ \rbr{c_1 + \frac{c_2c_4}{c_1}}\sqrt{(\frN_{[1,M']} + 1 + c_3^{-1}\Delta_{[1,M']})M'} + \rbr{c_1^{2/3} + \frac{c_2c_4}{c_1^{4/3}}}\Delta_{[1,M']}^{1/3}{M'}^{2/3} }\\
		&=\tilO{ \rbr{c_1 + \frac{c_2c_4}{c_1}}\sqrt{(\frN_{[1,M']}+1)M'} + \rbr{c_1^{2/3} + \frac{c_2c_4}{c_1^{4/3}}}\Delta^{1/3}_{[1,M']}{M'}^{2/3} }.
	\end{align*}
	The remaining $c_2\ell_{\calE} + (\frac{c_2^2}{c_3}+c_3)N$ is again a lower order term.
	%$c_2\ell_{\calE} + \frac{c_2^2}{c_3}N \lesssim \rbr{c_2 + \frac{c_2^2}{c_3}}\rbr{ 1 + (2c_1^{-1}\Delta)^{2/3}{M'}^{1/3} + c_3^{-1}\Delta }$
\end{proof}

\subsection{\pfref{thm:unknown}}

We are ready to present the regret guarantee of the MASTER algorithm combining with different base algorithms.
Recall $L=1+\sum_{k=1}^{K-1}\Ind\{P_{k+1}\neq P_k \;\text{or}\; c_{k+1}\neq c_k\}$.

\begin{theorem}
	\label{thm:mvp-ssp master}
	Let $\frA$ be \pref{alg:master} with \pref{alg:MVP-SSP} as base algorithm.
	Then \pref{alg:fha} with $\frA$ ensures with high probability, for any $K'\in[K]$,
	%Let the base algorithm be \pref{alg:fha} with $\frA$ being \pref{alg:MVP-SSP}.
	%Then, \pref{alg:master} ensures with high probability, for any $K'\in [K]$,
	$$R_{K'}=\tilO{ \min\cbr{\B S\sqrt{ALK'}, \B S\sqrt{AK'} + (\B^2S^2A(\Delta_c+\B\Delta_P)\Tmax)^{1/3}{K'}^{2/3}} }.$$
\end{theorem}
\begin{proof}
	By \pref{lem:mvp-ssp assum} and \pref{thm:master} with $\frN_{[1,M]}=0$, we have for any $M'\leq M$,
	\begin{align*}
		\rR_{M'} \leq \tilR_{M'} =\tilO{ \min\cbr{\B S\sqrt{AL_{[1,M]}M'}, \B S\sqrt{AM'} + (\B^2S^2A\Delta_{[1,M]})^{1/3}{M'}^{2/3}} },
	\end{align*}
	where $L_{[1,M]}=L$ and $\Delta_{[1,M]}=\tilo{(\Delta_c+\B\Delta_P)\Tmax}$.
	%For $L$-dependent bound, we follow the arguments in \citep[Theorem 1]{chen2021improved}, and for $\Delta$-dependent bound, we follow the arguments in \pref{lem:bound M}.
	Applying \pref{lem:bound M}, we have for any $K'\in[K]$ (ignoring lower order terms),
	\begin{align*}
		\rR_{M_{K'}}=\tilO{ \min\cbr{\B S\sqrt{ALK'}, \B S\sqrt{AK'} + (\B^2S^2A(\Delta_c+\B\Delta_P)\Tmax)^{1/3}{K'}^{2/3}} }.
	\end{align*}
	%$M'=\tilo{K}$ (ignoring lower order terms) for any $M'\leq M$.
	Applying \pref{lem:bound reg} completes the proof.
\end{proof}

We are now ready to prove \pref{thm:unknown}.

\begin{proof}[\pfref{thm:unknown}]
	By \pref{lem:optQ base} and \pref{lem:chi mvp}, when \pref{alg:mvp-base} terminates in interval $E$ where $[m_0, E]$ is an epoch, we have $\Delta'_{[m_0, E+1]}>\eta_{E-m_0+2}$.
	Therefore, $\frN_{[1,K]}=\tilo{1 + (\B^2S^2A)^{-1/3}(\T\Delta'_{[1,K]})^{2/3}K^{1/3} + H\Delta'_{[1,K]}}$ by \pref{lem:bound l} and the definition of $\eta_m$.
	Then by \pref{lem:mvp-base assum} and \pref{thm:master}, we have $\frA_1$ ensures when $s^m_1=\sinit$ for $m\leq K$,
	\begin{align*}
		\rR_K=\tilR_K=\tilO{ \min\cbr{\B S\sqrt{ALK}, \B S\sqrt{AK} + (\B^2S^2A(\Delta_c+\B\Delta_P)\T)^{1/3}{K}^{2/3} } },
	\end{align*}
	where we apply $\Delta'_{[1,K]}=\tilo{\Delta_c+\B\Delta_P}$, $L_{[1,K]}=L$, and $\Delta_{[1,K]}=\tilo{(\Delta_c+\B\Delta_P)\T}$.
	Moreover, by \pref{thm:mvp-ssp master}, $\frA_2$ ensures $R_{K'}$ being sub-linear w.r.t $K'$ for any $K'\in [K]$.
	Applying \pref{thm:two phase general} completes the proof.
\end{proof}

%\begin{lemma}
%	$\T\sum_{m=1}^{M'}\eta_m\leq \B S\sqrt{ALM'}$.
%\end{lemma}
%\begin{proof}
%	Note that within a block of order $n$, we have:
%	\begin{align*}
%		\T\sum_{\tau=0}^{2^n}\eta_{\tau} \leq \sum_{l=0}^n\rbr{\frac{R(2^n)}{R(2^l)} + 4\ln(M'/\delta)}R(2^l) \leq \hatR(2^l).
%	\end{align*}
%	Therefore, for each epoch, by the doubling trick, the sum is of order $\hatR(2^n)$.
%	Summing over all epochs completes the proof.
%\end{proof}

\section{Auxiliary Lemmas}
% !TEX root = main.tex

\begin{lemma}\citep[Lemma 48]{chen2022policy}
	\label{lem:quad}
	$x\leq a\sqrt{x} + b$ implies $x\leq (a+\sqrt{b})^2 \leq 2a^2 + 2b$.
\end{lemma}

\begin{lemma}{\citep[Lemma 6]{rosenberg2020adversarial}}
	\label{lem:hitting}
	Let $\pi$ be a policy whose expected hitting time starting from any state is at most $\tau$.
	Then for any $\delta\in(0, 1)$, with probability at least $1-\delta$, it takes no more than $4\tau\ln\frac{2}{\delta}$ steps to reach the goal state following $\pi$.
\end{lemma}

\begin{lemma}{\citep[Lemma 30]{chen2021implicit}}
	\label{lem:var X2}
	For any random variable $X$ with $\norm{X}_{\infty}\leq C$, we have $\var[X^2]\leq 4C^2\var[X]$.
\end{lemma}

\begin{lemma}{(\citep[Lemma 31]{chen2021implicit})}
	\label{lem:mvp}
	Define $\Upsilon=\{ v\in[0, B]^{\calS_+}: v(g)=0 \}$.
	Let $f: \Delta_{\calS_+}\times\Upsilon\times\fR^+\times\fR^+\times\fR^+\rightarrow\fR^+$ with $f(p, v, n, B, \iota)=pv-\max\Big\{c_1\sqrt{\frac{\fV(p, v)\iota}{n}}, c_2\frac{B\iota}{n}\Big\}$ with $c_1^2\leq c_2$.
	Then $f$ satisfies for all $p\in\Delta_{\calS_+}, v\in\Upsilon$ and $n, \iota>0$,
	\begin{enumerate}
		\item $f(p, v, n, B, \iota)$ is non-decreasing in $v(s)$, that is,
		$$\forall v, v'\in\Upsilon, v(s)\leq v'(s), \forall s\in\calS^+ \implies f(p, v, n, B, \iota)\leq f(p, v', n, B, \iota);$$
		\item $f(p, v, n, B, \iota)\leq pv-\frac{c_1}{2}\sqrt{\frac{\fV(p, v)\iota}{n}}-\frac{c_2}{2}\frac{B\iota}{n}$.
	\end{enumerate}
\end{lemma}

\begin{lemma}[Any interval Freedman's inequality]
	\label{lem:freedman}
	Let $\{X_i\}_{i=1}^{\infty}$ be a martingale difference sequence w.r.t the filtration $\{\calF_i\}_{i=0}^{\infty}$ and $|X_i|\leq B$ for some $B>0$.
	Then with probability at least $1-\delta$, for all $1\leq l\leq n$ simultaneously,
	\begin{align}
		\abr{\sum_{i=l}^nX_i} &\leq 3\sqrt{\sum_{i=l}^n\E[X_i^2|\calF_{i-1}]\ln\frac{16B^2n^5}{\delta}} + 2B\ln\frac{16B^2n^5}{\delta} \label{eq:freedman}\\
		&\leq 3\sqrt{2\sum_{i=l}^nX_i^2\ln\frac{16B^2n^5}{\delta}} + 18B\ln\frac{16B^2n^5}{\delta}. \label{eq:freedman emp}
	\end{align}
\end{lemma}
\begin{proof}
	For each $l\geq 1$, by \citep[Lemma 38]{chen2021improved}, with probability at least $1-\frac{\delta}{4l^2}$, \pref{eq:freedman} holds for all $n\geq l$.
	Then by \pref{lem:e2r}, with probability at least $1-\frac{\delta}{4l^2}$, \pref{eq:freedman emp} holds for all $n\geq l$.
	Applying a union bound over $l$ completes the proof.
\end{proof}

\begin{lemma}
	\label{lem:e2r}
	Suppose $\{X_i\}_{i=1}^{\infty}$ is a sequence of random variables w.r.t the filtration $\{\calF_i\}_{i=0}^{\infty}$ and satisfies $X_i\in[0, B]$ for some $B>0$.
	Then with probability at least $1-\delta$, for all $1\leq l\leq n$ simultaneously,
	\begin{align*}
		\sum_{i=l}^n\E[X_i|\calF_{i-1}] &\leq 2\sum_{i=l}^nX_i + 12B\ln\frac{2n}{\delta},\\
		\sum_{i=l}^nX_i &\leq 2\sum_{i=l}^n\E[X_i|\calF_{i-1}] + 24B\ln\frac{2n}{\delta}.
	\end{align*}
\end{lemma}
\begin{proof}
	For each $l\geq 1$, by \citep[Lemma 39]{chen2021improved}, with probability at least $1-\frac{\delta}{2l^2}$, the two inequalities above hold for all $n\geq l$.
	Taking a union bound over $l$ completes the proof.
\end{proof}

%\stopcontents[sections]

\end{document}